\documentclass[twoside,11pt]{article}

\usepackage{jmlr2e}
\usepackage{hyperref,url}            
\usepackage{booktabs,nicefrac,microtype}      
\usepackage{times,graphicx,subfigure}
\usepackage{algorithm,algorithmic}
\usepackage{amsmath}\allowdisplaybreaks
\usepackage{amsfonts,amssymb}
\usepackage{lastpage}
\jmlrheading{24}{2023}{1-\pageref{LastPage}}{10/22; Revised
7/23}{9/23}{22-1210}{Haishan Ye, Luo Luo, Ziang Zhou, and Tong Zhang}
\ShortHeadings{Mudag}{Ye, Luo, Zhou, and Zhang}
\usepackage{multirow}
\usepackage{tablefootnote}

\newcommand{\cO}{\mathcal{O}}
\newcommand{\RR}{\mathbb{R}}

\newcommand{\bbs}{\bar{s}}

\newcommand{\bbg}{\bar{g}}
\newcommand{\bbG}{\bar{G}}

\newcommand{\TB}{\mathbb{T}}

\usepackage{bbm}

\newcommand{\xb}{\mathbf{x}}

\newcommand{\bA}{\mathbf{A}}

\newcommand{\zb}{\mathbf{z}}
\newcommand{\yb}{\mathbf{y}}
\newcommand{\bby}{\bar{y}}
\newcommand{\bbx}{\bar{x}}
\newcommand{\bs}{\mathbf{s}}
\newcommand{\bv}{\mathbf{v}}

\newcommand{\bbv}{\bar{v}}

\newcommand{\proximal}{{\rm{\bf prox}}}

\newcommand{\norm}[1]{\left\|#1\right\|}
\newcommand{\dotprod}[1]{\left\langle #1\right\rangle}

\newcommand{\argmin}{\mathop{\mathrm{argmin}}}

\usepackage[dvipsnames]{xcolor}

\def\mudag{\texttt{Mudag}}
\def\dapg{\texttt{ProxMudag}}
\newcommand{\geng}{\tilde{\nabla}}

\firstpageno{1}

\begin{document}
	
	\title{Multi-Consensus Decentralized Accelerated Gradient Descent}
	
	\author{\name Haishan Ye  \email yehaishan@xjtu.edu.cn \\
		\addr Center for Intelligent Decision-Making and Machine Learning\\
		School of Management\\
		Xi'an Jiaotong University\\
		Xi'an, China
		\AND
		\name Luo Luo\thanks{Corresponding author} \email luoluo@fudan.edu.cn\\
		\addr School of Data Science \\
		Fudan University \\
		Shanghai, China
		\AND
		\name Ziang Zhou \email 20071642r@connect.polyu.hk\\
		\addr Department of Computing \\
        The Hong Kong Polytechnic University \\
		Hong Kong, China
		\AND
		\name Tong Zhang \email tongzhang@tongzhang-ml.org \\
		\addr  Computer Science \& Mathematics \\
		The Hong Kong University of Science and Technology  \\
		Hong Kong, China}
	\editor{Ohad Shamir}
	
	\maketitle
	
	\begin{abstract}%
		This paper considers the decentralized convex optimization problem, which has a wide range of applications in large-scale machine learning, sensor networks, and control theory. 
		We propose novel algorithms that achieve optimal computation complexity and near optimal communication complexity. 
		Our theoretical results give affirmative answers to the open problem on whether there exists an algorithm that can achieve a communication complexity (nearly) matching the lower bound depending on the global condition number instead of the local one. 
		Furthermore, the linear convergence of our algorithms only depends on the strong convexity of global objective and it does \emph{not} require the local functions to be convex.  
		The design of our methods relies on a novel integration of well-known techniques
		including Nesterov's acceleration, multi-consensus and gradient-tracking. 
		Empirical studies show the outperformance of our methods for machine learning applications. 
	\end{abstract}
	
	\begin{keywords}
		consensus optimization, decentralized algorithm, accelerated gradient descent, gradient tracking, composite optimization
	\end{keywords}
	
	\section{Introduction}
	
	In this paper, we consider the decentralized optimization problem,
	where the objective function is composed of $m$ local functions
	$f_i(x)$ that are located on $m$ different agents.  
	The agents form a connected and undirected network 
	and each of them only accesses its local function and communicates with its neighbors. 
	All of the agents target to cooperatively solve the convex optimization problem
	\begin{equation}\label{eq:prob}
		\min_{x\in\RR^d} h(x)\triangleq f(x)+r(x) \quad \mbox{with}\quad f(x)\triangleq \frac{1}{m}\sum_{i=1}^m f_i(x),
	\end{equation}
	where $f(x)$ is $L$-smooth and $\mu$-strongly convex,  $r(x)$ is convex but may be non-differentiable. 
	Many machine learning models have the form~\eqref{eq:prob} such as logistic regression and elastic net regression.
	Decentralized optimization has been widely studied and applied in many applications such as large-scale machine learning \citep{tsianos2012consensus,kairouz2019advances}, automatic control \citep{bullo2009distributed,lopes2008diffusion}, wireless communication \citep{ribeiro2010ergodic}, and sensor networks \citep{rabbat2004distributed,khan2009diland}.
	
	Many decentralized optimization algorithms have been proposed.
	One class of them is primal-only methods, including decentralized gradient methods \citep{nedic2009distributed,yuan2016convergence}, decentralized accelerated gradient method \citep{jakovetic2014fast,qu2019accelerated} and \texttt{EXTRA} \citep{shi2015extra,li2019decentralized,mokhtari2016dsa}.
	They only access the gradients of $f_i(x)$ and are usually computationally efficient.
	Another class of algorithms are the dual-based decentralized algorithms, such as the dual subgradient ascent \citep{terelius2011decentralized}, dual gradient ascent and its accelerated version \citep{seaman2017optimal,uribe2018dual}, the primal-dual method \citep{lan2018communication,scaman2018optimal,hong2017prox}, and ADMM \citep{erseghe2011fast,shi2014linear}. However, dual-based algorithms commonly need more computation cost when the gradient of the dual function is not explicitly available. 
	
	There are several important open problems in the area of decentralized optimization.
	First, \citet{seaman2017optimal,scaman2019optimal} raised the problem 
	whether there exists an algorithm that has a (near) optimal communication complexity depending on the global condition number $\kappa_g=L/\mu$ instead of the local condition number $\kappa_{\ell}$ (defined in Eq.~\eqref{eq:kappa}).
	Since the data distributed on different agents are potentially quite different, the global condition number $\kappa_g$ could be much smaller than $\kappa_\ell$.
	In the extreme case, the local function $f_i$ may be non-strongly convex, then it is possible that $\kappa_{\ell}$ is infinitely large while $\kappa_{g}$ is still small.
	However, existing works only achieved  the optimal computation and communication complexities with respect to the local condition number $\kappa_\ell$ in the case of $r(x)=0$~\citep{kovalev2020optimal,li2021accelerated,song2021optimal,seaman2017optimal}.
	Furthermore, it is unclear whether the convexity of each individual $f_i(x)$
	is essential for computation-efficient and communication-efficient decentralized algorithms. 
	Most of existing algorithms with linear convergence rates such as \texttt{EXTRA} \citep{shi2015extra} and \texttt{OPAC} \citep{kovalev2020optimal} all require each $f_i(x)$ to be (strongly) convex.
	\cite{sun2019convergence} first proposed the linear-convergent algorithm that allows some individual functions to be non-convex.
	However, \citet{sun2019convergence}'s algorithm cannot achieve the (near) optimal computation and communication complexities.
	Finally, existing methods cannot achieve the optimal computation and (near) optimal communication complexities for non-differentiable $r(x)$~\citep{xu2021distributed,alghunaim2020decentralized,alghunaim2019linearly,sun2019convergence}.
	How to design computation and communication efficient accelerated decentralized proximal gradient descent is still an open question.
	
	This paper addresses the theoretical issues discussed above and designs two novel decentralized algorithms \dapg~and \mudag. 
	We summarize our contributions as follows:
	\begin{enumerate}
		\item Our algorithms have the optimal computation complexity
		$\cO\left(\sqrt{\kappa_g}\log({1}/{\epsilon})\right)$ and
		the near optimal communication complexity $\cO\big(\sqrt{{\kappa_g}/{(1-\lambda_2(W))}}\log\big({M\kappa_g}/{L}\big)\log({1}/{\epsilon})\big)$, where $M$ and $L$ are the smoothness parameters of $f_i(x)$ and $f(x)$ respectively. 
		To the best of our knowledge, this is the first (near) optimal decentralized algorithm that depends on the global condition number which  provides an affirmative answer to the open problem whether there exists an algorithm that can achieve a communication complexity of $\cO\left(\sqrt{{\kappa_g}/{(1-\lambda_2(W))}}\log({1}/{\epsilon})\right)$ or even close to it \citep{seaman2017optimal}.
		\item Our algorithms do \emph{not} require each individual function to be  convex. Hence, they can be used in a wider range of applications than existing optimal decentralized algorithms. For example, the sub-problem of fast PCA  by the shift-invert method is non-convex.
		\item The proposed \dapg~ can achieve optimal computation and (near) optimal communication complexity when $r(x)$ is convex but non-differentiable. 
		To the best of our knowledge, it obtains the best-known communication complexity for the decentralized strongly-convex optimization problems with the composite objective function.
	\end{enumerate}
	
	\section{Related Work}
	
	We first review the penalty-based algorithms.
	\cite{nedic2009distributed} proposed the well-known decentralized
	gradient descent method, where each agent performs a consensus step and a gradient descent step with a fixed step-size related to the penalty parameter.
	\cite{yuan2016convergence} proved the convergence rate of
	decentralized gradient descent and showed how the penalty parameter affects the computation complexity.
	To avoid the diminishing step-size commonly required in penalty-based algorithms, \cite{jakovetic2014fast} combined multi-consensus and Nesterov's acceleration to achieve the optimal computation complexity for minimizing non-strongly convex functions. 
	\cite{berahas2018balancing} proposed to use multi-consensus to achieve the balance between computation and communication complexity.
	Recently, \cite{li2018sharp} proposed \texttt{APM-C}, which employed multi-consensus and increased the penalty parameter properly for each iteration.
	Combining Nesterov's acceleration, \texttt{APM-C} can achieve a linear convergence rate and a low communication complexity.
	\cite{li2019communication} applied multi-consensus to network Newton method to achieve computation and communication efficiency. 
	
	Dual-based methods are another important research line. 
	These methods introduce a Lagrangian function and work in the dual space.
	There are different ways to solve the reformulated problem such as gradient descent method \citep{terelius2011decentralized}, accelerated gradient method \citep{seaman2017optimal,uribe2018dual}, primal-dual method \citep{lan2018communication,scaman2018optimal} and ADMM \citep{shi2014linear,erseghe2011fast}.
	However, such methods are typically computationally inefficient.
	For example, using the accelerated gradient method to solve the dual counterpart of the decentralized optimization problem can achieve optimal communication complexity \cite{seaman2017optimal,uribe2018dual}, but its computation complexity will have an additional dependency on the eigenvalue gap of gossip matrix~\citep{uribe2018dual}.
	
	The gradient-tracking method is a popular way to reduce the computational cost~\citep{qu2017harnessing,xu2015augmented,qu2019accelerated,di2016next,di2015distributed,sun2019convergence,nedic2017achieving,zhu2010discrete}.
	There are two different techniques for gradient-tracking. 
	One of them is keeping a variable to estimate the average gradient and uses this estimation in the gradient descent step~\citep{sun2019convergence,di2016next,qu2017harnessing}.
	Another one is introducing two different weight matrices to track the difference of gradients \citep{shi2015extra,li2019decentralized}.
	Recently, \citep{nedic2017achieving,li2020revisiting,jakovetic2018unification,xu2021distributed} studied the connection between these two strategies and showed that they can be transformed to each other. 
	Due to the tracking of history information, gradient-tracking based algorithms can achieve linear convergence rates for strongly convex objective functions \citep{qu2017harnessing,shi2015extra,nedic2017achieving,sun2019convergence}. 
	However, the previously obtained convergence rates and communication complexities
	are much worse than the results in this paper.
	
	\begin{table*}
		\begin{center}
			\scriptsize            
			\begin{tabular}{cccc}
				\hline
				Methods & Computation & Communication & Is $f_i(x)$ convex? \\
				\hline \addlinespace
				Acc-DNGD \citep{qu2019accelerated} &$\cO\left(\frac{\kappa_\ell^{5/7}}{(1-\lambda_2(W))^{1.5}}\log\left(\frac{1}{\epsilon}\right)\right)$  & $\cO\left(\frac{\kappa_\ell^{5/7}}{(1-\lambda_2(W))^{1.5}}\log\left(\frac{1}{\epsilon}\right)\right)$  & Yes\\ \addlinespace
				NIDS \citep{li2019decentralized} & $\cO\left(\big(\kappa_\ell+ \frac{1}{1-\lambda_2(W)}\big)\log\left(\frac{1}{\epsilon}\right)\right)$ & $\cO\left(\big(\kappa_\ell+ \frac{1}{1-\lambda_2(W)}\big)\log\left(\frac{1}{\epsilon}\right)\right)$ & Yes\\\addlinespace
				ADA \citep{uribe2018dual}\hspace*{-0.25cm}& $\cO\left(\frac{\kappa_\ell}{\sqrt{1-\lambda_2(W)}}\log^2\frac{1}{\epsilon}\right)$ & $\cO\left(\sqrt{\frac{\kappa_\ell}{1-\lambda_2(W)}}\log\left(\frac{1}{\epsilon}\right)\right)$ & Yes\\\addlinespace
				APM-C \cite{li2018sharp} & $\cO\left(\sqrt{\kappa_\ell}\log\left(\frac{1}{\epsilon}\right)\right)$  & $\cO\left(\sqrt{\frac{\kappa_\ell}{1-\lambda_2(W)}}\log^2\frac{1}{\epsilon}\right)$ &Yes\\\addlinespace
				Acc-EXTRA \citep{li2020revisiting} & $\widetilde{\cO}\left(\sqrt{\frac{\kappa_\ell}{1-\lambda_2(W)}}\log\left(\frac{1}{\epsilon}\right)\right)$  & $\widetilde{\cO}\left(\sqrt{\frac{\kappa_\ell}{1-\lambda_2(W)}}\log\left(\frac{1}{\epsilon}\right)\right)$ &Yes
				\\\addlinespace
				OPAC \citep{kovalev2020optimal} & $\cO\left( \sqrt{\kappa_\ell}\log \frac{1}{\epsilon} \right)$ & $\cO\left( \sqrt{\frac{\kappa_\ell}{1 - \lambda_2(W)}}\log \frac{1}{\epsilon} \right)$  &  Yes\\\addlinespace
				{\bf Mudag} (Algorithm~\ref{alg:DAGD}) & $\cO\left(\sqrt{\kappa_g}\log\left(\frac{1}{\epsilon}\right)\right)$ & $\widetilde{\cO}\left(\sqrt{\frac{\kappa_g}{1 - \lambda_2(W)}}\log\left(\frac{1}{\epsilon}\right)\right)$ &{\bf No}\\\addlinespace
				\hline\addlinespace
				Lower Bound \citep{seaman2017optimal} \hspace*{-0.25cm}& $\Omega\left(\sqrt{\kappa_g}\log\left(\frac{1}{\epsilon}\right)\right)$ & $\Omega\left(\sqrt{\frac{\kappa_g}{1-\lambda_2(W)}}\log\left(\frac{1}{\epsilon}\right)\right)$\tablefootnote{,,,}  & \\\addlinespace
				\hline
			\end{tabular}
		\end{center}
		\caption{Complexity comparisons between  our algorithm and existing works for smooth and strongly convex problems. That is,  $r(x) $ equals to zero in Problem~\eqref{eq:prob}. The notation $\cO(\cdot)$ hides the constant terms and $\widetilde{\cO}(\cdot)$ also hides $\log$ terms which are independent of $\epsilon$. }\label{table1}
		\vskip -0.2in
	\end{table*}
	
	\footnotetext{It holds that $\kappa_{g} = \Omega(\kappa_{\ell})$ for the  case used to prove the lower bound of communication complexity \citep{seaman2017optimal}.}
	
	The communication complexities of existing works commonly depend on the local condition number $\kappa_{\ell}$~\citep{seaman2017optimal,li2018sharp,li2019decentralized,li2021accelerated,kovalev2020optimal}. 
	Only \texttt{EXTRA}~\citep{shi2015extra} and \texttt{DIGing} \citep{nedic2017achieving} achieve computation and communication complexities depending on $\hat{\kappa}_g$ (defined in Eq.~\eqref{eq:kappa}), which is still worse than our results that depend on $\kappa_{g}$ and $\log\hat{\kappa}_g$. 
	Due to the fact $\kappa_{g} \le \hat{\kappa}_g \le m \kappa_g$, a communication complexity depending on $\kappa_{g}$ is preferred in real applications. 
 
	We summarize the results for the case of $r(x) = 0$ in Table~\ref{table1}.
	\texttt{Acc-DNGD} is most relevant to our algorithm.
	It also utilizes Nesterov's acceleration and gradient-tracking~\citep{qu2019accelerated}.
	However, the multi-consensus step in our algorithm can analog the centralized accelerated gradient descent more efficiently and leads the convergence analysis  almost to be the same as standard analysis~\citep{nesterov2018lectures} in the centralized scenario.
	In contrast, \texttt{Acc-DNGD} does not have such a good property and it does not achieve near optimal computation complexity nor near optimal communication complexity.
	Since our algorithm can effectively approximate the centralized accelerated gradient descent, our  algorithm does not require each individual function $f_i(x)$ to be convex but only requires $f(x)$ to be strongly convex while this condition is required in \texttt{Acc-DNGD}.
	Finally, the convergence rate of our algorithm depends on the global
	condition number $\kappa_g$, while that of
	\texttt{Acc-DNGD} depends on the local condition number $\kappa_\ell$.
	
	In \cite{seaman2017optimal}, a lower bound of communication
	complexity was obtained for the decentralized optimization problem,
	which is
	$\cO\big(\sqrt{{\kappa_g}/{(1-\lambda_2(W))}}\log({1}/{\epsilon})\big)$
	for strongly convex problems. A dual-based algorithm was proposed
	to match the lower bound.
	However, this method is only suitable for the cases where dual
	functions of each local agent are easy to compute. 
	Hence, the computation complexity of the method in
	\cite{seaman2017optimal}  severely deteriorates once the dual
	functions are computationally inefficient to work with.
	Recently, \cite{uribe2018dual} proposed an accelerated dual ascent
	algorithm  which achieves the same communication complexity as the one
	of \cite{seaman2017optimal}, but with a computation complexity of $\cO\big({\kappa_\ell}/{\sqrt{1-\lambda_2(W)}}\log^2({1}/{\epsilon})\big)$.
	
	Recently, \cite{li2020revisiting} proposed \texttt{Acc-EXTRA} by applying Catalyst to accelerate \texttt{EXTRA}.
	However, due to the lack of the multi-consensus,  \texttt{Acc-EXTRA} fails to  achieve the optimal computation complexity. 
	On the other hand, its communication complexity is also no better than \texttt{Mudag}.
	Furthermore,  Catalyst introduces an additional loop of iteration. 
	In contrast, \texttt{Mudag} is simple and easy to implement.
	\citet{kovalev2020optimal} proposed \texttt{OPAC}, which is a primal-dual based algorithm. The computation and communication complexities of \texttt{OPAC} are $\cO\left(\sqrt{\kappa_\ell} \log ({1}/{\epsilon}) \right)$ and~$\cO\big(\sqrt{{\kappa_{\ell}}/{(1 - \lambda_2(W))}} \log (1/\epsilon)\big)$ respectively, which depends on the local condition number.
	Additionally, it requires each $f_i(x)$ to be strongly convex.
	
	\begin{table*}
		\begin{center}
			\scriptsize
			\begin{tabular}{cccc}
				\hline 
				Methods & Computation & Communication & Is $f_i(x)$  convex?\\
				\hline\addlinespace
				NIDS\tablefootnote{,,,}  \citep{li2019decentralized} & $\cO\left(\left(\kappa_\ell+\frac{1}{1-\lambda_2(W)}\right)\log\left(\frac{1}{\epsilon}\right)\right)$ & $\cO\left(\left(\kappa_\ell+ \frac{1}{1-\lambda_2(W)}\right)\log\left(\frac{1}{\epsilon}\right)\right)$ & Yes\\\addlinespace
				D2P2 \citep{alghunaim2019linearly} & $\cO\left(\frac{\kappa_{\ell}}{1 - \lambda_2(W)}\log\left(\frac{1}{\epsilon}\right)\right)$  & $\cO\left(\frac{\kappa_{\ell}}{1 - \lambda_2(W)}\log\left(\frac{1}{\epsilon}\right)\right)$ &Yes\\\addlinespace
				\hline\addlinespace
				\textbf{ProxMudag} (Algorithm~\ref{alg:DAGD_p}) & $\cO\left(\sqrt{\kappa_g}\log\left(\frac{1}{\epsilon}\right)\right)$ & $\widetilde{\cO}\left(\sqrt{\frac{\kappa_g}{1 - \lambda_2(W)}}\log\left(\frac{1}{\epsilon}\right)\right)$ &{\bf No}\\\addlinespace
				\hline
			\end{tabular}
		\end{center}
		\caption{Complexity comparisons between  our algorithm and existing works for composite and strongly convex problems. }\label{table2}
		\vskip -0.2in
	\end{table*}
	
	\footnotetext{\citet{li2019decentralized} only gave a sublinear convergence rate for \texttt{NIDS} when $r(x)$ is convex, the linear convergence rate is proved in works \citep{alghunaim2020decentralized,xu2021distributed}.}

	For the case $r(x)$ is convex but non-differentiable, many gradient tracking based algorithms have been extended to decentralized composite optimization problems with a non-differentiable regularization term  such as \texttt{PG-EXTRA} \citep{ShiLWY15} and \texttt{NIDS} \citep{li2019decentralized}.
	However, due to the non-differentiable term, these algorithms can only achieve  sub-linear convergence rates.
	Recently, \cite{sun2019convergence} proposed a gradient
	tracking based method called 
	\texttt{SONATA}, and
	established a linear convergence rate with the assumption that $f(x)$
	is strongly convex.
	In addition, \citet{alghunaim2019linearly} proposed a primal-dual algorithm which can achieve a linear convergence rate when each $f_i(x)$ is convex.
	Recently, \citet{alghunaim2020decentralized,xu2021distributed} proposed a unified framework to analyze a large group of algorithms. They showed the algorithms including \texttt{EXTRA} (\texttt{PG-EXTRA}) \citep{shi2015extra}, \texttt{NIDS} \citep{li2019decentralized} and \texttt{Harnessing} \citep{qu2017harnessing} can also achieve linear convergence rates with a non-differentiable regularization term.
	Despite intensive studies in the literature, the convergence rates of
	these previous algorithms do not match the optimal convergence rate. 
	Moreover, the communication complexities achieved by algorithms
	analyzed in the framework of \citet{xu2021distributed} and \citet{alghunaim2020decentralized} are sub-optimal.
	The conference version of our paper proposed \texttt{DAPG} which achieves the optimal computation complexity and near optimal communication complexity \citep{Ye2020}. 
	However \texttt{DAPG} takes three multi-consensus steps while \dapg~in this paper only takes two multi-consensus steps.
	Thus, \dapg~can achieve better communication-efficiency than \texttt{DAPG}.
	We compare our \dapg~with existing state-of-the-art decentralized algorithms for the composite optimization in Table~\ref{table2}.

	\section{Preliminaries}
	
	We let $x_i\in\RR^d$ be the local copy of the variable of $x$ for agent $i$ and we introduce the aggregate variable $\xb\in\RR^{m\times d}$, aggregate objective function $F(\xb)$ and  aggregate gradient $\nabla F(\xb)\in\RR^{m\times d}$  as
	\begin{equation}
		\label{eq:x_def}
		\xb =
		\begin{bmatrix}
			x_1^\top\\
			\vdots\\
			x_m^\top
		\end{bmatrix},\quad
		F(\xb) = \frac{1}{m}  \sum_{i=1}^{m}f_i(x_i)\quad \text{and} \quad
		\nabla F(\xb) 
		=  
		\begin{bmatrix}
			\nabla f_1(x_1)^\top\\
			\vdots\\
			\nabla f_m(x_m)^\top
		\end{bmatrix}.
	\end{equation}
	We denote that 
	\begin{equation}
		\label{eq:xyg}
		\bbx_t = \frac{1}{m}\sum_{i=0}^{m} \xb_t^{(i)}, \quad \bby_t = \frac{1}{m}\sum_{i=0}^m \yb_t^{(i)}\quad\text{and}\quad\bbg_t = \frac{1}{m} \sum_{i=0}^m \nabla f_i(\yb_t^{(i)}),
	\end{equation}
	where $\xb^{(i)}$ means the $i$-th row of matrix $\xb$.
	Moreover, we use $\norm{\cdot}$ to denote the Frobenius norm of vector or matrix and use $\dotprod{x,y}$ to denote the inner product of vectors $x$ and $y$. 
	
	Furthermore, we denote
	\begin{equation*}
		R(\xb) = \frac{1}{m}  \sum_{i=1}^{m}r(x_i).
	\end{equation*}
	Accordingly, we introduce the proximal operator and aggregated proximal operator with respect to $r(\cdot)$ and $R(\cdot)$ as 
	\begin{equation}
		\label{eq:proxes}\small
		\begin{split}
			\!\proximal_{\eta,r}(x)=\argmin_{z\in\mathbb{R}^d}\Big(r(z)+\frac{1}{2\eta}\|z-x\|^2\Big)
			~~\text{and}~~
			\proximal_{m\eta,R}(\xb)=\argmin_{\zb\in\mathbb{R}^{m\times d}}\Big(R(\zb)+\frac{1}{2m\eta}\|\zb-\xb\|^2\Big).
		\end{split}
	\end{equation}
	Using above notations, we define the (aggregated) generalized gradients as 
	\begin{equation}
		\label{eq:g_grad}
		G_t=\eta^{-1}\left(\yb_t- \proximal_{\eta  m,R}(\yb_t-\eta\bs_t)\right)~~\mbox{and}~~ G_t^{(i)}=\eta^{-1}\left(\yb_t^{(i)}- \proximal_{\eta  ,r}(\yb_t^{(i)}-\eta\bs_t^{(i)})  \right).
	\end{equation} 
	We also denote
	\begin{align}
		\bbG_t = \frac{1}{m}\sum_{i=1}^m G_t^{(i)}. \label{eq:bbG}
	\end{align}
 
	Then we introduce the following definitions that will be used in the whole paper:
	\begin{itemize}
		\item  We say $f(x)$ is $L$-smooth if for all $x,y\in\RR^{d}$, it holds that
		\begin{equation*}
			f(y) \leq f(x) + \dotprod{\nabla f(x), y - x} + \frac{L}{2}\norm{y - x}^2.
		\end{equation*}
		\item We say $f(x)$ is $\mu$-strongly convex, if for all $x,y\in\RR^{d}$, it holds that
		\begin{equation*}
			f(y) \geq f(x) + \dotprod{\nabla f(x), y - x} + \frac{\mu}{2}\norm{y - x}^2.
		\end{equation*}
		\item We say $f_i(x)$ is locally $M_i$-smooth if for all $x,y\in\RR^{d}$, it holds that
		\begin{equation*}
			f_i(y) \leq f_i(x) + \dotprod{\nabla f_i(x), y - x} + \frac{M_i}{2}\norm{y - x}^2.
		\end{equation*}
		\item We say $f_i(x)$ is locally $\nu_i$-strongly convex if for all $x,y\in\RR^{d}$,
		it holds that
		\begin{equation*}
			f_i(y) \geq f_i(x) + \dotprod{\nabla f_i(x), y - x} + \frac{\nu_i}{2}\norm{y - x}^2.
		\end{equation*}
	\end{itemize}
	Based on the smoothness and strong convexity, we can define global and local condition numbers of the objective function as 
	\begin{equation}
		\label{eq:kappa}
		\kappa_g = \frac{L}{\mu},  \qquad \hat{\kappa}_g = \frac{M}{\mu}\qquad\mbox{and} \qquad \kappa_\ell = \frac{M}{\nu},
	\end{equation}
	where
	\begin{equation}
		M = \max_{i\in\{1,\dots,m\}} M_i \qquad\mbox{and}\qquad \nu = \min_{i\in\{1,\dots,m\}} \nu_i. 
	\end{equation}
	It is easy to verify that
	\begin{equation}
		\label{eq:lg}
		L \le M\qquad
		\mbox{and}
		\qquad
		\kappa_g \leq \hat{\kappa}_g  \le \kappa_\ell.
	\end{equation}
	
	For the topology of the network, we let $W$ be the weight matrix associated with the network, indicating how agents are connected to each other.
	We assume that the weight matrix $W$ has the following properties:
	\begin{enumerate}
		\item $W$ is symmetric with $W_{i,j} \neq 0$ if and if only agents $i$ and $j$ are connected or $i=j$;
		\item ${\mathbf{0}}\preceq W\preceq I$, $W{\mathbf{1}}  = {\mathbf{1}}$, ${\rm null}(I - W) = {\rm span}(\mathbf{1})$;
	\end{enumerate}
	where we use $I$ to denote the $m\times m$ identity matrix and ${\mathbf{1}} = [1,\dots,1]^\top\in\RR^m$ denotes the vector with all ones.
	The weight matrix has an important property that $W^\infty = \frac{1}{m}\mathbf{1}\mathbf{1}^\top$ \citep{xiao2004fast}.
	Thus, one can achieve the effect of averaging local $x_i$ on different
	agents by using $W\xb$ for iterations.
	Instead of directly multiplying $W$, \citet{liu2011accelerated} proposed a more efficient way to achieve averaging described in Algorithm~\ref{alg:mix}, which has the following important proposition.
	\begin{proposition}
		\label{lem:mix_eq}
		Let $\xb^K$ be the output of Algorithm~\ref{alg:mix} with $\eta_w = 1/(1+\sqrt{1-\lambda_2^2(W)})$ and we denote~$\bar{x} = \frac{1}{m}\mathbf{1}^\top\xb^0$.
		Then it holds that 
		\[
		\bar{x} = \frac{1}{m}\mathbf{1}^\top\xb^K \quad\mbox{and}\quad 
        \norm{\xb_K - \mathbf{1}\bar x} 
        \leq \sqrt{14} \left(1 - \left(1-\frac{1}{\sqrt{2}}\,\right) \sqrt{1-\lambda_2(W)}\right)^K \norm{\xb_0 - \mathbf{1}\bar x},
		\]
		where $\lambda_2(W)$ is the second largest eigenvalue of $W$. 
	\end{proposition}

	\begin{algorithm}[tb]
		\caption{Mudag}
		\label{alg:DAGD}
		\begin{small}
			\begin{algorithmic}[1]
				\STATE {\bf Input:} $x_0$,   $\eta$,  $\alpha$, and $K$ \\[0.15cm]
				\STATE {\bf Initialization:} $\xb_0 = \mathbf{1} x_0$, $\yb_0 = \xb_0$ \\[0.15cm]
				\STATE 
				$\xb_{1} = \mathrm{FastMix} (\yb_0 - \eta \nabla F(\yb_0))$ \\[0.15cm]
				\STATE 
				$\yb_1 = \xb_1 + \frac{1-\alpha}{1+\alpha}(\xb_1 - \xb_0)$ \\[0.15cm]
                \STATE \textbf{for} $t=1,\dots, T$  \textbf{do} \\[0.1cm]
				\STATE \quad $\xb_{t+1} = \mathrm{FastMix}(\yb_t + (\xb_t - \yb_{t-1}) - \eta (\nabla F(\yb_t) - \nabla F(\yb_{t-1})),\; K)$   \label{stp:mudag} \\[0.15cm]
				\STATE \quad $\yb_{t+1} = \xb_{t+1} + \frac{1-\alpha}{1+\alpha}\left(\xb_{t+1} - \xb_t\right)$ \\[0.15cm]
				\STATE \textbf{end for} \\[0.15cm]
				\STATE {\bf Output:} $\xb_T$
			\end{algorithmic}
		\end{small}
	\end{algorithm}	
	
	\begin{algorithm}[tb]
		\caption{FastMix}
		\label{alg:mix}
		\begin{small}
			\begin{algorithmic}[1]
				\STATE {\bf Input:} $\xb^{0}$, $K$, $W$ and $\eta_w$ \\[0.15cm]
                \STATE $\xb^{-1} = \xb^{0}$ \\[0.15cm]
				\STATE\textbf{for} $k=0,\dots, K$ \textbf{do} \\[0.15cm]
				\STATE\quad $\xb^{k+1} = (1+\eta_w)W\xb^k - \eta_w \xb^{k-1}$ \\[0.15cm]  
				\STATE\textbf{end for} \\[0.15cm]
				\STATE {\bf Output:} $\xb^K$
			\end{algorithmic}
		\end{small}
	\end{algorithm}

	\section{Multi-Consensus Decentralized Accelerated Gradient Descent}
	
	In this section, we propose two novel decentralized  algorithms achieving the optimal computation complexity and near optimal communication complexity. 
	These two algorithms are suitable for the case of $r(x) = 0$ and the case of $r(x)$ is general convex respectively.

	\subsection{Algorithms and Main Ideas} 
	\label{subsec:main_idea}
	Our algorithms are based on the multi-consensus, gradient-tracking and Nesterov's acceleration technique. 
	We first introduce \dapg~(Algorithm~\ref{alg:DAGD_p}) for solving the problem with $r(x)\neq 0$. 
	It has the following algorithmic procedure: 
	\begin{align}
		\xb_{t+1} =& \proximal_{\eta m,R}( \yb_t - \eta \bs_t), \label{eq:pnag1} 
		\\
		\yb_{t+1}=&\mathrm{FastMix}\left( \xb_{t+1}+\frac{1-\alpha}{1+\alpha}(\xb_{t+1}-\xb_{t}),K\right),  \label{eq:pnag2}
		\\
		\bs_{t+1} = &\mathrm{FastMix} (\bs_t + \nabla F(\yb_{t+1}) - \nabla F(\yb_t),\; K), \label{eq:pnag3}
	\end{align}
	where $\eta$ is the step size and $K$ is the step number in multi-consensus. 
	We can observe that Eqs.~\eqref{eq:pnag1} and~\eqref{eq:pnag2} belong to the algorithmic framework of accelerated proximal gradient descent since $\bs_t$ can approximate the average gradient.  
	In Eq.~\eqref{eq:pnag3}, we introduce $\bs_t$ to track the gradient by using history information and the gradient difference.  Thus, $\bs_t$ can well approximate the average gradient $\mathbf{1}\bbg_t$ (defined in Eq.~\eqref{eq:xyg}).
	Furthermore, the variable $\yb_t$ can also approximate $\mathbf{1}\bby_t$ well by the ``FastMix'' operator. 
	Since both $\bby_t$ and $\bbg_t$ well approximate the averages, then we can obtain that $\bbg_t \approx \nabla f(\bby_t)$.
	Thus, the convergence properties of our algorithm are similar to the centralized accelerated proximal gradient descent, which is the main idea behind our approach to the decentralized optimization. 
	In other words, we combine multi-consensus with gradient-tracking to approximate the centralized accelerated proximal gradient descent.
	As we will show, this seemingly simple idea leads to establishing a near optimal algorithm for the decentralized optimization.
	Note that Algorithm~\ref{alg:DAGD_p} only takes two multi-consensus steps at each iteration. In contrast, the algorithm in conference version~\citep[Algorithm 1]{Ye2020} of this paper requires three multi-consensus steps at each round.
	Though reducing one multi-consensus step will not improve the order of communication complexity, it requires much less communication cost and benefits in real applications.

	\begin{algorithm}[tb]
		\caption{ProxMudag}
		\label{alg:DAGD_p}
		\begin{small}
			\begin{algorithmic}[1]
				\STATE {\bf Input:} $x_0$, $\eta$, $\alpha$, $K$ \\[0.15cm]
				\STATE {\bf Initialization:} $\xb_0 = \mathbf{1} x_0$,  $\yb_0 = \xb_0 $, $\bs_0 = \nabla F(\xb_0)$ \\[0.15cm]
				\STATE \textbf{for} $t=0,\dots, T$  \textbf{do} \\[0.15cm]
				\STATE \quad $\xb_{t+1} =\proximal_{\eta m,R}( \yb_t - \eta \bs_t) $  \\[0.15cm]
				\STATE \quad $\yb_{t+1}=\mathrm{FastMix}\left( \xb_{t+1}+\frac{1-\alpha}{1+\alpha}(\xb_{t+1}-\xb_{t}),K\right)$  \\[0.15cm]
				\STATE \quad $\bs_{t+1} = \mathrm{FastMix} (\bs_t + \nabla F(\yb_{t+1}) - \nabla F(\yb_t),\; K)$  \\[0.15cm]
				\STATE \textbf{end for} \\[0.15cm]
				\STATE {\bf Output:} $\xb_T$
			\end{algorithmic}
		\end{small}
	\end{algorithm}
	
	In the case of $r(x)=0$, we propose \mudag~(Algorithm~\ref{alg:DAGD}) that only needs one multi-consensus step for each iteration. 
	The \mudag~has the following algorithmic procedure:
	\begin{align}
		\xb_{t+1} =& \mathrm{FastMix}(\yb_t + (\xb_t - \yb_{t-1}) - \eta (\nabla F(\yb_t) - \nabla F(\yb_{t-1})),\; K), \label{eq:xb_up}
		\\
		\yb_{t+1} =& \xb_{t+1} + \frac{1-\alpha}{1+\alpha}\left(\xb_{t+1} - \xb_t\right). \nonumber
	\end{align}
	To understand \mudag~from perspective of gradient tracking, we can reformulate the above procedure in a form similar to Eqs.~\eqref{eq:pnag1} to~\eqref{eq:pnag3} as follows (The reformulation is proved in Lemma~\ref{lem:proc})
	\begin{align}
		\xb_{t+1} =& \mathrm{FastMix}\left(\yb_t - \eta \bs_t, K\right), \label{eq:x_y}
		\\
		\yb_{t+1} =& \xb_{t+1} + \frac{1-\alpha}{1+\alpha}(\xb_{t+1} - \xb_t), \label{eq:x_y_1}
		\\
		\bs_{t+1} =& \mathrm{FastMix}(\bs_t  , K) + (\nabla F(\yb_{t+1}) -\nabla F(\yb_t))- \eta^{-1}(\mathrm{FastMix}(\yb_t, K) - \yb_t). \label{eq:s}
	\end{align} 
	Note that Eq.~\eqref{eq:s} is an explicit gradient tracking step similar to Eq.~\eqref{eq:pnag3}.
	Comparing Eqs.~\eqref{eq:x_y}-\eqref{eq:s} with Eqs.~\eqref{eq:pnag1}-\eqref{eq:pnag3}, we can observe that these two algorithms share a similar procedure since they share the same intuition.
	However,  the iteration of \dapg~cannot be improved to one multi-consensus step like \mudag. 
	If we directly replace Eq.~\eqref{eq:xb_up} by
	\begin{align*}
		\xb_{t+1} =& \mathrm{FastMix}\left( \proximal_{\eta  m,R} \left(\yb_t + (\xb_t - \yb_{t-1}) - \eta (\nabla F(\yb_t) - \nabla F(\yb_{t-1}))\right),\, K\right),
	\end{align*}
	it is easy to check that the algorithm cannot converge to the optimum. 
	
	Because \mudag~only has one multi-consensus step for each iteration while \dapg~takes two multi-consensus steps, in practice, \mudag~commonly requires much less communication cost than \dapg~when \mudag~is applicable.
	Thus, the \mudag~is a better choice than \dapg~in the case of $r(x) = 0$.

	\subsection{Main Results}

	In this work, we focus on the synchronized setting in which the computation complexity depends on the number of gradient calls and the communication complexity depends on the rounds of local communication. 
	We give the detailed upper complexity bounds for our algorithms in the following theorems.
	\begin{theorem}
		\label{thm:main_1}
		Let $f(x)$ be $L$-smooth and $\mu$-strongly convex. Assume each $f_i(x)$ is $M$-smooth. 
		We set~$\eta = {1}/{L}$ and $\alpha = \sqrt{\mu\eta}$ in Algorithm~\ref{alg:DAGD}.
		Letting $K$ in Algorithm~\ref{alg:DAGD} satisfy that
		\begin{equation*}
			K = \frac{\sqrt{2}}{\sqrt{2}-1}\sqrt{\frac{1}{1-\lambda_2(W)}}\log\left(\frac{\sqrt{14}}{\rho}\right)\qquad\mbox{with}\qquad
			\rho \le \frac{1}{4^3\cdot9\cdot 288 } \cdot \left(\frac{L}{M}\right)^4 \kappa_g^{-3},
		\end{equation*}
		then the sequence $\{\bbx_t\}$ satisfies that
		\begin{equation*}
        \small\begin{split}
			f(\bbx_T) - f(x^*) 
			\le
			\left(1 - \frac{\alpha}{2}\right)^T\left(f(\bbx_0) - f(x^*) + \frac{\mu\norm{\bbx_0 - x^*}^2}{2} + \frac{\mu}{288m}\sum_{i=1}^m\norm{\nabla f_i(\bbx_0) - \nabla f(\bbx_0)}^2 \right),
        \end{split}
		\end{equation*}
		where $x^*$ is the global minimum of $f(x)$.
		To achieve $\xb_T$ such that $\norm{\xb_T - \mathbf{1}x^*}^2  = \cO(m\epsilon/\mu)$ and~$f(\bbx_T) - f(x^*)\le \epsilon$, the computation and communication complexities of Algorithm~\ref{alg:DAGD}  are at most
		\begin{equation*}
			T =
			\cO\left(\sqrt{\kappa_g}\log\left(\frac{1}{\epsilon}\right)\right)\qquad
			\mbox{and}\qquad Q = \cO\left(\sqrt{\frac{\kappa_g}{1 -
					\lambda_2(W)}}\log\left(\frac{M\kappa_g}{L} \right)
			\log\left(\frac{1}{\epsilon}\right) \right)
		\end{equation*}
		respectively.
	\end{theorem}
	
	\begin{theorem}
		\label{thm:main}
		Let $f(x)$ be $L$-smooth and $\mu$-strongly convex. Assume each $f_i(x)$ is $M$-smooth. 
		We set~$\eta = {1}/{(2L)}$ and $\alpha = \sqrt{\mu\eta}$ in Algorithm~\ref{alg:DAGD_p}.
		Letting $K$ in Algorithm~\ref{alg:DAGD_p} satisfy that
		\begin{equation*}
			K = \frac{\sqrt{2}}{\sqrt{2}-1}\sqrt{\frac{1}{1-\lambda_2(W)}}\log\left(\frac{\sqrt{14}}{\rho}\right)\qquad\mbox{with}\qquad 
			\rho \le \frac{1}{ 5.5\cdot 10^8} \cdot \left(\frac{L}{M}\right)^6 \kappa_g^{-3/2},
		\end{equation*}
		then sequence $\{\bbx_t\}$ generated by Algorithm~\ref{alg:DAGD_p} satisfies that
		\begin{equation*}
            \small\begin{split}
			h(\bbx_T) - h(x^*) 
			\le
			\left(1 - \frac{\alpha}{2}\right)^T\left(h(\bbx_0) - h(x^*) + \frac{\mu\norm{\bbx_0 - x^*}^2}{2} + \frac{52L}{m}\sum_i^m\norm{\nabla f_i(\bbx_0) - \nabla f(\bbx_0)}^2 \right),
            \end{split}
		\end{equation*}
		where $x^*$ is the global minimum of $h(x)$.
		To achieve $\xb_T$ such that $\norm{\xb_T - \mathbf{1}x^*}^2  = \cO(m\epsilon/\mu)$ and~$h(\bbx_T) - h(x^*)\le \epsilon$, the computation and communication complexities of Algorithm~\ref{alg:DAGD}  are at most
		\begin{equation*}
			T =
			\cO\left(\sqrt{\kappa_g}\log\left(\frac{1}{\epsilon}\right)\right)\qquad
			\mbox{and}\qquad Q = \cO\left(\sqrt{\frac{\kappa_g}{1 -
					\lambda_2(W)}}\log\left(\frac{M\kappa_g}{L} \right)
			\log\left(\frac{1}{\epsilon}\right)\right)
		\end{equation*}
		respectively.
	\end{theorem}

	\begin{remark}
		Theorem~\ref{thm:main_1} shows that \mudag~ achieves the same order of computation complexity as that of the centralized Nesterov's accelerated gradient descent.
		At the same time, the communication complexity nearly
		matches the known lower bound of decentralized optimization
		problem up to a factor of $\log \left({M\kappa_g}/{L}\right)$.
		We conjecture that it may be possible to remove the
		$\log(\kappa_g)$ factor,
		because the term only comes from the inequality $\norm{\bby_t
			- x^*} \le \sqrt{2V_t/{\mu}}$, where $V_t$ is defined in
		Eq.~\eqref{eq:V_t} in the proof, which may be loose. 
	\end{remark}
	
	\begin{remark}
		Theorem~\ref{thm:main_1} and~\ref{thm:main} only assume that $f(x)$ is
		$\mu$-strongly convex and $L$-smooth, and $f_i(x)$ is
		$M$-smooth (note that unlike many previous works, our
		dependency on $M$ is logarithmic only).
		Thus, our algorithms can be used when $f_i(x)$ is possibly non-convex.
		This kind of problem has been widely studied in recent years \citep{allen2018katyusha,garber2016robust} and one important example is the fast PCA by shift-invert method \citep{garber2016robust}.
		In contrast, the previous works~\citep{seaman2017optimal,li2018sharp,li2019decentralized,qu2019accelerated,kovalev2020optimal,li2021accelerated} require the (strong) convexity of $f_i(x)$  to obtain the linear convergence rate.
	\end{remark}
	
	\begin{remark}
		The computation and communication complexities of our
		algorithms depend on $\sqrt{\kappa_g}$ rather than $\sqrt{\kappa_\ell}$, which is a novel result. 
		Before our work, it was unknown whether there exists a decentralized algorithm that can achieve a communication complexity close to the lower bound~$\Omega\left(\sqrt{{\kappa_g}/{(1 - \lambda_2(W))}} \log({1}/{\epsilon})\right)$~\citep{seaman2017optimal,scaman2019optimal}.
	\end{remark}

	\begin{remark}
		Observe that the step 3 of Algorithm~\ref{alg:DAGD} resorts to  multi-consensus and gradient tracking to encourage $\xb(i,:)$ on different agents to be close to each other. 
		Similarly, for the centralized distributed optimization problem, the consensus step is also needed, which is often implemented by two rounds of communications between agents and the central server. 
		In this view, the centralized optimization methods and the decentralized one only differ in the way to achieve consensus.
		We can also regard the decentralized optimization methods as an approximation to the decentralized one.
	\end{remark}

	\section{Convergence Analysis}
	\label{sec:converg}
	
	In this section, we give a detailed characterization on how our decentralized algorithms approximate accelerated (proximal) gradient descent.
	Since \mudag~ and \dapg~ have similar ideas for convergence analysis, we only present how to obtain the convergence rate of \dapg~ in this section and leave the analysis of \mudag~in Appendix~\ref{app:mudag}.
	Note that the analysis of \dapg~may be more sophisticated than the one of \mudag~because of the additional step of proximal operation. 
	
	We first introduce the Lyapunov function as follows
	\begin{equation}
		\label{eq:V_t}
		V_t = h(\bbx_t) - h(x^*) + \frac{\mu}{2}\norm{\bbv_t - x^*}^2,
	\end{equation}
	where $\bbv_t$ is defined as
	\begin{equation}
		\label{eq:v_t}
		\bbv_t = \bbx_{t-1} + \frac{1}{\alpha}(\bbx_t - \bbx_{t-1}) \quad\mbox{with}\qquad \alpha = \sqrt{\mu \eta}.
	\end{equation}
	In the rest of this section, we will show how the Lyapunov function
	$V_t$ converges and how multi-consensus and gradient-tracking help us
	to approximate centralized accelerated proximal gradient descent.
	
	Then we show that $\bbx_t$, $\bby_t$, $\bbg_t$ (defined in
	Eq.~\eqref{eq:xyg} and generated by Algorithm \ref{alg:DAGD_p}) and
	$\bbv_t$ (defined in Eq.~\eqref{eq:v_t}) can be fit into the framework of the centralized  accelerated proximal gradient descent.
	\begin{lemma}
		\label{lem:equalities}
		Let $\bbx_t$, $\bby_t$, $\bbg_t$ and $\bbG_t$ (defined in Eqs.~\eqref{eq:xyg} and~\eqref{eq:bbG}) be generated by Algorithm~\ref{alg:DAGD_p}. 
		By setting~$\bbs_t = \frac{1}{m} \mathbf{1}^\top \bs_t$ with  $\bs_t$ defined in Eq.~\eqref{eq:pnag3}, 
		it satisfies:
		\begin{align}
			\bbx_{t+1} =&
			\bby_t - \eta\bbG_t,    \label{eq:xy_1}
			\\
			\bby_{t+1} =& \bbx_{t+1} + \frac{1-\alpha}{1+\alpha}(\bbx_{t+1} - \bbx_t) \label{eq:yx},\\
			\bbs_{t+1}=& \bbs_t+\bbg_{t+1} - \bbg_t = \bbg_{t+1}. \label{eq:bs}
		\end{align}
	\end{lemma} 
	\begin{proof}
		Proposition~\ref{lem:mix_eq} provides a property of  FastMix that 
		$\bbx = \frac{1}{m}\mathbf{1}^\top \mathrm{FastMix}(\xb, K)$.
		Combining the algorithmic procedures of Algorithm~\ref{alg:DAGD_p} and the definition of $\bbG_t$ in Eq.~\eqref{eq:bbG}, we can obtain the first two equations.
		
		We first the last equality by induction. For $t=0$, we use the fact that $\bs_0 = \nabla F(\yb_0)$.  
		Then, it holds that $\bbs_0 = \bbg_0$.
		We assume that $\bbs_t = \bbg_t$ at time $t$. 
		By the update equation \eqref{eq:pnag3} and Proposition~\ref{lem:mix_eq}, we have
		\begin{equation*}
			\bbs_{t+1} = \bbs_t + (\bbg_{t+1} - \bbg_t) = \bbg_{t+1}.
		\end{equation*}
		Thus, we obtain the result at time $t+1$. 		
	\end{proof}

	Lemma~\ref{lem:equalities} shows that the averaged version of  Eqs.~\eqref{eq:pnag1}-\eqref{eq:pnag3} is almost the same as   accelerated proximal gradient descent \citep{nesterov2018lectures}.
	Thus, if $\bbs_t$ is an accurate estimation of $\nabla f(\bby_t)$, then Algorithm~\ref{alg:DAGD_p} has convergence properties similar to accelerated proximal gradient descent.
	Next, we are going to show $\yb_t(i,:) \approx \bby_t$ and $\bs_t(i,:)\approx \bbs_t$ by the following lemma.
	\begin{lemma}
		\label{lem:yvs}
		Let $\zb_t=[\norm{\xb_{t}-\mathbf{1}\bbx_{t}},\norm{\yb_{t}-\mathbf{1}\bby_{t}}, \eta \norm{\bs_{t}-\mathbf{1}\bbs_{t}}]^\top$ with $\xb_t$, $\yb_t$ and $\bs_t$ generated by Algorithm~\ref{alg:DAGD_p}, then it holds that
		\begin{equation}
			\label{eq:zaz}
			\zb_{t+1}\le \bA\zb_t+ \frac{4 \rho\sqrt{m} M}{L}\left[0,0,\sqrt{\frac{2V_t}{\mu}}\,\right]^\top,
		\end{equation}
		where $\rho$ and $\mathbf{A}$ are defined as 
		\begin{equation} \label{eq:A}
			\rho = \left(1 - \sqrt{1 - \lambda_2(W)}\,\right)^K 
			\quad \text{and}\quad
			\bA=\begin{bmatrix}
				0 & 2 & 2\\
				2\rho& 4\rho &4\rho \\
				\rho M/L&   8\rho M^2/L^2   & 5 \rho M/L 
			\end{bmatrix}.
		\end{equation}
	\end{lemma}
	If the spectral radius of $\mathbf{A}$ is less than $1$ and $V_t$ converges to zero, then $\norm{\zb_t}$ will converge to zero.
	Note that $\norm{\xb_t - \mathbf{1}\bbx_t}$, $\norm{\yb_t-\mathbf{1}\bby_t}$ and $\eta \norm{\bs_t-\mathbf{1}\bbs_t}$ are no larger than $\norm{\zb_t}$.
	Hence, Algorithm~\ref{alg:DAGD_p} can well approximate centralized accelerated proximal gradient descent in such conditions.
	
	Next, we prove above two conditions that lead to the convergence of $\norm{\zb_t}$. 
	First, the following lemma shows the spectral radius of $\mathbf{A}$ is less than $\frac{1}{2}$ if $\rho$ is small enough.
	\begin{lemma}
		\label{lem:lam_max}
		Matrix  ${\mathbf{A}}$ defined in Eq.~\eqref{eq:A} satisfies that
		\begin{equation*}
			0<\lambda_1(\mathbf{A}) \qquad \text{and} \qquad |\lambda_3(\mathbf{A})| \le |\lambda_2(\mathbf{A})|< \lambda_1(\mathbf{A})
		\end{equation*}
		with $\lambda_i(\mathbf{A})$ being the $i$-th largest eigenvalue of $\mathbf{A}$. 
		Letting $\eta = 1/(2L)$ and $\rho \le {L^3}/{(1280 M^3)}$,
		then it holds that  
		\begin{equation*}
			\lambda_1({\mathbf{A}})\le \frac12,
		\end{equation*}
		and the eigenvector $\bv$ associated with $\lambda_1(\mathbf{A})$ is positive and its entries satisfy
		\begin{equation}
			\label{eq:v_v}
			0<\bv(1) \le \frac{8\bv(3)}{\sqrt{\rho}}, \qquad 0<\bv(2) \le 3\bv(3) \qquad \text{and} \quad 0<\bv(3),
		\end{equation}  
		where $\bv(i)$ is $i$-th entry of $\bv$.
	\end{lemma}
	Now, we are going to show  $V_t$ converges linearly but with some perturbation terms related to $\zb_t$.
	\begin{lemma} \label{lem:VV}
		Letting $\xb_t, \yb_t, \bs_t$ be generated by Algorithm~\ref{alg:DAGD_p}, it holds that
		\begin{align}
			V_{t+1} 
			\le 
			(1 - \alpha) V_t  
			+ \frac{13\eta}{m} \norm{\bs_t - \mathbf{1}\bbs_t}^2 + \frac{20M^2 \eta + 10\eta^{-1}}{m} \norm{\yb_t - \mathbf{1}\bby_t}^2. \label{eq:VV}
		\end{align}
	\end{lemma}
	The above lemma shows that the Algorithm~\ref{alg:DAGD_p} has a convergence property similar to the accelerated proximal gradient descent but with some perturbation terms.
	Next, using above lemmas and choosing proper $\rho$ by a proper $K$, we will obtain the convergence rate of Algorithm~\ref{alg:DAGD_p}.

    Finally, we provide the proof of our main result Theorem~\ref{thm:main}.
	\begin{proof}
		It is easy to check that $\rho$ satisfies the conditions required in Lemma~\ref{lem:lam_max}.
		Let the eigenvector $\bv$ defined in Lemma~\ref{lem:lam_max} and set $\bv(3)=1$. 
		Combining with the fact that the first two entries of $\zb_0$ are zero, we can obtain that,
		\begin{equation*}
			\zb_0 \le \norm{\zb_0}\bv \qquad \mbox{and}\qquad [0,0, 1]^\top \le \bv.
		\end{equation*} 
		By Eq.~\eqref{eq:zaz}, we can obtain that
		\begin{equation}
			\label{eq:nm_z}
			\begin{aligned}
				\zb_{t+1} 
				\le& 
				\norm{\zb_0} \cdot\mathbf{A}^{t+1}\bv 
				+ 
				\frac{4\rho M}{L}\sqrt{\frac{2m}{\mu}}\cdot\sum_{i=0}^t\sqrt{V_i}\cdot\mathbf{A}^{t-i}\bv
				\\
				=&
				\norm{\zb_0}\lambda_1(\mathbf{A})^{t+1}\bv 
				+
				\frac{4\rho M}{L}\sqrt{\frac{2m}{\mu}}\cdot\sum_{i=0}^t\sqrt{V_i}\cdot\lambda_1(\mathbf{A})^{t-i}\bv 
				\\
				\le&
				\norm{\zb_0}\left(\frac{1}{2}\right)^{t+1}\bv 
				+
				\frac{4\rho M}{L}\sqrt{\frac{2m}{\mu}}\cdot\sum_{i=0}^t\left(\frac{1}{2}\right)^{t-i}\sqrt{V_i}\cdot\bv,
			\end{aligned}
		\end{equation}
		where the first equality is because $\bv$ is the eigenvector associated with $\lambda_1(\mathbf{A})$ and the last inequality is because of $\lambda_1(\bA) \le \frac{1}{2}$ obtained in Lemma~\ref{lem:lam_max}.
		
		Next, we will prove our result by induction.  For $t = 0$, we have $\norm{\yb_0 - \mathbf{1}\bby_0} = 0 $ and 
		\begin{align*}
			V_1 
			\stackrel{\eqref{eq:VV}}{\le}&
			(1 - \alpha) V_0 + \frac{13\eta^{-1}}{m} (\eta\norm{\bs_0 -\mathbf{1} \bbs_0})^2 
			= (1 - \alpha) V_0 + \frac{13\eta^{-1}}{m} \norm{\zb_0}^2
			\\
			\le& (1 - \alpha) V_0 + \left(1 - \frac{\alpha}{2}\right) \frac{52L}{m} \norm{\zb_0}^2
			\le \left(1 - \frac{\alpha}{2}\right) \left( V_0 + \frac{52L}{m} \norm{\zb_0}^2 \right).
		\end{align*}
		Next, we assume that for $i= 1,\dots, t$,  it holds that 
		\begin{equation}
			\label{eq:assmp}
			V_i \leq \left(1 - \frac{\alpha}{2}\right)^i \left(V_0+\frac{52L}{m}\norm{\zb_0}^2\right).
		\end{equation}
		Combining with Eq.~\eqref{eq:nm_z}, we can obtain that
		\begin{equation}
			\label{eq:z}
			\begin{aligned}
				& \zb_{t-1} \\
				\le&
				\bv\cdot\left(4\rho \frac{M}{L} \sqrt{\frac{2m}{\mu}} \sum_{j=0}^{t-2} 2^{-(t-2-j)} \sqrt{V_j}
				+
				2^{-(t-1)}\norm{\zb_0}
				\right)
				\\
				\le&
				\bv \cdot
				\left(
				4\rho \frac{M}{L} \sqrt{\frac{2m}{\mu}}\sum_{j=0}^{t-2} 2^{-(t-2-j)} \left(\sqrt{1 - \frac{\alpha}{2}}\right)^j \sqrt{V_0+\frac{52L}{m} \norm{\zb_0}^2}
				+
				2^{-(t-1)}\norm{\zb_0}
				\right)
				\\
				=&
				\bv\cdot
				\left(4\rho\frac{M}{L}\sqrt{\frac{2m}{\mu}}\frac{2\left(\sqrt{1 - \frac{\alpha}{2}} \right)^{t-1}- 2^{-(t-2)}}{2\sqrt{1-\frac{\alpha}{2}}-1} \sqrt{V_0+\frac{52L}{m} \norm{\zb_0}^2}
				+
				2^{-(t-1)}\norm{\zb_0}
				\right)
				\\
				\le&
				\bv\cdot
				\left(
				12\frac{M}{L}\sqrt{\frac{2m}{\mu}}\left(\sqrt{1 - \frac{\alpha}{2}} \right)^{t-1} \sqrt{V_0+\frac{52L}{m} \norm{\zb_0}^2}
				+
				2^{-(t-1)}\norm{\zb_0}
				\right).
			\end{aligned}
		\end{equation}
		Thus, using the definition of $\zb$ and $\bA$, we can obtain that
		\begin{align*}
			& \norm{\yb_{t} - \mathbf{1}\bby_{t}} \\ 
			\stackrel{\eqref{eq:zaz}}{\le}&
			\dotprod{[2\rho, 4\rho, 4\rho ], \zb_{t-1}}
			\\
			\overset{\eqref{eq:z}}{\le}&
			\dotprod{[2\rho, 4\rho, 4\rho ], \bv} \cdot \left(
			\frac{12M}{L}\sqrt{\frac{2m}{\mu}}\left(\sqrt{1 - \frac{\alpha}{2}} \right)^{t-1} \sqrt{V_0+\frac{52L}{m} \norm{\zb_0}^2}
			+
			2^{-(t-1)}\norm{\zb_0}
			\right)
			\\
			\le& \rho\left(2\bv(1) + 4\bv(2) + 4\right) \cdot
			\left(
			\frac{12M}{L}\sqrt{\frac{2m}{\mu}}\left(\sqrt{1 - \frac{\alpha}{2}} \right)^{t-1} \sqrt{V_0+\frac{52L}{m} \norm{\zb_0}^2}
			+
			2^{-(t-1)}\norm{\zb_0}
			\right)
			\\
			\stackrel{\eqref{eq:v_v}}{\le}&
			\rho\cdot \left( \frac{16}{\sqrt{\rho}} + 12 + 4 \right) 
			\cdot
			\left(
			\frac{12M}{L}\sqrt{\frac{2m}{\mu}}\left(\sqrt{1 - \frac{\alpha}{2}} \right)^{t-1} \sqrt{V_0+\frac{52L}{m} \norm{\zb_0}^2}
			+
			2^{-(t-1)}\norm{\zb_0}
			\right)
			\\
			\le&\sqrt{\rho} \cdot 32 \cdot 
			\left(
			\frac{12M}{L}\sqrt{\frac{2m}{\mu}}\left(\sqrt{1 - \frac{\alpha}{2}} \right)^{t-1} \sqrt{V_0+\frac{52L}{m} \norm{\zb_0}^2}
			+
			2^{-(t-1)}\norm{\zb_0}
			\right).
		\end{align*}
		Then we have
		\begin{equation}\label{eq:yy}
			\begin{aligned}
				&\frac{20M^2\eta + 10\eta^{-1}}{m} \norm{\yb_t - \mathbf{1}\bby_t}^2 \\
				=&  \frac{10M^2/L + 20 L}{m} \norm{\yb_t - \mathbf{1}\bby_t}^2
				\\
				\le&  \frac{30M^2/L}{m} \cdot 2 \cdot 32^2 \cdot \rho 
				\left(
				\frac{288m M^2}{L^2 \mu} \left( 1 - \frac{\alpha}{2}\right)^{t-1} \left(V_0+\frac{52L}{m} \norm{\zb_0}^2\right) + 4^{-(t-1)} \norm{\zb_0}^2
				\right).
			\end{aligned}
		\end{equation}	
		Similarly, we have 
		\begin{align*}
            \small\begin{split}      
			&\eta \norm{\bs_t - \mathbf{1}\bbs_t} \\
			\le&
			\rho \dotprod{\left[\frac{M}{L}, \frac{8M^2}{L^2}, \frac{5M}{L}\right], \zb_{t-1}} + \frac{4\rho M}{L}\sqrt{\frac{2m}{\mu}V_{t-1}}
			\\
			\stackrel{\eqref{eq:assmp}\eqref{eq:z}}{\le}&
			\rho \left(\frac{24M^2}{L^2} + \frac{8M}{L\sqrt{\rho}}   + \frac{5M}{L}\right)\cdot
			\left(
			\frac{12M}{L}\sqrt{\frac{2m}{\mu}}\left(\sqrt{1 - \frac{\alpha}{2}}\,\right)^{t-1} \sqrt{V_0+\frac{52L}{m} \norm{\zb_0}^2}
			+
			2^{-(t-1)}\norm{\zb_0}
			\right)
			\\
			&+
			\frac{4\rho M}{L} \sqrt{\frac{2m}{\mu}}\left(\sqrt{1 - \frac{\alpha}{2}}\,\right)^{t-1} \sqrt{V_0+\frac{52L}{m} \norm{\zb_0}^2}\\		
			\le&\frac{72M^2 \sqrt{\rho}}{L^2}
			\cdot
			\left(
			\frac{12M}{L}\sqrt{\frac{2m}{\mu}}\left(\sqrt{1 - \frac{\alpha}{2}}\,\right)^{t-1} \sqrt{V_0+\frac{52L}{m} \norm{\zb_0}^2}
			+
			2^{-(t-1)}\norm{\zb_0}
			\right).
            \end{split}
		\end{align*}
		Consequently,  it holds that
		\begin{equation} \label{eq:ss}
			\begin{split}
				&\frac{13\eta}{m} \norm{\bs_t - \mathbf{1}\bbs_t}^2
				=
				\frac{13}{\eta m} \cdot \left( \eta \norm{\bs_t - \mathbf{1}\bbs_t} \right)^2
				\\
				\le&
				\frac{52M^4}{mL^3}  \cdot 72^2 \cdot \rho 
				\left(
				\frac{288m M^2}{L^2 \mu} \left( 1 - \frac{\alpha}{2}\right)^{t-1} \left(V_0+\frac{52L}{m} \norm{\zb_0}^2\right) + 4^{-(t-1)} \norm{\zb_0}^2
				\right).
			\end{split}
		\end{equation}
		Combining above results, we obtain 
		\begin{align*}
        \small\begin{split} 
			V_{t+1} & \\
			\stackrel{\eqref{eq:VV},\eqref{eq:yy},\eqref{eq:ss}}{\le} &
			\left(1 - \alpha\right) V_t \\
			&+  \frac{\left( 60\cdot 32^2 + 52\cdot 72^2 \right)\rho M^4}{mL^3} \left(
			\frac{288m M^2}{L^2 \mu} \left( 1 - \frac{\alpha}{2}\right)^{t-1} \left(V_0+\frac{52L}{m} \norm{\zb_0}^2\right) + 4^{-(t-1)} \norm{\zb_0}^2
			\right)
			\\
			\le~~~~~& (1 - \alpha) \left(1-\frac{\alpha}{2} \right)^t \left(V_0+ \frac{52L}{m}\norm{\zb_0}^2\right)
			+ 1.91\cdot10^8\cdot \frac{\rho M^6}{L^6} \kappa_{g} \left(1-\frac{\alpha}{2} \right)^t \left(V_0+ \frac{52L}{m}\norm{\zb_0}^2\right) \\
			& + 6366 \cdot  \frac{\rho M^2}{L^2} 4^{-(t-1)} \cdot  \frac{52L}{m} \norm{\zb_0}^2
			\\
			\le~~~~~&  (1 - \alpha) \left(1-\frac{\alpha}{2} \right)^t \left(V_0+ \frac{52L}{m}\norm{\zb_0}^2\right) + 1.92\cdot10^8\cdot \frac{\rho M^6}{L^6} \kappa_{g} \left(1-\frac{\alpha}{2} \right)^t \left(V_0+ \frac{52L}{m}\norm{\zb_0}^2\right)
			\\
			\le~~~~~& \left( 1 - \frac{\alpha}{2} \right)^{t+1} \cdot \left(V_0+ \frac{52L}{m}\norm{\zb_0}^2\right),
        \end{split}
		\end{align*}
		where the second inequality is because of the induction assumption and the last inequality is due to~$\rho \le {L^6}/\big(5.5\cdot10^8\cdot \kappa_g^{3/2} M^6\big)$.		
		Furthermore, it holds that
		\begin{align*}
			& \norm{\xb_t - \mathbf{1}\bbx_t} \\
			= &
			2 \dotprod{[0, 1, 1], \zb_t} \\
			\stackrel{\eqref{eq:z}}{\le} &
			2\dotprod{[0, 1, 1], \bv} \cdot  \left(
			\frac{12M}{L}\sqrt{\frac{2m}{\mu}}\left(\sqrt{1 - \frac{\alpha}{2}} \right)^{t-1} \sqrt{V_0+\frac{52L}{m} \norm{\zb_0}^2}
			+
			2^{-(t-1)}\norm{\zb_0} 
			\right)
			\\
			\stackrel{\eqref{eq:v_v}}{\le}& 8\cdot  \left(
			\frac{12M}{L}\sqrt{\frac{2m}{\mu}}\left(\sqrt{1 - \frac{\alpha}{2}} \right)^{t-1} \sqrt{V_0+\frac{52L}{m} \norm{\zb_0}^2}
			+
			2^{-(t-1)}\norm{\zb_0} 
			\right),
		\end{align*}
		which concludes the proof.
	\end{proof}

	\section{Experiments}
	
	We evaluate the performance of our algorithms on (sparse) logistic regression
	with different settings, including the situation in which each $f_i(x)$ is strongly convex and the local function $f_i(x)$ may be non-convex.
	
	\subsection{The Setting of Networks}
	
	In our experiments, we consider random networks where each pair of agents have a connection with a probability of $p$. We set $W = I - {\mathbf{L}}/{\lambda_1(\mathbf{L})}$, where $\mathbf{L}$ is the Laplacian matrix associated with a weighted graph, and $\lambda_1(\mathbf{L})$ is the largest eigenvalue of $\mathbf{L}$. 
	We also set $m = 100$, that is, there exist $100$ agents in this network.
	In our experiments, we run the algorithms on the settings of $p = 0.1$ and~$p = 0.5$, which correspond to $1-\lambda_2(W)=0.05$ and $1-\lambda_2(W)=0.81$ respectively.

	\subsection{Experiments on $\ell_2$-Regularized Logistic Regression}
	We consider the $\ell_2$-regularized logistic regression model whose local objective function of logistic regression is defined as
	\begin{equation}
		f_i(x) = \frac{1}{n}\sum_{j=1}^{n} \log \big(1+\exp(-b_{ij}\langle a_{ij}, x\rangle)\big) + \frac{\sigma_i}{2}\|x\|^2, \label{eq:log}
	\end{equation}
	where $a_{ij} \in \RR^{d}$ and $b_{ij}\in\{-1,1\}$ are the $j$-th input vector and the corresponding label on the $i$-th agent. 
	We conduct our experiments on a real-world dataset `a9a' which can be downloaded from LIBSVM repository~\citep{chang2011libsvm}.
	We set~$n=325$ and $d = 123$. 
	We conduct the following four experimental settings:
	\begin{enumerate}
		\item We set $\sigma_1=\dots=\sigma_m=10^{-3}$, then each $f_i(x)$ is strongly-convex.
		\item We set $\sigma_1=\dots=\sigma_m=10^{-4}$, then each $f_i(x)$ is strongly-convex.
		\item We set $\sigma_1=\dots=\sigma_{m-1} = -10^{-1}$ and $\sigma_m = 10$, then  functions $f_i(x)$ for $i<m$ are non-convex  but $f(x)$ is still strongly-convex.
		\item We set $\sigma_1=\dots=\sigma_{m-1} = -10^{-2}$ and $\sigma_m = 1$, then  functions $f_i(x)$  for $i<m$ are non-convex but $f(x)$ is still strongly-convex.
	\end{enumerate}
	
	\begin{figure*}[ht!]
		\centering
		\begin{tabular}{@{\extracolsep{0.0001em}}c@{\extracolsep{0.0001em}}c@{\extracolsep{0.0001em}}c@{\extracolsep{0.0001em}}c@{\extracolsep{0.0001em}}c}
			\includegraphics[width=38mm,keepaspectratio]{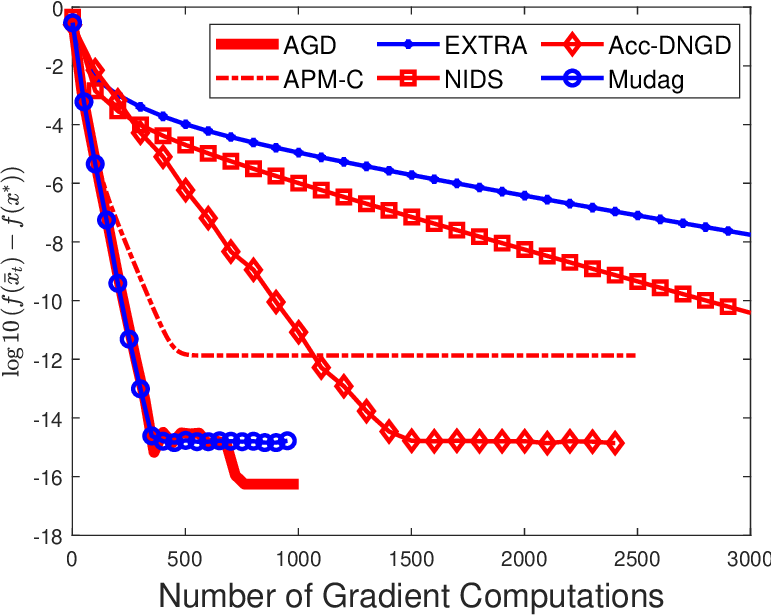}
			&\includegraphics[width=38mm,keepaspectratio]{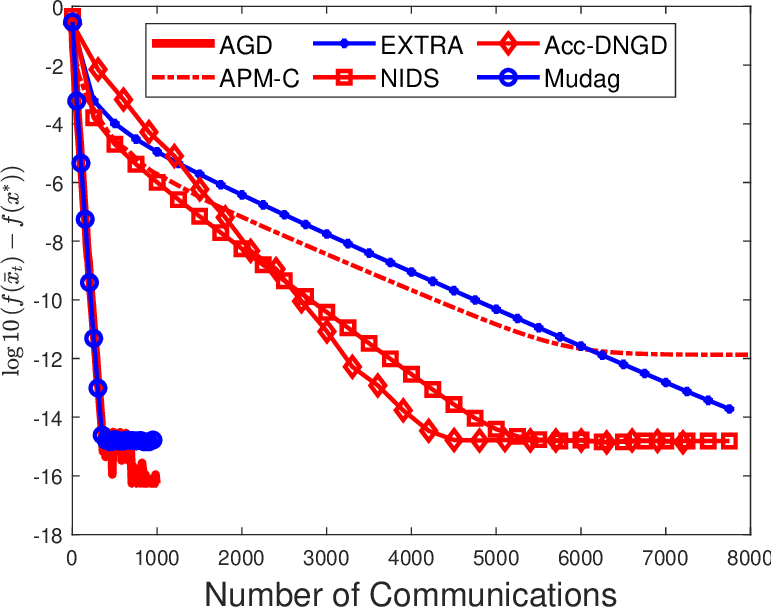}
			&\includegraphics[width=38mm,keepaspectratio]{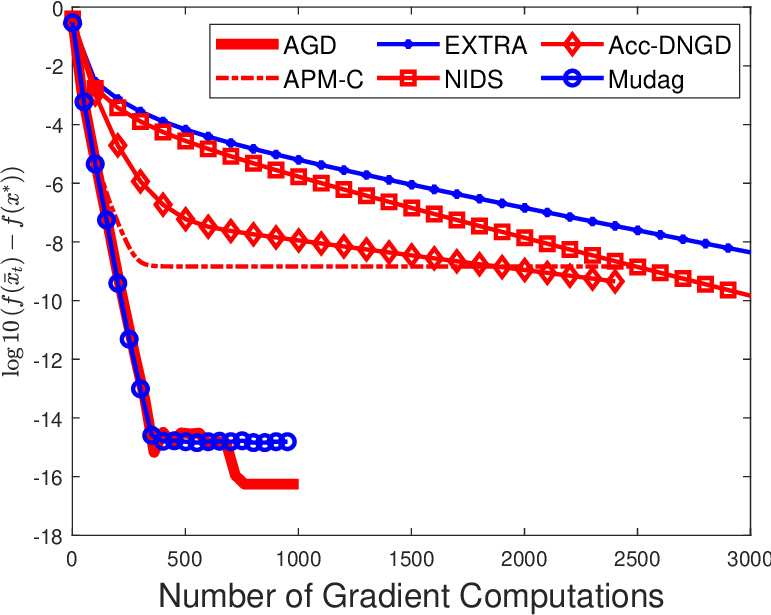}
			&\includegraphics[width=38mm,keepaspectratio]{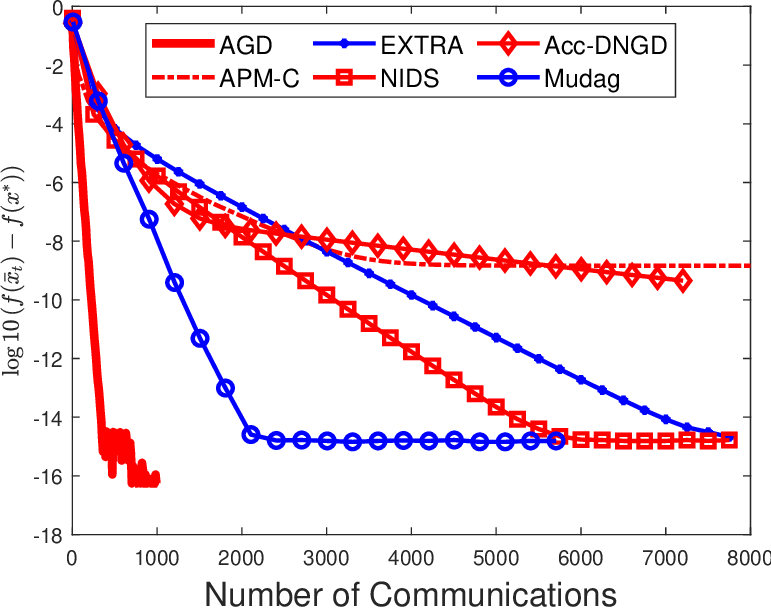} \\[0.15cm]
			\includegraphics[width=38mm,keepaspectratio]{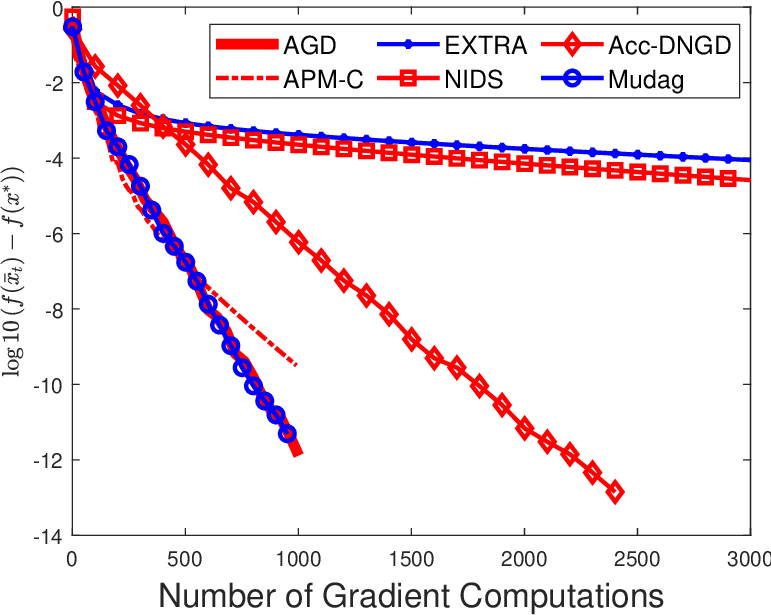}
			&\includegraphics[width=38mm,keepaspectratio]{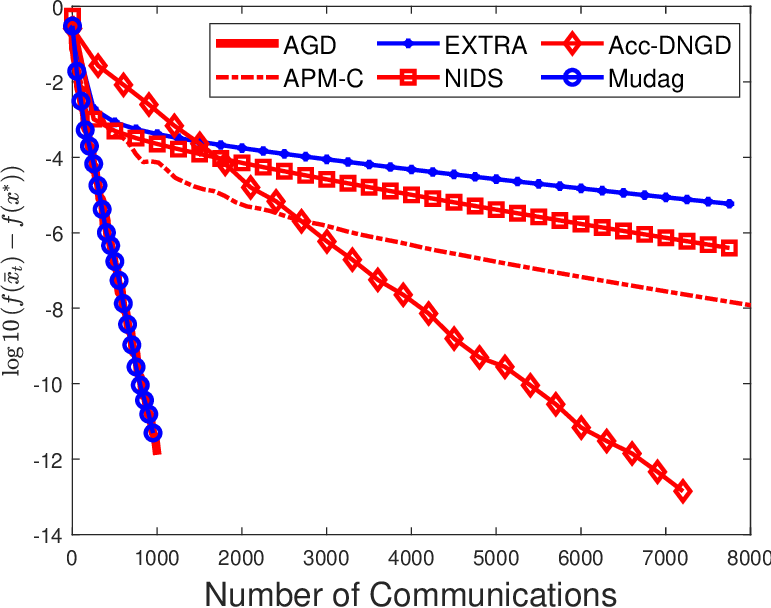}
			&\includegraphics[width=38mm,keepaspectratio]{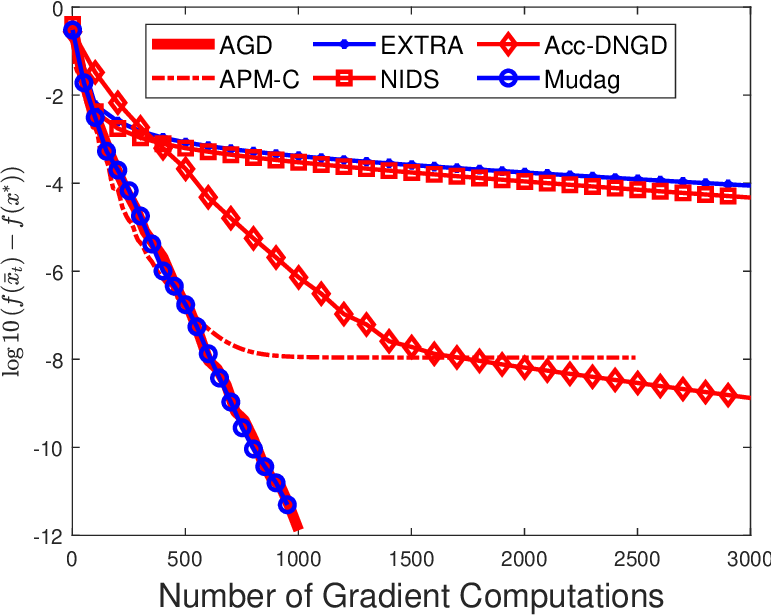}
			&\includegraphics[width=38mm,keepaspectratio]{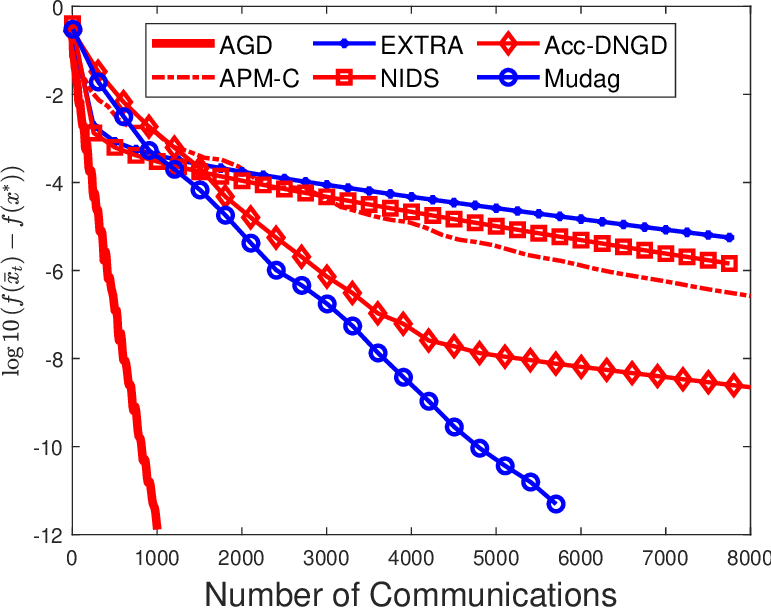} 
		\end{tabular}\vskip-0.1cm
		\caption{Comparisons with logistic regression  and random
			networks. Each $f_i(x)$ is strongly convex ($\sigma_i=0.001$
			in the top row,
			and $\sigma_i=0.0001$ in the bottom row).
			Random networks have $1-\lambda_2(W) = 0.81$ in the left two
			columns and $1-\lambda_2(W) = 0.05$ in the right two columns.}\label{fig1}
		\vskip -0.2in
	\end{figure*}
	
	We compare our algorithm (\texttt{Mudag}) to centralized accelerated gradient descent (\texttt{AGD}) in \citep{nesterov2018lectures}, \texttt{EXTRA} in \citep{shi2015extra}, \texttt{NIDS} in \citep{li2019decentralized}, \texttt{Acc-DNGD} in \citep{qu2019accelerated} and  \texttt{APM-C} in \citep{li2018sharp}. 
	In this paper, we do not compare our algorithm to the dual-based algorithms such as accelerated dual ascent algorithm \citep{uribe2018dual,seaman2017optimal} because 
	these algorithms cannot be applied to the case where some functions $f_i(x)$ are
	non-convex.
	The step sizes of all algorithms are well-tuned to achieve their best performances. 
	Furthermore, we set the momentum coefficient as ${\big(\sqrt{L} - \sqrt{\mu}\big)}/{\big(\sqrt{L} + \sqrt{\mu}\,\big)}$ for \texttt{Mudag}, \texttt{AGD} and \texttt{APM-C}. 
	We initialize $\xb_0$ at $\mathbf{0}$ for all the compared methods.

	\begin{figure*}[ht!]
		\centering
		\begin{tabular}{@{\extracolsep{0.0001em}}c@{\extracolsep{0.0001em}}c@{\extracolsep{0.0001em}}c@{\extracolsep{0.0001em}}c@{\extracolsep{0.0001em}}c}
			\includegraphics[width=37mm,keepaspectratio]{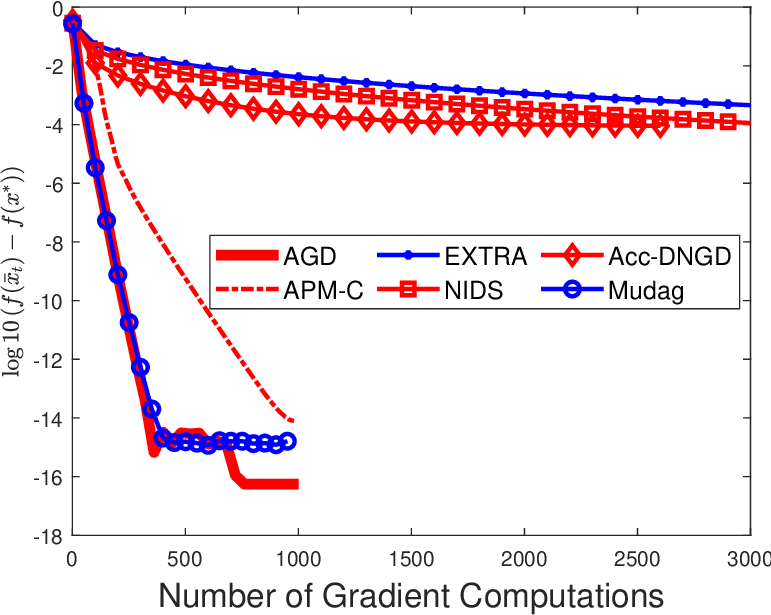}
			&\includegraphics[width=37mm,keepaspectratio]{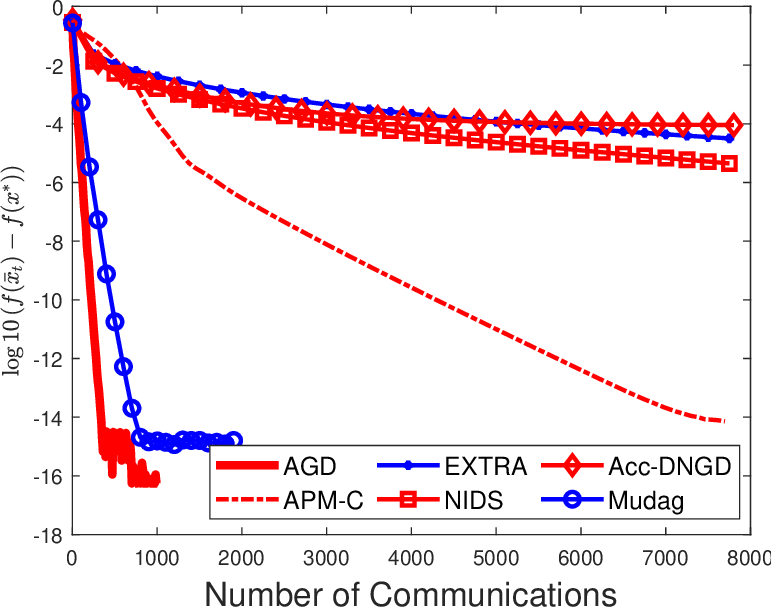}
			&\includegraphics[width=37mm,keepaspectratio]{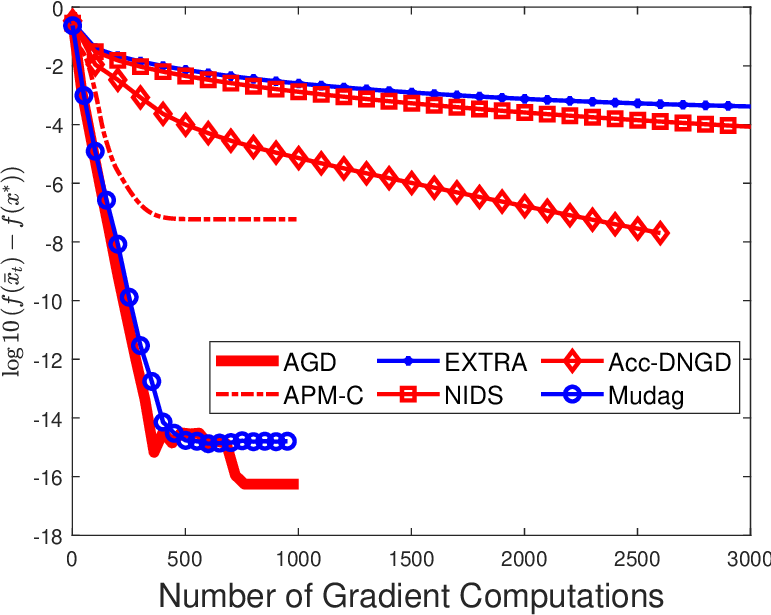}
			&\includegraphics[width=37mm,keepaspectratio]{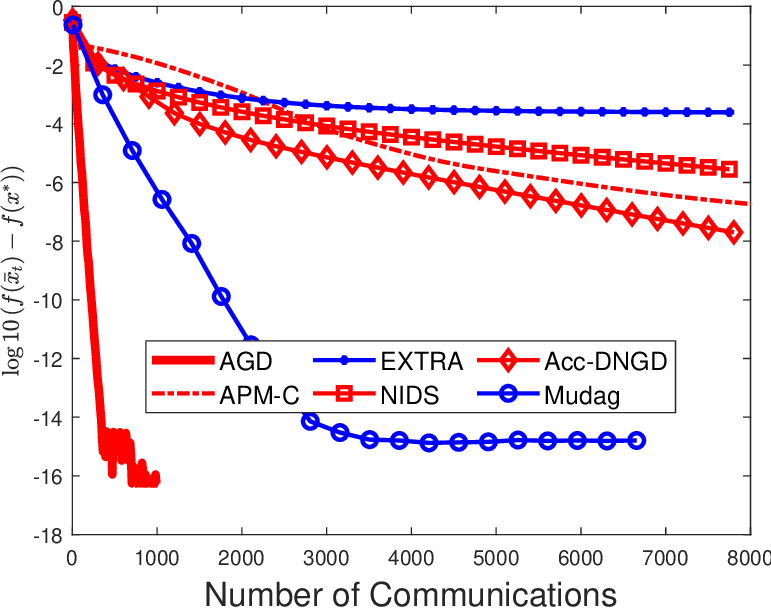} \\
			\includegraphics[width=41mm,keepaspectratio]{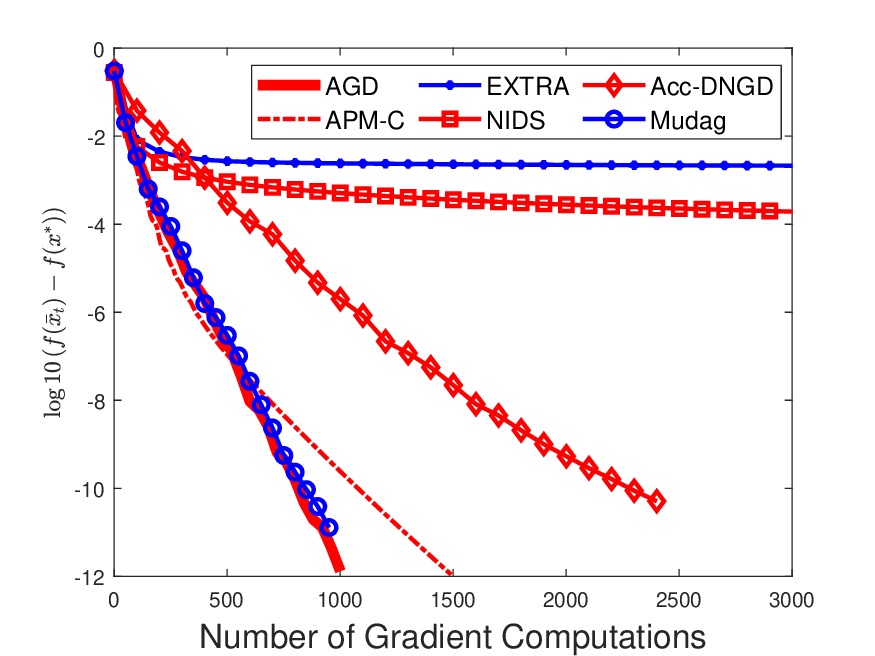}
			&\includegraphics[width=37mm,keepaspectratio]{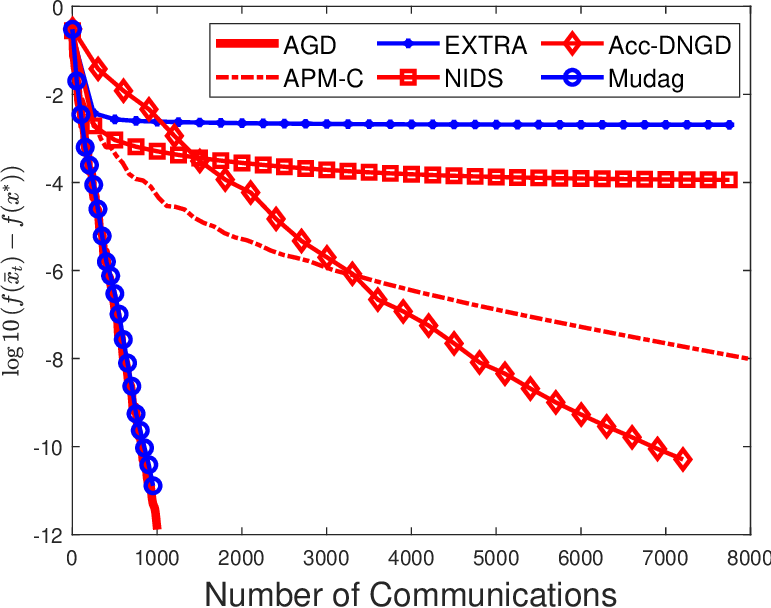}
			&\includegraphics[width=37mm,keepaspectratio]{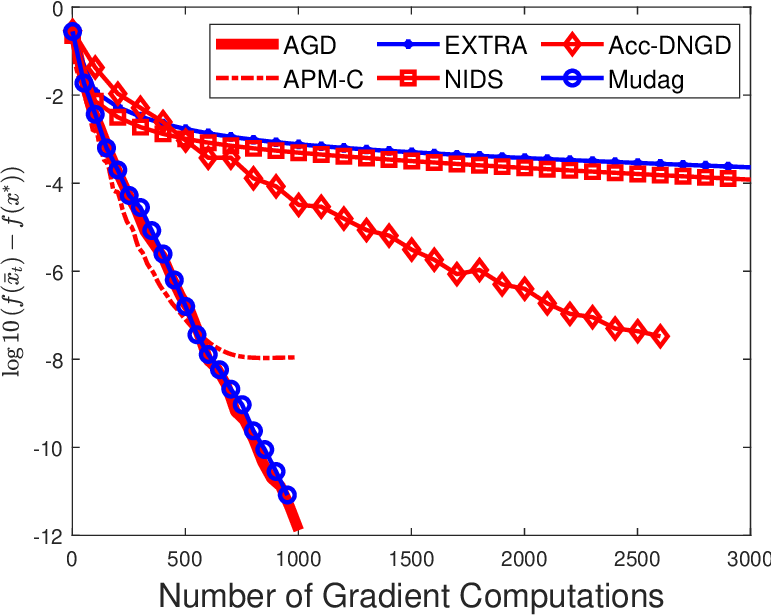}
			&\includegraphics[width=37mm,keepaspectratio]{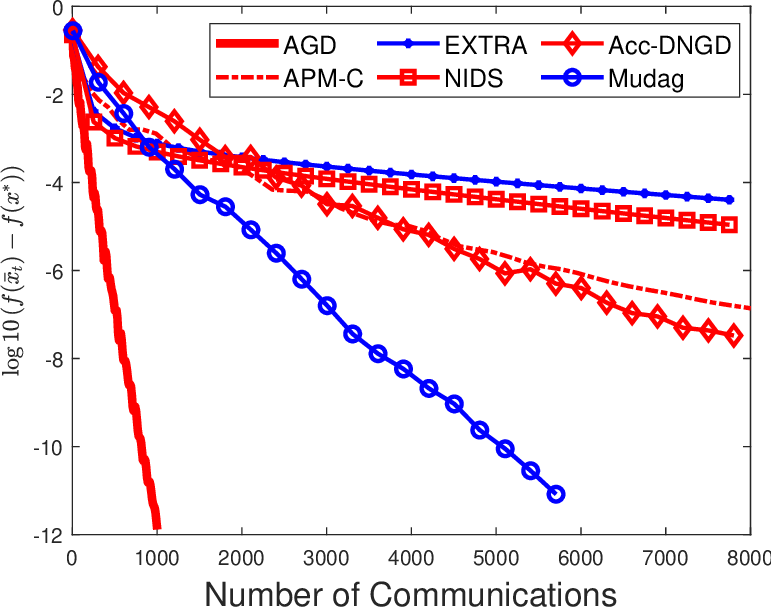} 
		\end{tabular}\vskip-0.1cm
		\caption{Comparisons with logistic regression and random
			networks. Each local objective $f_i(x)$ may be non-convex.
			In the top row,  $\sigma_i = -10^{-2}$ for agents $i=1\dots,m-1$ and
			$\sigma_i = 1$ for the agent $i=m$.
			In the bottom row,
			$\sigma_i = -10^{-1}$ for agents $i=1\dots,m-1$,
			and~$\sigma_i = 10$ for
			the agent $i=m$.
			Random networks have $1-\lambda_2(W) = 0.81$ in the left two
			columns and $1-\lambda_2(W) = 0.05$ in the right two columns.}
		\label{fig2}
		\vskip -0.2in
	\end{figure*}

	In the setting in which each $f_i(x)$ is strongly convex, we report the experimental results in Figure~\ref{fig1}. 
	Compared with \texttt{AGD}, our algorithm has almost the same computation cost, which validates our theoretical analysis.
	Assuming that \texttt{AGD}
	communicates once  per iteration,  we can also see that the
	communication cost of \texttt{Mudag} is almost the same communication cost as that of \texttt{AGD} when~$1-\lambda_2(W) = 0.81$, and six times of that
	of \texttt{AGD} when $1-\lambda_2(W) = 0.05$.
	This matches the theoretical results of communication complexity for our algorithm.
	Furthermore, our algorithm achieves both lower computation cost and lower communication cost than other decentralized algorithms on all settings.
	The advantages are more obvious for small $\sigma_i$, which also validates the comparison of the upper bounds with related works.

	\begin{figure*}[ht]
		\centering
		\begin{tabular}{@{\extracolsep{0.0001em}}c@{\extracolsep{0.0001em}}c@{\extracolsep{0.0001em}}c@{\extracolsep{0.0001em}}c@{\extracolsep{0.0001em}}c}
			\includegraphics[width=38mm,keepaspectratio]{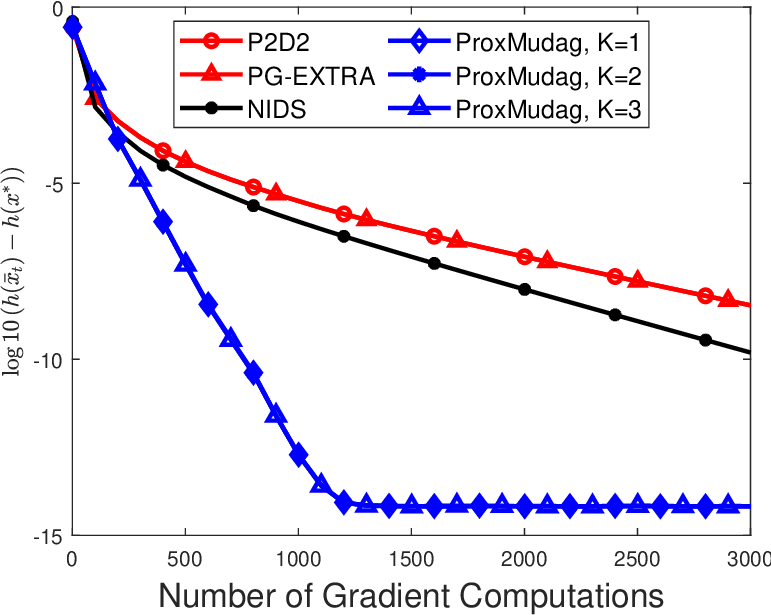}
			&\includegraphics[width=38mm,keepaspectratio]{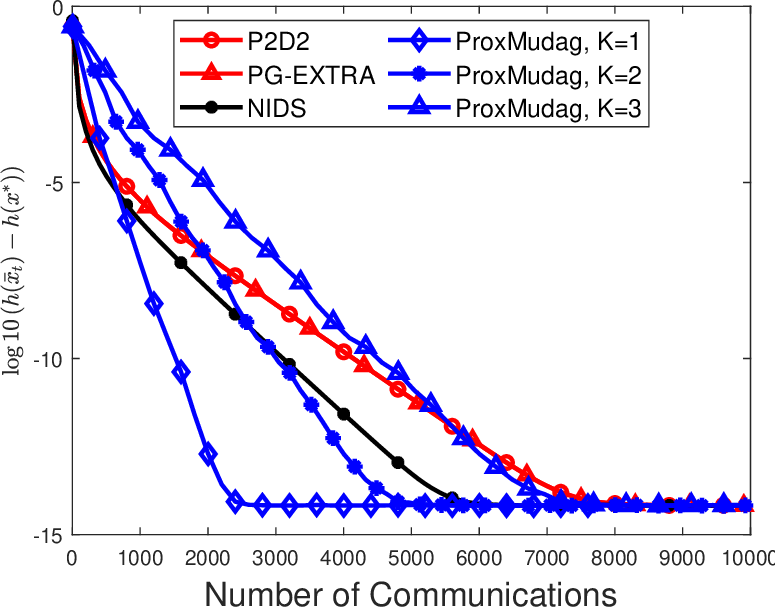}
			&\includegraphics[width=38mm,keepaspectratio]{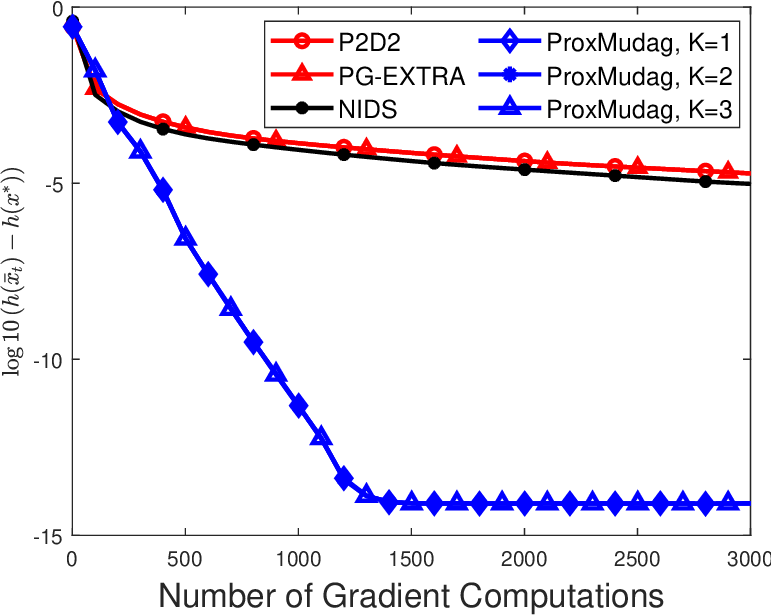}
			&\includegraphics[width=38mm,keepaspectratio]{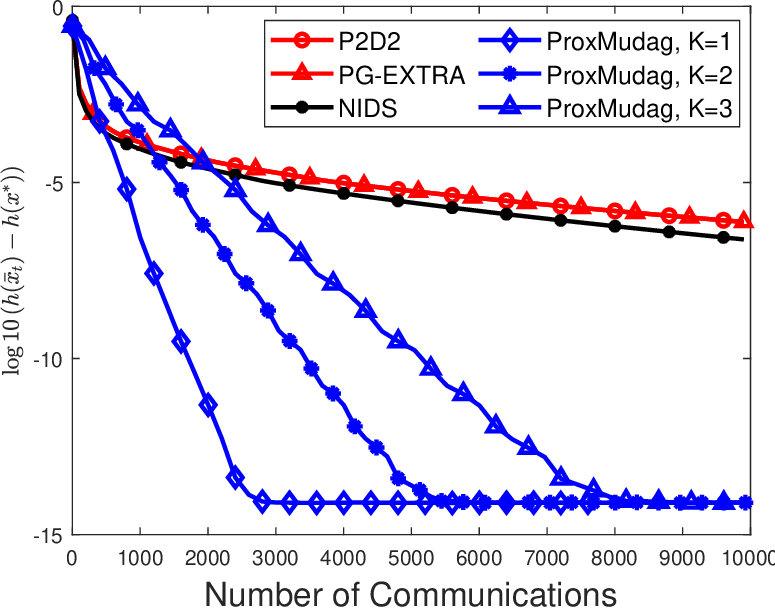}\\[0.1cm]
			\includegraphics[width=38mm,keepaspectratio]{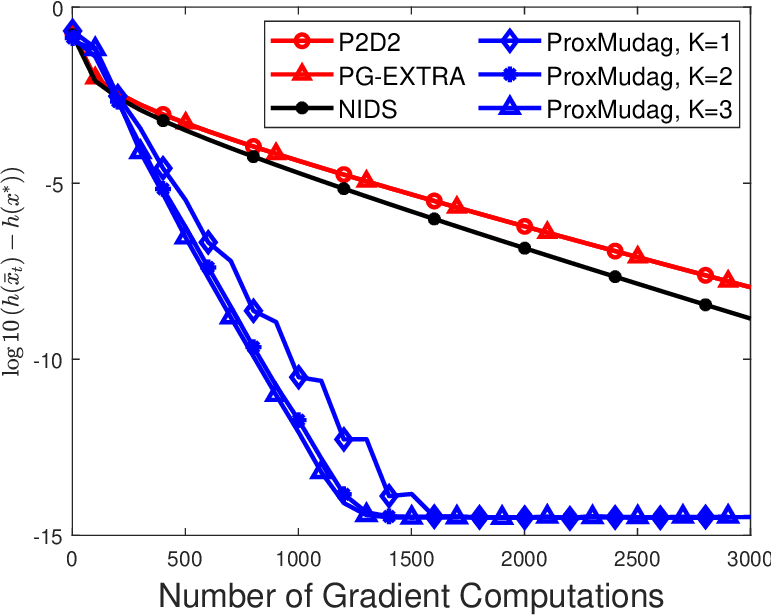}
			&\includegraphics[width=38mm,keepaspectratio]{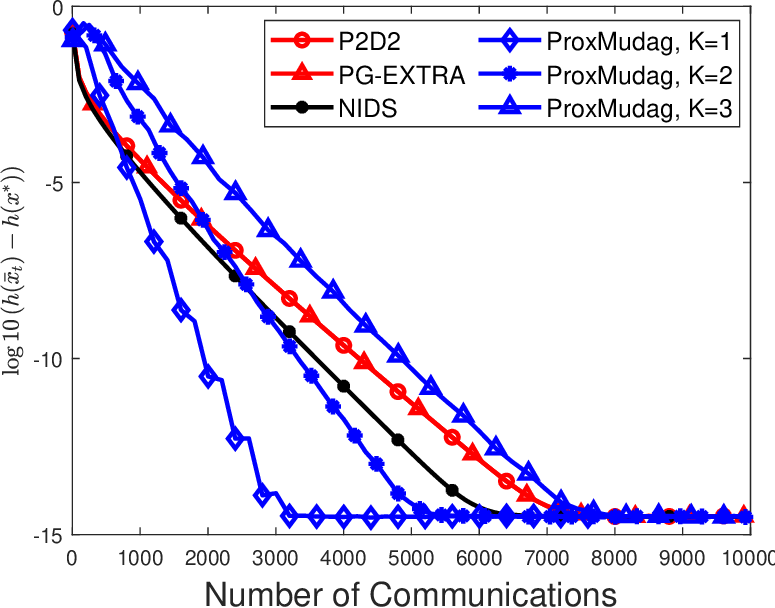}
			&\includegraphics[width=38mm,keepaspectratio]{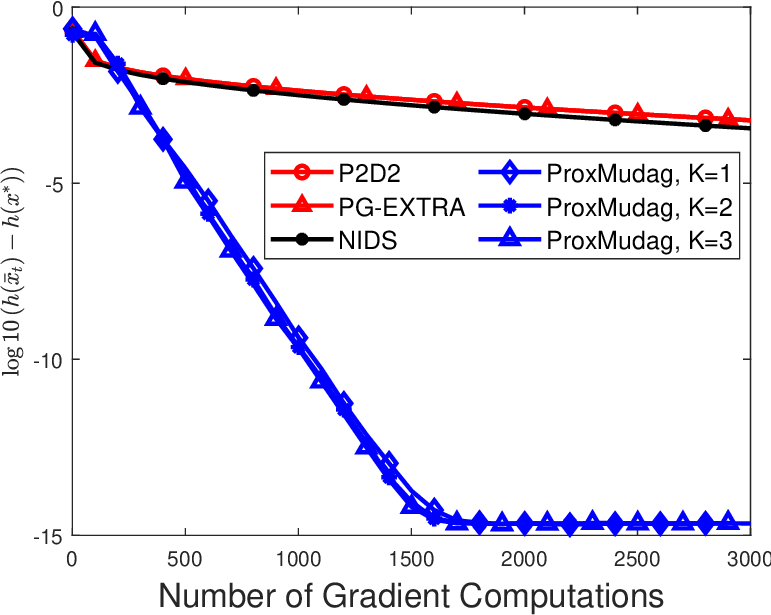}
			&\includegraphics[width=38mm,keepaspectratio]{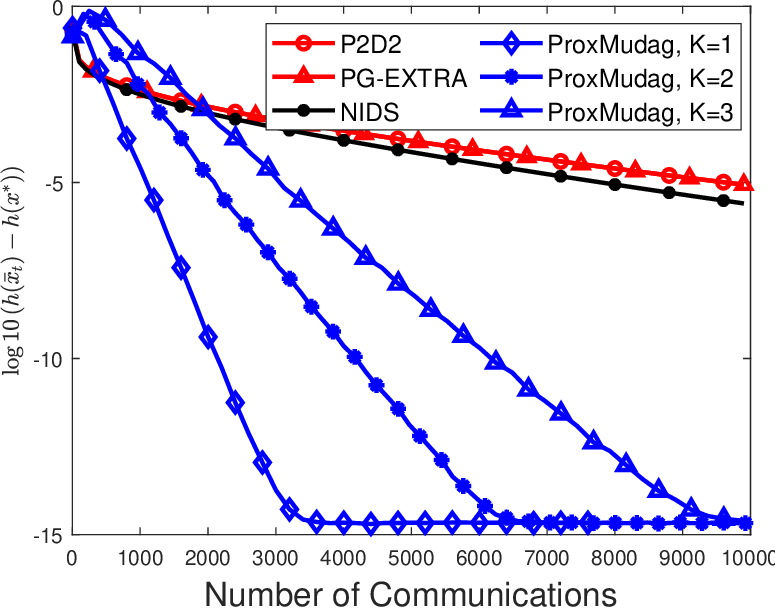}\\
		\end{tabular}
		\caption{Comparisons with sparse logistic regression and random
			networks. 
			In the top row, experiments on `a9a' with $\sigma_i = 10^{-3}$ for agents $i = 1,\dots, m$ in the left two columns and $\sigma_i = 10^{-4}$ in the right two columns.
			In the bottom row, experiments on `w8a' with~$\sigma_i = 10^{-3}$ for agents $i = 1,\dots, m$ in the left two columns and $\sigma_i = 10^{-4}$ in the right two columns.
		}
		\label{fig3}
		\vskip -0.2in
	\end{figure*}

	In the setting in which an individual function $f_i(x)$ could be non-convex, we report the experimental results in Figure~\ref{fig2}.
	Note that the global objective function of experiments reported in Figure~\ref{fig1} and Figure~\ref{fig2} are the same but the model that corresponds to Figure~\ref{fig2} contains some non-convex $f_i(x)$.
	Comparing the curves in these two figures, we can observe that the computation cost of \texttt{AGD} and our algorithm are \emph{not} affected by the non-convexity of $f_i(x)$ because their convergence rates only depend on $\sqrt{\kappa_g}$.
	On the other hand, the communication cost of our algorithm increases
	slightly compared to the setting where each $f_i(x)$ is convex.
	This is because  the ratio $M/L$  of $f_i(x)$ increases when we set $\sigma_i = -10^{-1}$ or $\sigma_i = -10^{-2}$ for agent $i=1,\dots,m-1$. 
	Our communication complexity theory shows $M/L$ will affect the communication cost by a $\log (M/L)$ factor.
	Compared with our algorithm, the performance of the other decentralized algorithms deteriorates greatly, which can be clearly observed by
	comparing the two figures in the top right corners of Figure~\ref{fig1}
	and Figure~\ref{fig2}.

	\subsection{Experiments on Sparse Logistic Regression }
	We consider the sparse logistic regression model whose objective function is defined as
	\begin{equation*}
		h(x) = \frac{1}{m} \sum_{i=1}^m f_i(x) + \gamma\norm{x}_1,
	\end{equation*}
	where $f_i(x)$ is defined in Eq.~\eqref{eq:log}.
	We conduct experiments on the graph with $1 - \lambda_2(W) = 0.05$ and $f_i(x)$ and only consider the case when each $f_i(x)$ is convex, since experiments on logistic regression have already shown the advantage of our ideas for non-convex $f_i(x)$.
	We conduct experiments on the datasets `a9a' and `w8a', which can be downloaded from Libsvm datasets. For `w8a', we set~$n = 497$ and $d = 300$. For `a9a', we set $n = 325$ and $d = 123$.
	We conduct the following two experimental settings:
	\begin{enumerate}
		\item We set $\gamma = 10^{-4}$ and $\sigma_1=\dots=\sigma_m=10^{-3}$.
		\item We set $\gamma = 10^{-4}$ and $\sigma_1=\dots=\sigma_m=10^{-4}$.
	\end{enumerate}
	
	We compare our algorithm (\dapg) with the state-of-the-art algorithms \texttt{PG-EXTRA} \citep{ShiLWY15}, \texttt{NIDS} \citep{li2019decentralized} and decentralized proximal algorithm (\texttt{D2P2}) \citep{alghunaim2019linearly}.
	In the experiments, we set $K=1$, $K = 2$ and $K = 3$ in \dapg~to evaluate how $K$ affects the convergence behavior.
	The parameters of all algorithms are well-tuned. 
	We report the experimental results in Figure~\ref{fig3}.
	We can observe that \dapg~ outperforms other algorithms in all cases.
	First, \dapg~ takes much less computation cost than other algorithms since \dapg~ uses Nesterov's acceleration to achieve a faster convergence rate. 
	This matches our theoretical analysis of the computation complexity.
	We can further observe that the advantage of \dapg~is more clear when $\sigma_2$ is small.
	This is because the small $\sigma_i$'s commonly lead to large condition numbers and the computation complexity of \dapg~depends on $\sqrt{\kappa_g}$ instead of $\kappa_\ell$ or $\sqrt{\kappa_\ell}$.
	The results also show \dapg~ has great advantages over other state-of-the-art decentralized proximal algorithms on the communication cost.

	\section{Conclusion}
	
	In this paper, we proposed two novel decentralized algorithms, which achieve the optimal computation complexity and the near optimal communication complexity.
	To the best of our knowledge, this is the best communication complexity that primal-based decentralized algorithms can achieve especially for the decentralized composite optimization problems.
	
	Our results provide an affirmative answer to the open problem whether there is a decentralized algorithm that can achieve the communication complexity  $\cO\big(\sqrt{{\kappa_{g}}/{(1-\lambda_2(W))}}\log({1}/{\epsilon})\big)$ or even close to this lower bound for a strongly convex objective function.
	Furthermore, our algorithm does not require each individual functions $f_i(x)$ to be convex.
	Our experiments showed that the non-convexity of individual function $f_i(x)$ rarely degrades the performance of our algorithm. 
	Our analysis also implies that integrating multi-consensus and gradient tracking can well approximate the decentralized optimization algorithm to the corresponding centralized counterpart. 
	The implementation of the resulting algorithms are simple, effective and with (near) optimal complexities. 
	This novel perspective may also provide useful insights for developing new decentralized optimization algorithms in other settings.

\section*{Acknowledgments}
The authors would like to thank Lesi Chen and Yuxing Liu's helpful discussion.
Haishan Ye is supported by National Natural Science
Foundation of China under Grant No. 12101491.
Luo Luo is supported by National Natural Science Foundation of China (No. 62206058) and Shanghai Sailing Program (22YF1402900). 

\newpage    
\appendix
	
	\section{Useful Lemmas}
	
	\begin{lemma}
		For any matrix $\xb\in\RR^{m\times d}$ and $\bbx = \frac{1}{m} \mathbf{1}^\top \xb$, it holds that
		\begin{equation} \label{eq:xx_m}
			\norm{\xb - \mathbf{1}\bbx}^2 
			\le
			\norm{\xb}^2.
		\end{equation}
	\end{lemma}
	\begin{proof}
		It holds that 
		\begin{align*}
			& \norm{\xb - \mathbf{1}\bbx}^2 \\
			=&
			\sum_{j=1}^m\norm{\xb(j,:) - \frac{1}{m}\sum_{i=1}^m \xb(i,:)}^2
			\\
			=&
			\sum_{j=1}^m\left(\norm{\xb(j,:)}^2 - \frac{2}{m}\sum_{i=1}^m\dotprod{\xb(j,:),\xb(i,:)} + \frac{1}{m^2}\norm{\sum_{i=1}^m \xb(i,:)}^2\right)
			\\
			=&\norm{\xb}^2 - \frac{2}{m}\sum_{i=1}^m\sum_{j=1}^m\dotprod{\xb(j,:),\xb(i,:)} + \frac{1}{m}\sum_{i=1}^m\sum_{j=1}^m\dotprod{\xb(j,:),\xb(i,:)} \\
			\le & 
			\norm{\xb}^2.
		\end{align*}
	\end{proof}
	
	\begin{lemma}
		We have
		\begin{align}
			\norm{\nabla F(\yb) - \nabla F(\xb)} \le M \norm{\yb - \xb} \label{eq:M_sm}
        \end{align}
        and
        \begin{align}
			\norm{\bbg_t - \nabla f(\bby_t)} \leq \frac{M}{\sqrt{m}}\norm{\yb_t - \mathbf{1}\bby_t}. \label{eq:g_y}
		\end{align}
		Furthermore, we have the $(L + 2/\eta)$-smooth property for the generalized gradient $\geng h(\cdot)$ (defined in Eq.~\eqref{eq:geng}), i.e.,
		\begin{align}\label{eq:62}
			\norm{\geng h(x)-\geng h(y)}\leq \left( \frac{2}{\eta}+L\right) \norm{x-y}.
		\end{align}
	\end{lemma}
	\begin{proof}
		The first inequality is because each $f_i(x)$ is $M$-smooth and
		\begin{align*}
			\norm{\nabla F(\yb) - \nabla F(\xb)} 
			= &	\sqrt{\sum_i^m \norm{\nabla f_i(\yb(i,:))- \nabla f_i(\xb(i,:))}^2 } \\
			\le & \sqrt{M^2\sum_i^m \norm{\yb(i,:)- \xb(i,:)}^2 }
			= M\norm{\yb - \xb}.
		\end{align*}
		The second inequality follows from
		\begin{align*}
			\norm{\bbg_t - \nabla f(\bby_t)} 
			=&
			\norm{\sum_{i=1}^m \frac{\nabla f_i(\yb_t(i,:)) - \nabla f_{i}(\bby_t)}{m}}
			\le
			M \sum_{i=0}^m \frac{\norm{\yb_t(i,:) - \bby_t}}{m}
			\\
			\le&
			M \sqrt{\sum_{i=1}^m \frac{\norm{\yb_t(i,:) - \bby_t}^2}{m}}
			=
			\frac{M}{\sqrt{m}} \norm{\yb_t - \mathbf{1}\bby_t}.
		\end{align*}
		Then we can prove Eq.~\eqref{eq:62} using $L$-smoothness of $f(x)$ and the non-expansiveness of the proximal operator
		\begin{align*}
			\norm{\geng h(x)-\geng h(y)}=&
			\norm{\frac{x-\proximal_{\eta,r}(x-\eta\nabla f(x))}{\eta}-\frac{y-\proximal_{\eta,r}(y-\eta\nabla f(y))}{\eta}}
			\\
			\le&
			\frac{1}{\eta}\norm{x-y}+\frac{1}{\eta}\norm{\proximal_{\eta,r}(x-\eta\nabla f(x))-\proximal_{\eta,r}(y-\eta\nabla f(y))}
			\\
			\le&
			\frac{1}{\eta}\norm{x-y}+\frac{1}{\eta}\norm{(x-\eta\nabla f(x))-(y-\eta\nabla f(y))}
			\\
			\le&\left( \frac{2}{\eta}+L\right) \norm{x-y},
		\end{align*}
		where the last inequality is due to the $L$-smoothness of $f(x)$. 
	\end{proof}
	
	\begin{lemma}\label{lem:v_t transform}
		For $\bbx_t$, $\bby_t$ and $\bbv_t$ defined in Eqs.~\eqref{eq:xyg} and~\eqref{eq:v_t}, then we can obtain that
		\begin{align}
			\bby_t-\bbx_t&=\alpha(\bbv_t-\bby_t), \label{eq:yxv} \\
			\bby_{t+1}&=\frac{\bbx_{t+1}+\alpha \bbv_{t+1}}{1+\alpha} \label{eq:y_v}, 
        \end{align}
        and
        \begin{align}
			\bbv_{t+1} &= 
			\begin{cases}
				(1 - \alpha)\bbv_t + \alpha \bby_t -\dfrac{\eta}{\alpha} \bbG_t  \quad & r(x), \mbox{ is convex,} \\[0.3cm] 
				(1 - \alpha)\bbv_t + \alpha \bby_t -\dfrac{\eta}{\alpha} \bbg_t,  \quad & r(x)=0.
			\end{cases}
			\label{eq:vv_p}
		\end{align}
	\end{lemma}
	\begin{proof}
		First using the definition of $\bbv_t$, we have
		\begin{equation*}
			\begin{aligned}
				\frac{\bbx_{t+1}+\alpha \bbv_{t+1}}{1+\alpha}\stackrel{\eqref{eq:v_t}}{=} & \frac{\bbx_{t+1}+\alpha[ \bbx_{t}+\frac{1}{\alpha}(\bbx_{t+1}-\bbx_{t})    ] }{1+\alpha} \\
				= & \bbx_{t+1}+\frac{1-\alpha}{1+\alpha}(\bbx_{t+1}-\bbx_{t}) \\
				= & \bby_{t+1}.
			\end{aligned}
		\end{equation*}
		Then we can have
		\begin{align*}
			\bby_t-\bbx_t=\alpha(\bbv_t-\bby_t).
		\end{align*}
		Now, we are going to prove Eq.~\eqref{eq:vv_p} with the case $r(x)$ is convex since $r(x) = 0$ is a special case of~$r(x)$ being convex. 
		Then we have
		\begin{align*}
			(1-\alpha)\bbv_t + \alpha \bby_t - \frac{\eta}{\alpha}\bbG_t 
			=&\bbv_t -\alpha(\bbv_t - \bby_t) - \frac{\eta}{\alpha}\bbG_t 
			=\bbx_t + \bbv_t - \bby_t- \frac{\eta}{\alpha}\bbG_t 
			\\
			=&\bbx_t + \frac{1}{\alpha}(\bby_t - \bbx_t - \eta\bbG_t) 
			=\bbx_t + \frac{1}{\alpha}(\bbx_{t+1} -\bbx_t)
			= \bbv_{t+1}.
		\end{align*}
	\end{proof}

	\begin{lemma}\label{lem:y_t lyapnov}
		Let $f(x)$ be $\mu$-strongly convex. For $\bby_t$, and $V_t$ defined in Eqs.~\eqref{eq:xyg} and~\eqref{eq:V_t} and $x^*$ being the optimum, we have the following inequality,
		\begin{equation}\label{eq:y_x_V}
			\norm{\bby_t -x^*}\le \sqrt{\frac{2V_t}{\mu}}.
		\end{equation}
	\end{lemma}
	\begin{proof}
		Since $f(x)$ is $\mu$-strongly convex, $h(x)$ in Eq.~\eqref{eq:prob} is also $\mu$-strongly convex. 
		Thus, we obtain
		\begin{equation*}
			\begin{aligned}
				\norm{\bby_t-x^*}\stackrel{\eqref{eq:y_v}}{=}&  \norm{\frac{\bbx_{t}+\alpha \bbv_{t}}{1+\alpha}-x^*}
				\le \frac{1}{1+\alpha}\norm{\bbx_t-x^*}+\frac{\alpha}{1+\alpha}\norm{\bbv_t-x^*}
				\\
				\le& \frac{1}{1 + \alpha} \sqrt{\frac{2 (h(\bbx) - h(x^*))}{\mu}} 
				+
				\frac{\alpha}{1+\alpha} \sqrt{\frac{2}{\mu}\cdot \frac{\mu}{2} \norm{\bbv_t - x^*}^2}
				\le
				\sqrt{\frac{2V_t}{\mu}}.
			\end{aligned}
		\end{equation*}
	\end{proof}

At the end of this section, we provide the proof of Proposition \ref{lem:mix_eq}.
 
\begin{proof}    
We let 
\begin{align*}
\Pi = I - \frac{1}{n}\mathbf{1}\mathbf{1}^\top,
\qquad
\tilde\Pi = \begin{bmatrix}
    \Pi & 0 \\
    0 & \Pi
\end{bmatrix}
\qquad\text{and}\qquad
\tilde W = \begin{bmatrix}
    (1+\eta_w)W & -\eta_w W \\
    I & 0
\end{bmatrix},
\end{align*}
then the iteration of Algorithm \ref{alg:mix} can be written as
\begin{align*}
\begin{bmatrix}
    \xb_{k+1} \\ \xb_k
\end{bmatrix}    
= \tilde W \begin{bmatrix}
    \xb_k \\ \xb_{k-1}
\end{bmatrix}.    
\end{align*}
The property $W\mathbf{1}=\mathbf{1}$ directly leads to $\bar x=\frac{1}{m}\mathbf{1}^\top\xb^K$. It also indicates
\begin{align*}
W\Pi = W \left(I-\frac{1}{n}\mathbf{1}\mathbf{1}^\top\right)
= W - \frac{1}{n}W\mathbf{1}\mathbf{1}^\top
= W - \frac{1}{n}\mathbf{1}\mathbf{1}^\top
\end{align*}
and
\begin{align*}
\Pi W = \left(I-\frac{1} {n}\mathbf{1}\mathbf{1}^\top\right)W
= W - \frac{1}{n}\mathbf{1}\mathbf{1}^\top W
= W - \frac{1}{n}\mathbf{1}\mathbf{1}^\top,
\end{align*}
which means 
\begin{align}\label{eq:PiW}
\Pi W = W\Pi.    
\end{align}
Consequently, we achieve
\begin{align*}
    \tilde W\tilde\Pi
= & \begin{bmatrix}
    (1+\eta_w)W & -\eta_w W \\
    I & 0
\end{bmatrix}
\begin{bmatrix}
    \Pi & 0 \\
    0 & \Pi
\end{bmatrix} \\
= & \begin{bmatrix}
    (1+\eta_w)W\Pi & -\eta_w W\Pi \\
    \Pi & 0
\end{bmatrix}, \\
   \tilde\Pi \tilde W
= & \begin{bmatrix}
    \Pi & 0 \\
    0 & \Pi
\end{bmatrix}
\begin{bmatrix}
    (1+\eta_w)W & -\eta_w W \\
    I & 0
\end{bmatrix} \\
= & \begin{bmatrix}
    (1+\eta_w)W\Pi & - \Pi\eta_w W \\
    \Pi & 0
\end{bmatrix},
\end{align*}
and
\begin{align*}
    \tilde\Pi\tilde W\tilde\Pi
= & \begin{bmatrix}
    \Pi & 0 \\
    0 & \Pi
\end{bmatrix}
\begin{bmatrix}
    (1+\eta_w)W\Pi & -\eta_w W\Pi \\
    \Pi & 0
\end{bmatrix} \\
= & \begin{bmatrix}
    (1+\eta_w)\Pi W\Pi & -\eta_w \Pi W\Pi \\
    \Pi^2 & 0
\end{bmatrix} \\
= & \begin{bmatrix}
    (1+\eta_w)\Pi W & -\eta_w \Pi W \\
    \Pi & 0
\end{bmatrix} \\
= & \tilde\Pi\tilde W = \tilde W\tilde\Pi,
\end{align*}
where we use the equality (\ref{eq:PiW}). 
This implies for any $K\geq 2$, we have
\begin{align*}
  & \tilde\Pi \tilde W^K 
= (\tilde\Pi \tilde W) \tilde W^{K-1}    
= (\tilde\Pi \tilde W \tilde\Pi) \tilde W^{K-1}    
= (\tilde\Pi \tilde W) \tilde\Pi \tilde W^{K-1} \\  
= & (\tilde\Pi \tilde W)^2 \tilde\Pi \tilde W^{K-2} = \dots = (\tilde\Pi \tilde W)^K \tilde\Pi \\
= & \tilde\Pi (\tilde W\tilde\Pi)^K 
= \tilde\Pi\tilde W (\tilde\Pi\tilde W)^{K-1} \tilde\Pi \\
= & \tilde\Pi\tilde W (\tilde\Pi\tilde W)^{K-2} \tilde\Pi\tilde W\tilde\Pi \\
=& \tilde\Pi\tilde W (\tilde\Pi\tilde W)^{K-2} \tilde W\tilde\Pi = \dots 
= \tilde\Pi\tilde W \tilde W^{K-1}\tilde\Pi \\
= & \tilde\Pi\tilde W^{K}\tilde\Pi.
\end{align*}
Combining above result with Lemma 9 of \citet{song2021optimal}, we have
\begin{align*}
& \norm{\xb_K - \mathbf{1}\bar x} 
= \norm{\Pi\xb_K} \\
\leq & \norm{\tilde\Pi \begin{bmatrix}
    \xb_{K} \\ \xb_{K-1}
\end{bmatrix}}
= \norm{\tilde\Pi\tilde W^K \begin{bmatrix}
    \xb_0 \\ \xb_{-1}
\end{bmatrix}}
=  \norm{\tilde\Pi\tilde W^K \tilde\Pi \begin{bmatrix}
    \xb_0 \\ \xb_{-1}
\end{bmatrix}} \\
\leq & \sqrt{14} \tilde\rho_w^K \norm{\Pi \xb_0} 
= \sqrt{14} \tilde\rho_w^K \norm{\xb_0 - \mathbf{1}\bar x},
\end{align*}
where 
\begin{align*}
    \Tilde{\rho}_w = \frac{1}{\sqrt{1+\sqrt{1-\lambda_2^2(W)}}}.
\end{align*}
Since we have 
\begin{align}
    \frac{1}{\sqrt{1+x}} \le 1 - \left(1-\frac{1}{\sqrt{2}}\right)x,
\end{align}
for any $x\in [0,1]$, it holds that
\begin{align*}
    \Tilde{\rho}_w =& \frac{1}{\sqrt{1+\sqrt{1-\lambda_2^2(W)}}} \\
    \le &
    1 - \left(1-\frac{1}{\sqrt{2}}\right) \sqrt{1-\lambda_2^2(W)} \\
    \le &
    1 - \left(1-\frac{1}{\sqrt{2}}\right) \sqrt{1-\lambda_2(W)}.
\end{align*}
This implies
\begin{align*}
 \norm{\xb_K - \mathbf{1}\bar x} 
\leq \sqrt{14} \left(1 - \left(1-\frac{1}{\sqrt{2}}\,\right) \sqrt{1-\lambda_2(W)}\right)^K \norm{\xb_0 - \mathbf{1}\bar x}.
\end{align*}
\end{proof}

\section{Proof of Lemmas in Section~\ref{sec:converg}}

	\subsection{Collection of Lemmas}
	
	We list several important lemmas that will be used in our proofs.
	
	\begin{lemma}[\cite{nesterov2018lectures}]
		Letting $\geng h(x)$ the generalized gradient of $h(x)$ (refer to Eq.~\eqref{eq:prob}) be defined as 
		\begin{equation}
			\label{eq:geng}
			\geng h(x)\triangleq\frac{x-\proximal_{\eta,r}(x-\eta\nabla f(x))}{\eta}, \;\mbox{with}\; \eta\;\mbox{being the step size},
		\end{equation}
		then it holds that $\geng h(x^*) = 0$ if $x^*$ minimizes $h(x)$.
	\end{lemma}

	\begin{lemma}\label{lem:proximal global to local}
		Letting $\proximal_{\eta m,R}^{(i)}(\xb)$ denote the $i$-th row of the matrix $\proximal_{\eta m,R}(\xb)$ (defined in Eqn.~\eqref{eq:proxes}), we have the following equation
		\begin{equation*}\label{eq:local proximal equality}
			\proximal_{\eta m,R}^{(i)}(\xb)=\proximal_{\eta,r}(\xb^{(i)}),
		\end{equation*} 
		which implies that $G_t^{(i)}$ equals to the $i$-th row of $G_t$ defined in Eq.~\eqref{eq:g_grad}.
	\end{lemma}
	\begin{proof}
		By the definition of the proximal operators, we have
		\begin{align*}
			\proximal_{\eta m,R}(\xb)=&\argmin_{\zb} \left( R(\zb)+ \frac{1}{2\eta m}\norm{\zb-\xb}_F^2\right) \\
			=&\argmin_{\zb} \left( \frac{1}{m}\sum_{i=1}^{m}r(\zb^{(i)})+ \sum_{i=1}^{m}\frac{1}{2\eta m}\norm{\zb^{(i)}-\xb^{(i)}}^2\right) \\
			=&\argmin_{\zb} \left( \sum_{i=1}^{m}r(\zb^{(i)})+ \sum_{i=1}^{m}\frac{1}{2\eta}\norm{\zb^{(i)}-\xb^{(i)}}^2\right) \\
			=&  \begin{pmatrix}
				\argmin_{z}\left( r(z)+\frac{1}{2\eta}\norm{z-\xb^{(1)}}  \right)^\top \\ 
				\vdots \\ 
				\argmin_{z}\left( r(z)+\frac{1}{2\eta}\norm{z-\xb^{(m)}}  \right)^\top
			\end{pmatrix}
			.
		\end{align*}
		Therefore, we have the following equation
		\begin{equation*}
			\proximal_{\eta m,R}^{(i)}(\xb)=\proximal_{\eta,r}(\xb^{(i)}).
		\end{equation*} 
	\end{proof}
	
	\begin{lemma}\label{lem:prox average 1}
		For any $\xb\in\RR^{m\times d}$,  $\proximal_{m \eta,R}(\cdot)$ (defined in Eq.~\eqref{eq:proxes}) has the following property
		\begin{equation}
			\norm{\proximal_{\eta  m,R}\left(\frac{1}{m}\mathbf{1}\mathbf{1}^\top \xb\right)-\frac{1}{m}\mathbf{1}\mathbf{1}^\top\proximal_{\eta  m,R}( \xb)}
			\le \norm{  \xb-\mathbf{1} \bbx_t }. \label{eq:prox_non_exp}
		\end{equation}
	\end{lemma}


	\begin{proof}
		Using Lemma~\ref{lem:proximal global to local} and non-expansiveness of the proximal mapping, we have
		\begin{equation*}
			\begin{aligned}
				&\norm{\proximal_{\eta  m,R}(\frac{1}{m}\mathbf{1}\mathbf{1}^\top \xb)-\frac{1}{m}\mathbf{1}\mathbf{1}^\top\proximal_{\eta  m,R}( \xb)}^2 \\
				= & {m\norm{\proximal_{\eta  ,r}(\frac{1}{m}\mathbf{1}^\top \xb)-\frac{1}{m}\sum_{i=1}^m\proximal_{\eta ,r}( \xb^{(i)})}^2}
				\\
				=& m\norm{\frac{1}{m}\sum_{i=1}^m\left(\proximal_{\eta  ,r}( \frac{1}{m}\mathbf{1}^\top \xb)-\proximal_{\eta ,r}( \xb^{(i)})\right) }^2  \\
				\le & \sum_{i=1}^m\norm{\left(\proximal_{\eta  ,r}( \frac{1}{m}\mathbf{1}^\top \xb)-\proximal_{\eta ,r}( \xb^{(i)})\right) }^2
				\\
				\le&{\sum_{i=1}^m\norm{\left( \frac{1}{m}\mathbf{1}^\top \xb- \xb^{(i)}\right) }^2}
				=\norm{  \xb-\mathbf{1} \bbx_t }^2.
			\end{aligned}
		\end{equation*}
	\end{proof}

	\begin{lemma}\label{lem:sum of s_t gap}
		Letting $\bs_t^{(i)}$ be the $i$-th row of $\bs_t$ and $G_t^{(i)},\; \bbG_t$ (defined in Eqs.~\eqref{eq:g_grad},~\eqref{eq:bbG}) generated by Algorithm~\ref{alg:DAGD_p},  we have
		\begin{align}
			\sum_{i=1}^m\norm{\bs_t^{(i)}-\nabla f(\yb_t^{(i)})}^2\le&2\norm{\bs_t-\mathbf{1}\bbs_t}^2+8M^2\norm{\yb_t - \mathbf{1}\bby_t}^2\label{eq:s_gd}
        \end{align}
        and
        \begin{align}
			\eta^2 \sum_{i=1}^m \norm{  G_t^{(i)} -  \bbG_t}^2 \le& 18 \norm{\yb_t - \mathbf{1}\bby_t}^2 
			+ 16\eta^2\norm{\bs_t - \mathbf{1}\bbs_t}^2.  \label{eq:GG}
		\end{align}
	\end{lemma}
	\begin{proof}
		Using  the inequality that $(a+b)^2\le 2a^2+2b^2$, we have
		\begin{equation*}
			\begin{aligned}
				\sum_{i=1}^m\norm{\bs_t^{(i)}-\nabla f(\yb_t^{(i)})}^2
				\le&2\sum_{i=1}^m\norm{\bs_t^{(i)}-\bbs_t}^2+2\sum_{i=1}^m\norm{\bbs_t-\nabla f(\yb_t^{(i)})}^2\\
				\le&2\sum_{i=1}^m\norm{\bs_t^{(i)}-\bbs_t}^2+4\sum_{i=1}^m\norm{\bbs_t-\nabla f(\bby_t)}^2+4\sum_{i=1}^m\norm{\nabla f(\bby_t)-\nabla f(\yb_t^{(i)})}^2\\
				\le&2\norm{\bs_t-\mathbf{1}\bbs_t}^2+4M^2\norm{\mathbf{1}\bby_t-\yb_t}^2+4L^2\norm{\mathbf{1}\bby_t-\yb_t}^2\\
				\le&2\norm{\bs_t-\mathbf{1}\bbs_t}^2+8M^2\norm{\mathbf{1}\bby_t-\yb_t}^2,
			\end{aligned}
		\end{equation*}
		where the third inequality is from  Eq.~\eqref{eq:g_y} and the $L$-smoothness of $f(x)$, the last inequality is due to $L\le M$.
		
		Furthermore, it holds that
		\begin{align*}\small
			\begin{split}
				&\eta^2 \sum_{i=1}^m \norm{  G_t^{(i)} -  \bbG_t}^2
				=\norm{\eta G_t-\frac{\eta }{m}\mathbf{1}\cdot \mathbf{1}^\top G_t}^2
				\\
				=&
				\sum_{i=1}^m \norm{ \yb_t^{(i)} - \proximal_{\eta  ,r}(\yb_t^{(i)} - \eta \bs_t^{(i)} )  - \frac{1}{m}\sum_{j=1}^m  \left( \yb_t^{(j)} - \proximal_{\eta  ,r}(\yb_t^{(j)} - \eta \bs_t^{(j)} )\right)}^2
				\\
				\le& 2 \sum_{i=1}^m \norm{\yb_t^{(i)} - \bby_t}^2
				+ 2 \sum_{i=1}^m \norm{ \proximal_{\eta  ,r}(\yb_t^{(i)} - \eta \bs_t^{(i)} ) - \frac{1}{m}\sum_{j=1}^m   \proximal_{\eta  ,r}(\yb_t^{(j)} - \eta \bs_t^{(j)} )}^2
				\\
				\le&
				2 \sum_{i=1}^m \norm{\yb_t^{(i)} - \bby_t}^2
				+ 4 \sum_{i=1}^m \norm{ \proximal_{\eta  ,r}(\yb_t^{(i)} - \eta \bs_t^{(i)} ) - \proximal_{\eta  ,r}(\bby_t - \eta \bbs_t) }^2
				\\&
				+4\sum_{i=1}^m \norm{\proximal_{\eta  ,r}(\bby_t - \eta \bbs_t)-  \frac{1}{m}\sum_{j=1}^m   \proximal_{\eta  ,r}(\yb_t^{(j)} - \eta \bs_t^{(j)} )}^2
				\\
				\le&
				2  \norm{\yb_t - \mathbf{1} \bby_t}^2
				+ 4\sum_{i=1}^m\norm{ \yb_t^{(i)} - \bby_t - \eta (\bs_t^{(i)} - \bbs_t) }^2
				+ 4\sum_{i=1}^m\norm{\proximal_{\eta  ,r}(\bby_t - \eta \bbs_t)-     \proximal_{\eta  ,r}(\yb_t^{(i)} - \eta \bs_t^{(i)} )}^2
				\\
				\le&
				2  \norm{\yb_t - \mathbf{1}\bby_t}^2
				+ 16 \norm{\yb_t - \mathbf{1}\bby_t}^2 
				+ 16\eta^2\norm{\bs_t - \mathbf{1}\bbs_t}^2
				\\
				=& 18 \norm{\yb_t - \mathbf{1}\bby_t}^2 
				+ 16\eta^2\norm{\bs_t - \mathbf{1}\bbs_t}^2,
			\end{split}
		\end{align*}
		where the third and forth inequalities are due to the non-expansiveness of proximal mapping.
	\end{proof}

	\begin{lemma}
		Letting $\geng h(\bby_t)$ be defined in Eq.~\eqref{eq:geng}, then we  have the following error bound for the estimated generalized gradient
		\begin{align}
			\eta \norm{ \bbG_t-\geng h(\bby_t)}
			\le& \frac{4+2M\eta}{\sqrt{m}} \norm{   \yb_t-\mathbf{1}\bby_t }
			+\frac{2\eta}{\sqrt{m}}\norm{\bs_t-\mathbf{1}\bbs_t }. \label{eq:G_h}
		\end{align} 
	\end{lemma}
	\begin{proof}
		It holds that 
		\begin{equation*}
			\begin{aligned}
				&\eta\norm{\bbG_t-\eta\geng h(\bby_t)}
				=\sqrt{\norm{\frac{1}{m} \sum_{i=1}^m \left( \eta G_t^{(i)}-\eta\geng h(\bby_t)\right) }^2}
				\le\sqrt{\frac{1}{m}\cdot\sum_{i=1}^m\norm{  \left( \eta G_t^{(i)}-\eta\geng h(\bby_t)\right) }^2}
				\\
				=&\sqrt{\frac{1}{m}}\cdot\sqrt{\sum_{i=1}^m\norm{   \left(\yb_t^{(i)}- \proximal_{\eta  ,r}(\yb_t^{(i)}-\eta\bs_t^{(i)})  \right)-\left( \bby_t-\proximal_{\eta  ,r}(\bby_t-\eta\nabla f(\bby_t)) \right) }^2}\\
				\le &\sqrt{\frac{1}{m}}\cdot\sqrt{\sum_{i=1}^m\left( 2\norm{   \yb_t^{(i)}-\bby_t }^2
					+2\norm{\proximal_{\eta  ,r}(\yb_t^{(i)}-\eta\bs_t^{(i)})-\proximal_{\eta  ,r}(\bby_t-\eta\nabla f(\bby_t))  }^2\right) }\\
				\le &\sqrt{\frac{1}{m}}\cdot\sqrt{\sum_{i=1}^m\left( 2\norm{   \yb_t^{(i)}-\bby_t }^2
					+2\norm{(\yb_t^{(i)}-\eta\bs_t^{(i)})-(\bby_t-\eta\nabla f(\bby_t))  }^2\right) }\\
				= &\sqrt{\frac{1}{m}}\cdot\sqrt{ 2\norm{   \yb_t-\mathbf{1}\bby_t }^2
					+2\norm{\eta\bs_t-\eta\mathbf{1}\nabla f(\bby_t) +\yb_t-\mathbf{1}\bby_t }^2 }\\
				\le& \sqrt{\frac{1}{m}} \cdot \left(4 \norm{   \yb_t-\mathbf{1}\bby_t }
				+2\eta\norm{\bs_t-\mathbf{1}\bbs_t }+2\eta\norm{\mathbf{1}\bbs_t-\mathbf{1}\nabla f(\bby_t) } \right)\\
				\le&\sqrt{\frac{1}{m}}\cdot\Big((4+2M\eta )\norm{   \yb_t-\mathbf{1}\bby_t }
				+2\eta \norm{\bbs_t-\mathbf{1}\bbs_t }\Big),
			\end{aligned}
		\end{equation*}
		where the second inequality is due to the non-expansiveness of proximal operator, and the last inequality is from Eq.~\eqref{eq:g_y}. 
		
		%
	\end{proof}
	
	\subsection{Proof of Lemma~\ref{lem:yvs}}
	
	\begin{proof}
		For simplicity, we denote FastMix$(\cdot,K)$ operation as $\mathbb{T}(\cdot)$. From Proposition~\ref{lem:mix_eq} we can know that
		\begin{equation*}\label{eq:mix}
			\norm{\mathbb{T}(\xb)-\frac{1}{m}\mathbf{1}\mathbf{1}^\top\xb}\le \rho\norm{\xb-\frac{1}{m}\mathbf{1}\mathbf{1}^\top\xb}.
		\end{equation*}
		First, we have
		\begin{equation}\label{eq:x_error}
			\begin{aligned}
				& \norm{\mathbf{1}\bbx_{t+1}-\xb_{t+1}} \\
				\le & \norm{\proximal_{\eta m,R}(\yb_{t}-\eta \bs_t)-\frac{1}{m}\mathbf{1}\mathbf{1}^\top\proximal_{\eta m,R}(\yb_{t}-\eta \bs_t)}\\
				\le &\norm{\proximal_{\eta m,R}(\yb_{t}-\eta \bs_t)- \proximal_{\eta m,R}\left( \mathbf{1}(\bby_{t}-\eta\bbs_{t})\right)  }\\&+
				\norm{\proximal_{\eta m,R}\left( \mathbf{1}(\bby_{t}-\eta\bbs_{t})\right)-\frac{1}{m}\mathbf{1}\mathbf{1}^\top\proximal_{\eta m,R}(\yb_{t}-\eta \bs_t)}\\
				\le &\norm{\yb_{t}-\mathbf{1}\bby_{t}}+\eta\norm{\bs_t-\mathbf{1}\bbs_{t}}+
				\norm{ \left(\yb_t-\eta\bs_t \right)   -\mathbf{1}\left( \bby_t-\eta\bbs_t \right) } \\
				\le
				&2\norm{\yb_{t}-\mathbf{1}\bby_{t}}+2\eta\norm{\bs_t-\mathbf{1}\bbs_{t}},
			\end{aligned}
		\end{equation}
		where the  third inequality is because of Lemma~\ref{lem:prox average 1} and the non-expansiveness of proximal operator.

		Using the definition of $\yb_{t+1}$ in Algorithm~\ref{alg:DAGD} and the property of ``FastMix" operation, we have
		\begin{equation}\label{eq:y_err}
			\begin{aligned}
				& \norm{\yb_{t+1}-\mathbf{1}\bby_{t+1}} \\
				\le\,&\frac{2\rho}{1+\alpha}\norm{\xb_{t+1}-\mathbf{1}\bbx_{t+1}}+\rho\frac{1-\alpha}{1+\alpha}\norm{\xb_t-\mathbf{1}\bbx_t}\\
				\stackrel{\eqref{eq:x_error}}{\le}&4\rho\norm{\yb_{t}-\mathbf{1}\bby_{t}}+4\rho\eta\norm{\bs_t-\mathbf{1}\bbs_{t}}+\rho\norm{\xb_t-\mathbf{1}\bbx_t}.
			\end{aligned}
		\end{equation}
		Now we are going to bound the value of $\norm{\bs_{t+1}-\mathbf{1}\bbs_{t+1}}$. We have
		\begin{equation*}
            \small
			\begin{aligned}
				&\norm{\bs_{t+1}-\mathbf{1}\bbs_{t+1}}\\
				\le&\rho\norm{\bs_t+\nabla F(\yb_{t+1})-\nabla F(\yb_{t})-\mathbf{1}\cdot\left( \bbs_t+\bbg_{t+1}-\bbg_t\right) }\\
				\stackrel{\eqref{eq:xx_m}}{\le}& \rho\norm{\bs_t-\mathbf{1}\bbs_t}+\rho M\norm{\yb_{t+1}-\yb_{t}}\\
				\le& \rho\norm{\bs_t-\mathbf{1}\bbs_t}+\rho M\norm{\yb_{t+1}-\mathbf{1}\bby_{t+1}}+\rho M\norm{\mathbf{1}\bby_{t+1}-\mathbf{1}\bby_{t}}+\rho M\norm{\mathbf{1}\bby_{t}-\yb_t}\\
				\stackrel{\eqref{eq:y_err}}{\le}& \rho\norm{\bs_t-\mathbf{1}\bbs_t}+\rho M\left( 
				4\rho\norm{\yb_{t}-\mathbf{1}\bby_{t}}+4\rho\eta\norm{\bs_t-\mathbf{1}\bbs_{t}}+\rho\norm{\xb_t-\mathbf{1}\bbx_t}\right)\\& +\rho M\norm{\mathbf{1}\bby_{t+1}-\mathbf{1}\bby_{t}}+\rho M\norm{\mathbf{1}\bby_{t}-\yb_t}\\
				=& \left( \rho+4\rho^2 M\eta\right) \norm{\bs_t-\mathbf{1}\bbs_t}+ \rho^2M\norm{\xb_t-\mathbf{1}\bbx_t} +\left( \rho M+4\rho^2M\right) \norm{\mathbf{1}\bby_{t}-\yb_t}+\rho M\norm{\mathbf{1}\bby_{t+1}-\mathbf{1}\bby_{t}}
				\\
				\le&\rho(1 + 4M\eta)  \norm{\bs_t-\mathbf{1}\bbs_t}+ \rho M\norm{\xb_t-\mathbf{1}\bbx_t} 
				+ 5\rho M\norm{\yb_t - \mathbf{1}\bby_{t}}+\rho M\norm{\mathbf{1}\bby_{t+1}-\mathbf{1}\bby_{t}},
			\end{aligned}
		\end{equation*}
		where the last inequality is because of $\rho \le 1$.
		Then we only need to consider the term $\norm{\mathbf{1}\bby_{t+1}-\mathbf{1}\bby_t}$. 
		Using the iteration of average variables illustrated in Eq.~\eqref{eq:yx}, we have
		\begin{equation*}
			\begin{aligned}
				&\norm{\bby_{t+1}-\bby_t} \\
				= & \norm{ \frac{2}{1+\alpha}\bbx_{t+1}-\frac{1-\alpha}{1+\alpha}\bbx_t-\bby_{t}}
				\\
				=&\norm{ \frac{2}{1+\alpha}(\bby_t - \eta \bbG_t)-\frac{1-\alpha}{1+\alpha}\bbx_t-\bby_{t}} \\
				\le & \frac{1-\alpha}{1+\alpha}\norm{ \bby_t - \bbx_t } + \frac{2 \eta}{1+\alpha} \norm{\bbG_t} 
				\\
				\le& \norm{ \bby_t -x^*} +\norm{ \bbx_t -x^*} + 2\eta\norm{ \bbG_t - \tilde{\nabla}h(\bby_t)} + 2\eta \norm{\tilde{\nabla} h(\bby_t)}
				\\
				\stackrel{\eqref{eq:G_h}}{\le}&
				\norm{ \bby_t -x^*} +\norm{ \bbx_t -x^*} +\frac{8+4M\eta}{\sqrt{m}} \norm{   \yb_t-\mathbf{1}\bby_t }
				+\frac{4\eta}{\sqrt{m}}\norm{\bs_t-\mathbf{1}\bbs_t } 
				+2\eta \norm{\tilde{\nabla} h(\bby_t)}.
			\end{aligned}
		\end{equation*}
		Furthermore, by Lemma~\ref{lem:y_t lyapnov} and the fact that $\geng h(x^*) = 0$, we can obtain
		\begin{align*}
        \small\begin{split}            
			&2\eta \norm{ \geng h(\bby_t)}+\norm{\bbx_t-x^*} +\norm{\bby_t-x^*}
			\\
			=& 2\eta \norm{\geng h(\bby_t) - \geng h(x^*)} + \norm{\bbx_t-x^*} +\norm{\bby_t-x^*}
			\\
			=& 2\norm{ \bby_t - \proximal_{\eta  ,r}(\bby_t - \eta \nabla f(\bby_t)) - ( x^* - \proximal_{\eta  ,r}(x^* - \eta \nabla f(x^*)) ) }
			+\norm{\bbx_t-x^*} +\norm{\bby_t-x^*}
			\\
			\stackrel{\eqref{eq:62}}{\le}&2\norm{\bby_t - x^*} + 2\norm{ \bby_t - x^* } + 2 \eta \norm{\nabla f(\bby_t) - \nabla f(x^*)} 
			+ \norm{\bbx_t-x^*} +\norm{\bby_t-x^*}
			\\
			\le& (5+2\eta L) \norm{\bby_t-x^*} + \norm{\bbx_t-x^*} 
			\\
			\stackrel{\eqref{eq:y_x_V}}{\le}& (5+2\eta L) \sqrt{\frac{2 V_t}{\mu}} + \sqrt{\frac{2}{\mu} (h(\bbx_t) - h(x^*))}
			\\
			\le& 7 \sqrt{\frac{2 V_t}{\mu}},
        \end{split}
		\end{align*}
		where the last equality is because of $\eta = {1}/{(2L)}$.
		Thus, we can obtain that
		\begin{equation*}
			\norm{\mathbf{1}\bby_{t+1}-\mathbf{1}\bby_t} 
			\le
			(8+4M\eta) \norm{   \yb_t-\mathbf{1}\bby_t }
			+4\eta\norm{\bs_t-\mathbf{1}\bbs_t }
			+ 7\sqrt{m} \sqrt{\frac{2}{\mu}V_t}.
		\end{equation*}
		Combining above results, we can bound the value of $\eta\norm{\bs_{t+1}-\mathbf{1}\bbs_{t+1}}$ as follows
		\begin{equation}\label{eq:s_err}
			\begin{aligned}
				\eta \norm{\bs_{t+1} - \mathbf{1} \bbs_{t+1}}
				\le&
				\rho\left( 1 + 8 M\eta \right)\cdot \eta \norm{\bs_t-\mathbf{1}\bbs_t}
				+ \rho \cdot M\eta \norm{\xb_t-\mathbf{1}\bbx_t}\\
				& +\rho M\eta \left( 13+4 M\eta \right) \norm{\yb_t - \mathbf{1}\bby_{t}}+7\rho M \eta \sqrt{m}\sqrt{\frac{2}{\mu}V_t}
				\\
				\le&   \frac{5\rho M}{L} \cdot  \eta \norm{\bs_t-\mathbf{1}\bbs_t}
				+  \frac{\rho M}{L} \norm{\xb_t-\mathbf{1}\bbx_t}\\
				& +\frac{8\rho M^2}{L^2}  \norm{\yb_t - \mathbf{1}\bby_{t}}+ \frac{4\rho M}{L} \sqrt{m}\sqrt{\frac{2}{\mu}V_t},
			\end{aligned}
		\end{equation}
		where the last inequality is because of $\eta = {1}/{(2L)}$ and $1 \le {M}/{L}$.
		
		Combining Eqs.~\eqref{eq:x_error},~\eqref{eq:y_err} and~\eqref{eq:s_err}, we can obtain 
		\begin{align*}
			\zb_{t+1}\le \bA\zb_t+ \frac{4 \rho M\sqrt{m}}{L} \left[0,0,\sqrt{\frac{2V_t}{\mu}}\,\right]^\top,
		\end{align*}
		where
		\begin{align*}		
			\bA=\begin{bmatrix}
				0 & 2 & 2\\
				2\rho& 4\rho &4\rho \\
				\rho M/L&   8\rho M^2/L^2   & 5 \rho M/L
			\end{bmatrix}.
		\end{align*}
	\end{proof}
	
	\subsection{Proof of {Lemma~\ref{lem:lam_max}}} 
	
	\begin{proof}
		It is easy to check that $\mathbf{A}$ is non-negative and irreducible. 
		Furthermore, every diagonal entry of $\mathbf{A}$ is not zero.
		Thus, by Perron-Frobenius theorem and Corollary 8.4.7 of \cite{horn2012matrix}, $\mathbf{A}$ has a real-valued positive number $\lambda_1(\mathbf{A})$ which is algebraically simple and associated with a strictly positive eigenvector $\bv$. 
		It also holds that $\lambda_1(\mathbf{A})$ is strictly larger than $|\lambda_i(\mathbf{A})|$ with~$i=2,3$.
		
		We write down the characteristic polynomial $p(\zeta)$ of ${\mathbf{A}}$, that is
		\begin{align*}
			p(\zeta) = \zeta p_0(\zeta) - 32(M/L)^2\rho^2 + 20\rho^2M/L,
		\end{align*}
		where
		\begin{equation*}
			p_0(\zeta) = \zeta^2 - \rho\left(4+5 M/L \right) \zeta - 4\rho\left(8\rho(M/L)^2+ 5M/L + 1 - 5\rho M/L \right).
		\end{equation*}
		Let us denote 
		\begin{equation}
			\label{eq:Delta}
			\Delta = 16 \rho \left(8\rho(M/L)^2+ 5M/L + 1 - 5\rho M/L \right).
		\end{equation}
		It is easy to check that $\Delta >0$.
		Thus, two roots of $p_0(\zeta)$ are
		\begin{equation*}
			\zeta_1 = \frac{ \rho(4+5 M/L)+\sqrt{ (4+5 M/L)^2\rho^2+\Delta}}{2}
		\end{equation*}		
		and
		\begin{equation*} 
			\zeta_2 = \frac{ \rho(4+5 M/L)-\sqrt{ (4+5 M/L)^2\rho^2+\Delta}}{2}.
		\end{equation*} 
		By letting 
		\begin{equation*}
			\zeta^* = \frac{2\rho\cdot\left(32(M/L)^2+2\right)(4+5 M/L)+\sqrt{\max\{\Delta,\frac{1}{4}\}}}{2},
		\end{equation*}
		we have
		\begin{align*}
			&p\left(\zeta^*\right) + 32(M/L)^2\rho^2 - 20\rho^2M/L \\
			=&	\frac{2\rho\left(32(M/L)^2+2\right)(4+5 M/L)+\sqrt{\max\{\Delta,\frac{1}{4}\}}}{2}
			\\
			&\cdot
			\frac{2\rho\!\left(32(M/L)^2\!+\!2\right)(4\!+\!5 M/L)+\sqrt{\max\{\Delta,\frac{1}{4}\}} - \rho(4\!+\!5M/L) - \sqrt{(4\!+\!5M/L)^2 \rho^2\!+\!\Delta}}{2}
			\\
			&\cdot
			\frac{2\rho\!\left(32(M/L)^2\!+\!2\right)(4\!+\!5 M/L)+\sqrt{\max\{\Delta,\frac{1}{4}\}} - \rho(4\!+\!5M/L) + \sqrt{(4\!+\!5M/L)^2 \rho^2\!+\!\Delta}}{2}\\
			&\ge
			\frac{2\rho\left(32(M/L)^2+2\right)(4+5 M/L)+\sqrt{\max\{\Delta,\frac{1}{4}\}}}{2}
			\\
			&\cdot
			\frac{\left(2\rho\left(32(M/L)^2+1\right)(4+5 M/L)+\sqrt{\max\{\Delta,\frac{1}{4}\}}\,\right)^2 - (\sqrt{(4+5M/L)^2 \rho^2 + \Delta})^2}{2}\\
			=&
			\frac{2\rho\left(32(M/L)^2+2\right)(4+5 M/L)+\sqrt{\max\{\Delta,\frac{1}{4}\}}}{2}\\
			&\cdot
			\bigg( \frac{\left(2\rho\left(32(M/L)^2+1\right)(4+5 M/L)\right)^2+\max\{\Delta,\frac{1}{4}\} - ((4+5M/L)^2 \rho^2 + \Delta)}{2} 
			\\&+
			\left(2\rho\left(32(M/L)^2+1\right)(4+5 M/L)\right) \sqrt{\max\left\{\Delta,\frac{1}{4}\right\}}\,\bigg)\\
			\ge&
			\frac{2\rho\left(32(M/L)^2+2\right)(4+5 M/L)+\sqrt{\max\{\Delta,\frac{1}{4}\}}}{2} \\
			& \cdot 
			\left(2\rho\left(32(M/L)^2+1\right)(4+5 M/L)\right) \sqrt{\max\left\{\Delta,\frac{1}{4}\right\}}\\
			>&
			\frac{\left(2\rho\left(32(M/L)^2+1\right)(4+5 M/L)\right)\cdot \max\{\Delta,\frac{1}{4}\} }{2}
			\\
			\ge&
			\frac{2\rho(32(M/L)^2 +1) \cdot 5}{8}.
		\end{align*}
		Thus, we can obtain that
		\begin{align*}
			p\left(\zeta^*\right) > \frac{2\rho(32(M/L)^2 +1) \cdot 5}{8} -   32(M/L)^2 \rho^2 + 20\rho^2M/L > 0.
		\end{align*}
		Note that $p(\zeta)$ is monotonely increasing in the range $\left[\zeta^*, \infty\right]$.
		Thus, $p(\zeta)$ does not have real roots in this range.
		This implies 
		$\lambda_1(\mathbf{A})\le \zeta^*$.
		By Eq.~\eqref{eq:Delta}, we can obtain that if $\rho$ satisfies the condition that
		\begin{align}
			\rho \le \frac{1}{64\left( 8(M/L)^2 + 5 M/L +1 \right)}, \label{eq:rho1}
		\end{align}
		then it holds that
		$
		\Delta \le \frac{1}{4}
		$.
		If $\rho$ also satisfies the condition that
		\begin{align}
			\rho \le  \frac{1}{4\left(32(M/L)^2 + 2\right) (4+5M/L)}, \label{eq:rho2}
		\end{align}
		then we can obtain that
		$$
		\lambda_1(\mathbf{A})
		\le \zeta^* 
		\le
		\frac{\frac12+\sqrt{\max\{\Delta,\frac{1}{4}\}}}{2}
		=\frac{1}{2}.
		$$
		It is easy to check that if  $\rho \le {L^3}/{(1280M^3)}$, inequalities~\eqref{eq:rho1} and~\eqref{eq:rho2} hold.
		
		Now, we begin to prove that $\sqrt{\rho}<\lambda_1(\mathbf{A})$.
		We can conclude this result once it holds~$p(\sqrt{\rho})<0$.
		This is because $p(\zeta)$ will have a root between $\sqrt{\rho}$ and $1/2$ and $\lambda_1(\mathbf{A})$ must be no less than this root.
		We have
		\begin{align*}
        \small\begin{split}
			& p(\sqrt{\rho}) \\
			=&
			\sqrt{\rho}p_0(\sqrt{\rho}) - 32\rho^2(M/L)^2 + 20\rho^2 M/L
			\\
			=&
			\rho\left(\sqrt{\rho} - \rho(4+5 M/L) - 4\sqrt{\rho} \left(8\rho(M/L)^2+ 5M/L + 1 - 5\rho M/L \right) - 32\rho (M/L)^2 + 20\rho M/L\right)
			\\
			=&
			\rho\left( - (20M/L + 3) \sqrt{\rho}  -\rho(4+ 5M/L)  - 4\rho^{3/2}\left( 8(M/L)^2 - 5 M/L \right) -  \rho \left(32 (M/L)^2  -15 M/L\right)\right)
			\\
			<&
			\rho\left( - (20M/L + 3) \sqrt{\rho}  -\rho(4+ 5M/L)\right)\\
			<& 0,
        \end{split}
		\end{align*}
		where the  first inequality is because of $M/L \ge 1$.
		
		Since $\bv$ is the eigenvector associated with $\lambda_1(\mathbf{A})$, we can obtain that $\mathbf{A}\bv = \lambda_1(\mathbf{A})\bv$ and have the following equations
		\begin{align}
			2\bv(2) + 2 \bv(3) &= \lambda_1(\mathbf{A}) \bv(1), \label{eq:v_1}
			\\
			2\rho \bv(1) + 4\rho \bv(2)+4\rho\bv(3) &= \lambda_1(\mathbf{A})\bv(2),  \label{eq:v_2}
			\\
			\rho \frac{M}{L} \bv(1) + 8\rho\left(\frac{M}{L}\right)^2\bv(2) + 5\rho \frac{M}{L}\bv(3) &= \lambda_1(\mathbf{A})\bv(3). \label{eq:v_3} 
		\end{align}		
		By Eqs.~\eqref{eq:v_1} and~\eqref{eq:v_2}, we can obtain that 
		\begin{align*}
			2\rho\bv(1) = \lambda_1(\bA) \bv(2) - 2\rho \lambda_1(\bA) \bv(1),
		\end{align*}
		which implies that
		\begin{align*}
			\bv(1)=\frac{\lambda_1(\bA)\bv(2)}{2\rho(1+\lambda_1(\bA))}.
		\end{align*}
		Replacing above equation to Eq.~\eqref{eq:v_3}, we can obtain that
		\begin{align*}
			\frac{M\lambda_1(\bA) \bv(2)}{2L(1+ \lambda_1(\bA))}
			=
			\lambda_1(\mathbf{A})\bv(3) - \left(8\rho\left(\frac{M}{L}\right)^2\bv(2) + \frac{5\rho M}{L}\bv(3)\right) 
			< 
			\lambda_1(\mathbf{A})\bv(3),
		\end{align*}
		which implies that
		\begin{align*}
			\bv(2) 
			\le \frac{2L(1+\lambda_1(\bA))\bv(3)}{M} 
			\le \frac{3L\bv(3)}{M}  \le 3 \bv(3),
		\end{align*}
		where the second inequality is because of $\rho \le 1/2$.
		Combining Eq.~\eqref{eq:v_1}, we can obtain that
		\begin{align*}
			\bv(1) 
			\le \frac{2 \left( 3 \bv(3) + \bv(3)\right)}{\lambda_1(\bA)}
			\le \frac{8\bv(3)}{\sqrt{\rho}} ,
		\end{align*}
		where the last inequality is because of $\lambda_1(\bA)\ge \sqrt{\rho}$.
	\end{proof}
	
	\subsection{Proof of {Lemma~\ref{lem:VV}}}
	
	Before proving { Lemma~\ref{lem:VV}}, we first give several important lemmas which are closely related to the convergence rate of Algorithm~\ref{alg:DAGD_p}.
	
	\begin{lemma}
		Letting $\xb_t, \yb_t, \bs_t$ be generated by Algorithm~\ref{alg:DAGD_p}, it holds that
		\begin{equation}\label{eq:hx}
			\begin{aligned}
				h(\bbx_{t+1}) - h(x^*) 
				\le&
				(1 - \alpha) (h(\bbx_t) - h(x^*)) -\dotprod{ \bbG_t, (1-\alpha)\bbx_t+\alpha x^* - \bby_{t}} 
				\\&- \eta\left(\frac{3}{4} - \frac{\eta L}{2}\right) \norm{\bbG_t}^2 
				- \frac{\mu \alpha}{2} \norm{x^* - \bby_t}^2
				\\
				&+ \frac{13\eta}{m} \norm{\bs_t - \mathbf{1}\bbs_t}^2 + \frac{20M^2 \eta + 10\eta^{-1}}{m} \norm{\yb_t - \mathbf{1}\bby_t}^2.
			\end{aligned}
		\end{equation}
	\end{lemma}
	
	\begin{proof}
		By  $\mu$-strong convexity , $L$-smoothness of $f(x)$ and the property of proximal operator, we  have
		\begin{equation}\label{ieq:subgradient}
			\small\begin{aligned}
				&h(\proximal_{\eta,r}(\yb_t^{(i)}-\eta\bs_t^{(i)}))= f(\proximal_{\eta,r}(\yb_t^{(i)}-\eta\bs_t^{(i)}))+r(\proximal_{\eta,r}(\yb_t^{(i)}-\eta\bs_t^{(i)}))\\
				= & f(\yb_t^{(i)}+\proximal_{\eta,r}(\yb_t^{(i)}-\eta\bs_t^{(i)})-\yb_t^{(i)})+r(\proximal_{\eta,r}(\yb_t^{(i)}-\eta\bs_t^{(i)}))\\
				\le& f(\yb_t^{(i)})+\nabla f(\yb_{t}^{(i)})^\top \left( \proximal_{\eta,r}(\yb_t^{(i)}-\eta\bs_t^{(i)})-\yb_t^{(i)}\right) +\frac{L}{2}\norm{\proximal_{\eta,r}(\yb_t^{(i)}-\eta\bs_t^{(i)})-\yb_t^{(i)}}^2   \\
				&+r(z)+\frac{1}{\eta}(\proximal_{\eta ,r}(\yb_t^{(i)}-\eta \bs_t^{(i)})-\yb_t^{(i)}+\eta \bs_t^{(i)})^\top(z-\proximal_{\eta ,r}(\yb_t^{(i)}-\eta \bs_t^{(i)}))\\
				\le& h(z)-\nabla f(\yb_t^{(i)})^\top(z-\yb_t^{(i)})-\frac{\mu}{2}\norm{z-\yb_t^{(i)}}^2+\nabla f(\yb_{t}^{(i)})^\top \left( \proximal_{\eta,r}(\yb_t^{(i)}-\eta\bs_t^{(i)})-\yb_t^{(i)}\right)   \\
				&+\!\frac{L}{2}\norm{\proximal_{\eta,r}(\yb_t^{(i)}\!-\!\eta\bs_t^{(i)})\!-\!\yb_t^{(i)}}^2\!+\!\frac{1}{\eta}(\proximal_{\eta ,r}(\yb_t^{(i)}\!-\!\eta \bs_t^{(i)})\!-\!\yb_t^{(i)}\!+\!\eta \bs_t^{(i)})^\top(z\!-\!\proximal_{\eta ,r}(\yb_t^{(i)}\!-\!\eta \bs_t^{(i)}))
				\\
				=&h(z) - \dotprod{\nabla f(\yb_t^{(i)}), z - \yb_t^{(i)}} - \frac{\mu}{2}\norm{z - \yb_t^{(i)}}^2
				-\eta \dotprod{\nabla f(\yb_t^{(i)}), G_t^{(i)}} + \frac{\eta^2 L}{2} \norm{G_t^{(i)}}^2
				\\
				&+\dotprod{\bs_t^{(i)} - G_t^{(i)}, z -\yb_t^{(i)} + \eta G_t^{(i)}}
				\\
				=&h(z) - \dotprod{ \nabla f(\yb_t^{(i)}) - \bs_t^{(i)} + G_t^{(i)} , z-\yb_t^{(i)}}  - \frac{\mu}{2} \norm{z - \yb_t^{(i)}}^2 
				\\
				& - \eta \dotprod{\nabla f(\yb_t^{(i)}) - \bs_t^{(i)}, G_t^{(i)}} - \eta\left( 1 - \frac{\eta L}{2}  \right)\norm{G_t^{i}}^2 
			\end{aligned}
		\end{equation}
		for any given $z\in\RR^{d}$.
		Multiplying $1-\alpha$ on both sides of Eq.~\eqref{ieq:subgradient} and setting $z=\bbx_t$, we get
		\begin{equation*}
			\begin{aligned}
				&(1-\alpha)h(\proximal_{\eta,r}(\yb_t^{(i)}-\eta\bs_t^{(i)}))\\
				\le& (1-\alpha)h(\bbx_t)- (1 - \alpha) \dotprod{ \nabla f(\yb_t^{(i)}) - \bs_t^{(i)} + G_t^{(i)} , \bbx_t-\yb_t^{(i)}}  - \frac{\mu(1 - \alpha)}{2} \norm{\bbx_t - \yb_t^{(i)}}^2 
				\\
				& - (1 - \alpha) \eta \dotprod{\nabla f(\yb_t^{(i)}) - \bs_t^{(i)}, G_t^{(i)}} - (1-\alpha)\eta\left( 1 - \frac{\eta L}{2}  \right)\norm{G_t^{i}}^2.
			\end{aligned}
		\end{equation*}
		Similarly, multiplying $\alpha$ on both sides of Eq.~\eqref{ieq:subgradient} and setting $z=x^*$, we obtain that
		\begin{equation*}
			\begin{aligned}
				&\alpha h(\proximal_{\eta,r}(\yb_t^{(i)}-\eta\bs_t^{(i)}))\\
				\le& \alpha h(x^*)-\alpha \dotprod{ \nabla f(\yb_t^{(i)}) - \bs_t^{(i)} + G_t^{(i)} , x^*-\yb_t^{(i)}}  - \frac{\mu \alpha}{2} \norm{x^* - \yb_t^{(i)}}^2 
				\\
				& - \alpha \eta \dotprod{\nabla f(\yb_t^{(i)}) - \bs_t^{(i)}, G_t^{(i)}} - \alpha \eta\left( 1 - \frac{\eta L}{2}  \right)\norm{G_t^{i}}^2.
			\end{aligned}
		\end{equation*}
		Adding above two inequalities, we have
		\begin{equation}\label{eq:added}
			\begin{aligned}
				&h(\proximal_{\eta,r}(\yb_t^{(i)}-\eta\bs_t^{(i)}))-h(x^*)\\
				\le& (1-\alpha)\left( h(\bbx_t)-h(x^*)\right)   - \dotprod{ \nabla f(\yb_t^{(i)}) - \bs_t^{(i)} + G_t^{(i)} , (1-\alpha) \bbx_t + \alpha x^*-\yb_t^{(i)}}\\
				& - \eta \dotprod{\nabla f(\yb_t^{(i)}) - \bs_t^{(i)}, G_t^{(i)} } 
				- \eta\left( 1 - \frac{\eta L}{2}\right) \norm{G_t^i}^2 
				- \frac{\mu \alpha}{2} \norm{x^* - \yb_t^i}^2.
			\end{aligned}
		\end{equation}
		Note that by Jensen's inequality, we can get that
		\begin{equation}
			\norm{x^*-\bby_t} = \norm{x^*-\frac{1}{m}\sum_{i=1}^{m}\yb_t^{(i)}} \le \sqrt{\frac{1}{m}\sum_{i=1}^{m}\norm{x^*-\yb_t^{(i)}}^2}. \label{eq:aa}
		\end{equation} 
		Then averaging Eq.~\eqref{eq:added} from $i=1$ to $m$ and using the convexity of $h(x)$, we have
		\begin{equation}\label{eq:h_distance}
			\begin{aligned}
				&h(\bbx_{t+1})-h(x^*)\\
				\le& \frac{1}{m}\sum_{i=1}^{m}h(\proximal_{\eta,r}(\yb_t^{(i)}-\eta\bs_t^{(i)}))-h(x^*) \\
				\stackrel{\eqref{eq:added}}{\le}& 
				(1 - \alpha)( h (\bbx_t) - h(x^*)) 
				-\frac{1}{m} \sum_{i=1}^m  \dotprod{ \nabla f(\yb_t^{(i)}) - \bs_t^{(i)} + G_t^{(i)} , (1-\alpha) \bbx_t + \alpha x^*-\yb_t^{(i)}}\\
				& - \frac{\eta}{m} \sum_{i=1}^m\dotprod{\nabla f(\yb_t^{(i)}) - \bs_t^{(i)}, G_t^{(i)} } 
				- \eta\left( 1 - \frac{\eta L}{2}\right) \sum_{i=1}^m \norm{G_t^i}^2 
				- \frac{\mu \alpha}{2} \frac{1}{m}\sum_{i=1}^m\norm{x^* - \yb_t^i}^2\\
				\stackrel{\eqref{eq:aa}}{\le}&
				(1 - \alpha)( h (\bbx_t) - h(x^*)) 
				-\frac{1}{m} \sum_{i=1}^m  \dotprod{ \nabla f(\yb_t^{(i)}) - \bs_t^{(i)} + G_t^{(i)} , (1-\alpha) \bbx_t + \alpha x^*-\yb_t^{(i)}}\\
				& - \frac{\eta}{m} \sum_{i=1}^m\dotprod{\nabla f(\yb_t^{(i)}) - \bs_t^{(i)}, G_t^{(i)} } 
				- \eta\left( 1 - \frac{\eta L}{2}\right) \frac{1}{m} \sum_{i=1}^m \norm{G_t^i}^2 
				- \frac{\mu \alpha}{2} \norm{x^* - \bby_t}^2.
			\end{aligned}
		\end{equation}
		Furthermore, we have
		\begin{align*}
			&\frac{1}{m}\sum_{i=1}^{m}\left(  \bs_t^{(i)}- G_t^{(i)}-\nabla f(\yb_t^{(i)})\right) ^\top\left( (1-\alpha)\bbx_t+\alpha x^* -\yb_t^{(i)}  \right)
			\\
			=& \frac{1}{m}\sum_{i=1}^{m}\left(  \bs_t^{(i)}- G_t^{(i)}-\nabla f(\yb_t^{(i)})\right) ^\top \left( (1-\alpha)\bbx_t+\alpha x^* - \bby_{t} + \bby_{t} - \yb_{t}^{(i)} \right)
			\\
			\stackrel{\eqref{eq:bs}}{=}&
			-\dotprod{\bbG_t, (1-\alpha)\bbx_t+\alpha x^* - \bby_{t}} 
			+
			\frac{1}{m}\sum_{i=1}^m\dotprod{\bs_t^{(i)}- G_t^{(i)}-\nabla f(\yb_t^{(i)}), \bby_{t} - \yb_{t}^{(i)} }  
		\end{align*}
		and
		\begin{align*}
        \small\begin{split}   
			&\frac{1}{m}\sum_{i=1}^m\dotprod{\bs_t^{(i)}- G_t^{(i)}-\nabla f(\yb_t^{(i)}), \bby_{t} - \yb_{t}^{(i)} }  \\
			=~~\,&
			\frac{1}{m}\sum_{i=1}^m\dotprod{\bs_t^{(i)}- G_t^{(i)}-\nabla f(\yb_t^{(i)}) + \bbG_t, \bby_{t} - \yb_{t}^{(i)} }
			\\
			=~~\,&\frac{1}{m}\sum_{i=1}^m\dotprod{\bs_t^{(i)}-\nabla f(\yb_t^{(i)}), \bby_{t} - \yb_{t}^{(i)} }
			+
			\frac{1}{m}\sum_{i=1}^m\dotprod{ \bbG_t - G_t^{(i)}, \bby_{t} - \yb_{t}^{(i)} }
			\\
			\le~~\,&
			\sqrt{\frac{1}{m} \sum_{i=1}^m \norm{\bs_t^{(i)}-\nabla f(\yb_t^{(i)})}^2} \cdot \sqrt{\frac{1}{m}\sum_{i=1}^m \norm{\bby_t - \yb_t^{(i)}}^2} \\
			& + \sqrt{\frac{1}{m} \sum_{i=1}^m \norm{G_t^{(i)}- \bbG_t}^2} \cdot \sqrt{\frac{1}{m}\sum_{i=1}^m \norm{\bby_t - \yb_t^{(i)}}^2}
			\\
			\le~~\,&
			\frac{\eta}{m} \cdot \frac{\sum_{i=1}^m \norm{\bs_t^{(i)}-\nabla f(\yb_t^{(i)})}^2  + \sum_{i=1}^m \norm{G_t^{(i)}- \bbG_t}^2}{2} +  \frac{1}{m \eta }\norm{\yb_t - \mathbf{1}\bby_t}^2
			\\
			\stackrel{\eqref{eq:s_gd},\eqref{eq:GG}}{\le}&  \frac{\eta }{m} \left(9 \norm{\bs_t - \mathbf{1}\bbs_t}^2 + 4M^2\norm{\yb_t - \mathbf{1}\bby_t}^2 + 9\eta^{-2} \norm{\yb_t - \mathbf{1} \bby_t}^2\right) + \frac{1}{\eta m} \norm{\yb_t - \mathbf{1}\bby_t}^2
			\\
			=~~\,&
			\frac{9\eta}{m} \norm{\bs_t - \mathbf{1}\bbs_t}^2 + \frac{4M^2 \eta + 10\eta^{-1}}{m} \norm{\yb_t - \mathbf{1}\bby_t}^2,
    \end{split}
		\end{align*}
		where the first inequality is because of Cauchy's inequality and the second inequality is because of~$2ab \le \eta a^2 + b^2/\eta$.
		
		Combining above two inequalities, we can obtain that
		\begin{equation}\label{eq:g_err_1}
			\begin{aligned}
				&\frac{1}{m}\sum_{i=1}^m\dotprod{\bs_t^{(i)}- G_t^{(i)}-\nabla f(\yb_t^{(i)}), \bby_{t} - \yb_{t}^{(i)} }
				\\
				\le&
				-\dotprod{\bbG_t, (1-\alpha)\bbx_t+\alpha x^* - \bby_{t}} 
				+\frac{9\eta}{m} \norm{\bs_t - \mathbf{1}\bbs_t}^2 + \frac{4M^2\eta + 10\eta^{-1}}{m} \norm{\yb_t - \mathbf{1}\bby_t}^2.
			\end{aligned}
		\end{equation}
		
		Moreover, we have
		\begin{equation}\label{eq:g_err_2}
			\begin{aligned}
				&-\frac{\eta}{m} \sum_{i=1}^m\dotprod{\nabla f(\yb_t^{(i)}) - \bs_t^{(i)}, G_t^{(i)} } \\
				\le\,&
				\frac{\eta}{m} \sum_{i=1}^m \norm{\nabla f(\yb_t^{(i)}) - \bs_t^{(i)}} \norm{G_t^{(i)}}
				\\
				\le\,&
				\frac{\eta}{m} \sum_{i=1}^m \left(2 \norm{\nabla f(\yb_t^{(i)}) - \bs_t^{(i)}}^2 + \frac{1}{4}\norm{G_t^{(i)}}^2\right)
				\\
				\stackrel{\eqref{eq:s_gd}}{\le}&
				\frac{4\eta}{m} \norm{\bs_t -\mathbf{1} \bbs_t}^2 + \frac{16M^2\eta}{m} \norm{\yb_t - \mathbf{1}\bby_t}^2 + \frac{\eta}{4m} \sum_{i=1}^m \norm{G_t^{(i)}}^2.
			\end{aligned}
		\end{equation}
		
		Combining Eqs.~\eqref{eq:h_distance},~\eqref{eq:g_err_1} and~\eqref{eq:g_err_2}, we can obtain that
		\begin{equation*}
			\begin{aligned}
				& h(\bbx_{t+1}) - h(x^*) \\
				\le& (1 - \alpha) (h(\bbx_t) - h(x^*)) -\dotprod{\bbG_t, (1-\alpha)\bbx_t+\alpha x^* - \bby_{t}} 
				\\&- \eta\left(\frac{3}{4} - \frac{\eta L}{2}\right) \frac{1}{m} \sum_{i=1}^m \norm{G_t^{(i)}}^2 
				- \frac{\mu \alpha}{2} \norm{x^* - \bby_t}^2
				\\
				&+ \frac{13\eta}{m} \norm{\bs_t - \mathbf{1}\bbs_t}^2 + \frac{20M^2 \eta + 10\eta^{-1}}{m} \norm{\yb_t - \mathbf{1}\bby_t}^2
				\\
				\le&
				(1 - \alpha) (h(\bbx_t) - h(x^*)) -\dotprod{\bbG_t, (1-\alpha)\bbx_t+\alpha x^* - \bby_{t}} 
				\\&- \eta\left(\frac{3}{4} - \frac{\eta L}{2}\right) \norm{\bbG_t}^2 
				- \frac{\mu \alpha}{2} \norm{x^* - \bby_t}^2
				\\
				&+ \frac{13\eta}{m} \norm{\bs_t - \mathbf{1}\bbs_t}^2 + \frac{20M^2 \eta + 10\eta^{-1}}{m} \norm{\yb_t - \mathbf{1}\bby_t}^2,
			\end{aligned}
		\end{equation*}
		where the last inequality is because of Jensen's inequality.
	\end{proof}
	
	\begin{lemma}
		Letting $\xb_t, \yb_t, \bs_t$ be generated by Algorithm~\ref{alg:DAGD_p}, it holds that
		\begin{equation} \label{eq:vv_up}
			\begin{aligned}
				\frac{\mu}{2}\norm{\bbv_{t+1}-x^*}^2
				\le&
				\frac{(1-\alpha)\mu}{2}\norm{\bbv_t - x^*}^2 + \frac{\alpha \mu}{2} \norm{\bby_t - x^*}^2
				\\
				&+   \dotprod{\bbG_t, (1 - \alpha) \bbx_t + \alpha x^* - \bby_t}
				+\frac{ \eta}{2} \norm{\bbG_t}^2. 
			\end{aligned}
		\end{equation}
	\end{lemma}
	\begin{proof}
		We have
		\begin{align*}
			& \frac{\mu}{2}\norm{\bbv_{t+1}-x^*}^2 \\
			\stackrel{\eqref{eq:vv_p}}{=}~~&\frac{\mu}{2}\norm{(1 - \alpha) \bbv_t  + \alpha \bby_t - \frac{\eta}{\alpha} \bbG_t-x^*}^2\\
			=~~\,&
			\frac{\mu}{2} \norm{(1 - \alpha) \bbv_t  + \alpha \bby_t - x^*}^2 
			-  \frac{\mu \eta}{\alpha} \dotprod{\bbG_t, (1 - \alpha) \bbv_t  + \alpha \bby_t - x^*}
			+\frac{\mu \eta^2}{2\alpha^2} \norm{\bbG_t}^2\\
			\stackrel{\alpha = \sqrt{\mu \eta}}{=}&
			\frac{\mu}{2} \norm{(1 - \alpha) \bbv_t  + \alpha \bby_t - x^*}^2 
			-  \alpha \dotprod{\bbG_t, (1 - \alpha) \bbv_t  + \alpha \bby_t - x^*}
			+\frac{ \eta}{2} \norm{\bbG_t}^2.
		\end{align*}
		Furthermore, by Eq.~\eqref{eq:yxv}, we have  $\bbv_t = \bby_t + \frac{1}{\alpha} (\bby_t - \bbx_t)$ which implies that
		\begin{align*}
			(1 - \alpha)\bbv_t + \alpha \bby_t = \bby_t + \frac{1 - \alpha}{\alpha} (\bby_t - \bbx_t).
		\end{align*}
		Thus, it holds that
		\begin{align*}
			-\alpha \dotprod{\bbG_t, (1 - \alpha) \bbv_t  + \alpha \bby_t - x^*} 
			= \dotprod{\bbG_t, (1 - \alpha) \bbx_t + \alpha x^* - \bby_t}.
		\end{align*}
		It also holds that
		\begin{align*}
			& \norm{(1 - \alpha) \bbv_t  + \alpha \bby_t - x^*}^2 \\
			\le &
			\left( (1 - \alpha) \norm{\bbv_t - x^*} + \alpha \norm{\bby_t - x^*} \right)^2 \\
			\le & (1 - \alpha) \norm{\bbv_t - x^*}^2 + \alpha \norm{\bby_t - x^*}^2.
		\end{align*}
		Therefore, it holds that
		\begin{equation*} 
			\small\begin{aligned}
				\frac{\mu}{2}\norm{\bbv_{t+1}-x^*}^2
				\le
				\frac{(1-\alpha)\mu}{2}\norm{\bbv_t - x^*}^2 + \frac{\alpha \mu}{2} \norm{\bby_t - x^*}^2
				+   \dotprod{\bbG_t, (1 - \alpha) \bbx_t + \alpha x^* - \bby_t}
				+\frac{ \eta}{2} \norm{\bbG_t}^2 .
			\end{aligned}
		\end{equation*}		
	\end{proof}
	
	Combining above two lemmas, we can obtain the following result.
	\begin{proof}[Proof of Lemma~\ref{lem:VV}]
		Using the definition of $V_t$, we have
		\begin{align*}
        \small\begin{split}        
			V_{t+1} 
			=~~& 
			h(\bbx_{t+1}) - h(x^*) + \frac{\mu}{2}\norm{\bbv_{t+1}-x^*}^2
			\\
			\stackrel{\eqref{eq:hx},\eqref{eq:vv_up}}{\le}& (1 - \alpha) V_t - \eta\left(\frac{1}{4} - \frac{\eta L}{2}\right) \norm{\bbG_t}^2 
			+ \frac{13\eta}{m} \norm{\bs_t - \mathbf{1}\bbs_t}^2 + \frac{20M^2 \eta + 10\eta^{-1}}{m} \norm{\yb_t - \mathbf{1}\bby_t}^2
			\\
			\le~~& (1 - \alpha) V_t  
			+ \frac{13\eta}{m} \norm{\bs_t - \mathbf{1}\bbs_t}^2 + \frac{20M^2 \eta + 10\eta^{-1}}{m} \norm{\yb_t - \mathbf{1}\bby_t}^2,
        \end{split}
		\end{align*}
		where the last inequality is because of $\eta = {1}/{(2L)}$.
	\end{proof}

	\section{Convergence Analysis of Algorithm~\ref{alg:DAGD}}
	\label{app:mudag}
	
	The proof of Algorithm~\ref{alg:DAGD} is almost the same to the one of Algorithm~\ref{alg:DAGD_p}.
	But, without the proximal mapping which will cause extra consensus error terms, the detailed convergence analysis of Algorithm~\ref{alg:DAGD} is clean and easy to follow. 
	
	\begin{lemma}
		\label{lem:proc}
		The update procedure of Algorithm~\ref{alg:DAGD} can be represented as
		\begin{align}
			\xb_{t+1} =& \mathrm{FastMix}\left(\yb_t - \eta\bs_t, K\right), \label{eq:x_y_s}
			\\
			\yb_{t+1} =& \xb_{t+1} + \frac{1-\alpha}{1+\alpha}(\xb_{t+1} - \xb_t),
			\\
			\bs_{t+1} =& \mathrm{FastMix}(\bs_t  , K) + (\nabla F(\yb_{t+1}) -\nabla F(\yb_t))- \eta^{-1} (\mathrm{FastMix}(\yb_t, K) - \yb_t), \label{eq:sp}
		\end{align}
		with $\bs_0 = \nabla F(\yb_0)$.
	\end{lemma}
	\begin{proof}
		The proof of this reformulation is equivalent to  prove that
		given the reformulation of $\xb_t$, $\yb_t$ and $\bs_t$ at
		iteration $t$, the  reformulation of $\xb_{t+1}$ holds at iteration $t+1$.
		Therefore our induction focuses on $\xb_{t+1}$.
		First, when $t=0$, we can obtain that
		\begin{equation}
			\xb_{1} = \TB(\yb_0 - \eta\nabla F(\yb_0)) = \TB(\yb_0 -\eta \bs_0),
		\end{equation} 
		which implies that
		\begin{equation*}
			\xb_1 - \yb_0 = -\eta \TB(\bs_0) + \TB(\yb_0) - \yb_0.
		\end{equation*}
		Furthermore, have
		\begin{equation*}
			\bs_1 = \TB(\bs_0) + (\nabla F(\yb_2) - \nabla F(\yb_1)) - \eta^{-1}(\TB(\yb_0) - \yb_0).
		\end{equation*}
		Thus,  we can obtain that
		\begin{align*}
			\xb_{2} 
			\stackrel{\eqref{eq:xb_up}}{=}& 
			\TB(\yb_1 + (\xb_1 - \yb_0) - \eta(\nabla F(\yb_1) - \nabla F(\yb_0))) 
			\\
			=\,&\TB(\yb_1 - \left(\eta \TB(\bs_0) + \eta(\nabla F(\yb_1) - \nabla F(\yb_0))\right) + \TB(\yb_0) - \yb_0)
			\\
			=\,&\TB(\yb_1 - \eta \bs_1),
		\end{align*}
		where the first equation is because of the update of Algorithm~\ref{alg:DAGD}. 
		We obtain that the result holds at~$t = 0$.
		
		Next, we prove that if the results hold in the $t$-th
		iteration, then it also holds at the $(t+1)$-th iteration.  
		For the $t$-th iteration, we assume that
		$\xb_{t+1} = \TB(\yb_t - \eta\bs_t)$,
		which implies that 
        \begin{align*}
            \xb_{t+1} - \TB(\yb_t) = -\eta \TB(\bs_t).
        \end{align*}
		Therefore, we obtain that
		\begin{align*}
			\xb_{t+2} 
			\stackrel{\eqref{eq:xb_up}}{=}& 
			\TB(\yb_{t+1}+\xb_{t+1} - \yb_t - \eta(\nabla F(\yb_{t+1}) - \nabla F(\yb_t)))
			\\
			=\,&
			\TB(\yb_{t+1} + \xb_{t+1} - \TB(\yb_t)  - \eta(\nabla F(\yb_{t+1}) - \nabla F(\yb_t)) + \TB(\yb_t) - \yb_t)
			\\
			=\,&\TB(\yb_{t+1} - \eta \TB(\bs_t) - \eta(\nabla F(\yb_{t+1}) - \nabla F(\yb_t)) + \TB(\yb_t) - \yb_t)
			\\
			=\,&
			\TB(\yb_{t+1} - \eta \bs_{t+1}).
		\end{align*}
		This proves the desired result. 
	\end{proof}
	
	We now show that $\bbx_t$, $\bby_t$, $\bbg_t$ (defined in
	Eq.~\eqref{eq:xyg} and generated by Algorithm~\ref{alg:DAGD}) and
	$\bbv_t$ (defined in Eq.~\eqref{eq:v_t}) can be fit into the framework of the centralized Nesterov's accelerated gradient descent.

	\begin{lemma}
		Let  $\bbx_t$, $\bby_t$, $\bbg_t$ (defined in Eq.~\eqref{eq:xyg}) be generated by Algorithm~\ref{alg:DAGD}. 
		Then they satisfy the following equalities:
		\begin{align}
        \begin{split}        
			\bbx_{t+1} =& \bby_t - \eta\bbg_t, \\
			\bby_{t+1} =& \bbx_{t+1} + \frac{1-\alpha}{1+\alpha}(\bbx_{t+1} - \bbx_t), \\
			\bbs_{t+1}=& \bbs_t+\bbg_{t+1} - \bbg_t = \bbg_{t+1}. 
        \end{split}\label{eq:bs_1}
		\end{align}
	\end{lemma} 
	\begin{proof}
		We first prove the last equality. 
		First, we have $\frac{1}{m}\mathbf{1}\mathbf{1}^\top(\TB(\yb_t) - \yb_t) = \mathbf{1}\bby_t - \mathbf{1}\bby_t = 0$ by Lemma~\ref{lem:mix_eq}.
		Thus, we can obtain that
		\begin{equation*}
			\bbs_{t+1}= \bbs_t+\bbg_{t+1} - \bbg_t.
		\end{equation*}
		Furthermore, we will prove $\bbs_t = \eta\bbg_t$ by induction. For $t=0$, we use the fact that $\bs_0 = \eta\nabla F(\yb_0)$.  
		Then, it holds that $\bbs_0 = \bbg_0$.
		We assume that $\bbs_t = \bbg_t$ at time $t$. 
		By the update equation, we have
		\begin{equation*}
			\bbs_{t+1} = \bbs_t + (\bbg_{t+1} - \bbg_t) = \bbg_{t+1}.
		\end{equation*}
		Thus, we obtain the result at time $t+1$. 
		The first two equations can be proved using Eq.~\eqref{eq:bs_1} and Proposition~\ref{lem:mix_eq}.
	\end{proof}
	
	\begin{lemma}
		\label{lem:yvs_1}
		Let $\zb_t = \left[\norm{\yb_t - \mathbf{1}\bby_t}, \rho^{-1}\norm{\xb_t - \mathbf{1}\bbx_t}, M^{-1}\norm{\bs_t - \mathbf{1}\bbs_t}\right]^\top$ with $\xb_t$ and $\yb_t$ generated by Algorithm~\ref{alg:DAGD} and $\bs_t$ defined in Eq.~\eqref{eq:sp}, then it holds that
		\begin{equation}
			\label{eq:zaz_1}
			\zb_{t+1} \le \mathbf{A}\zb_t + 4\sqrt{m}\left[0,0,\sqrt{\frac{2}{\mu}V_t}\,\right]^\top,
		\end{equation}
		where $\rho$ and $\mathbf{A}$ are defined as 
		\begin{equation*}
			\rho = \sqrt{14} \left(1 - \left(1-\frac{1}{\sqrt{2}}\,\right) \sqrt{1-\lambda_2(W)}\right)^K 
            \qquad \text{and} \qquad
			\mathbf{A} 
			\triangleq  	
			\begin{bmatrix}
				2\rho & \rho & 2\rho M\eta\\
				1 & 0 & M\eta\\
				9M\eta & \rho &  3\rho M\eta
			\end{bmatrix}.
		\end{equation*}
	\end{lemma}
	\begin{proof}
		By the update step of $\yb_{t+1}$ in Algorithm~\ref{alg:DAGD}, we have
		\begin{align*}
			\norm{\yb_{t+1} - \mathbf{1}\bby_{t+1}} 
			\le
			\frac{2}{1+\alpha}\norm{\xb_{t+1} - \mathbf{1}\bbx_{t+1}} + \frac{1-\alpha}{1+\alpha}\norm{\xb_t - \mathbf{1}\bbx_t}.
		\end{align*}
		Furthermore, by Eq.~\eqref{eq:x_y_s}, we have
		\begin{align*}
			\frac{1}{\rho}\norm{\xb_{t+1} - \mathbf{1}\bbx_{t+1}} \le \norm{\yb_t - \mathbf{1}\bby_t}+M\eta\cdot\frac{1}{M}\norm{\bs_t-\mathbf{1}\bbs_t}.
		\end{align*}
		Therefore, we can obtain that
		\begin{align*}
			\norm{\yb_{t+1} - \mathbf{1}\bby_{t+1}} 
			\le&
			\frac{2\rho}{1+\alpha} \norm{\yb_t - \mathbf{1}\bby_t}
			+
			\frac{1-\alpha}{1+\alpha}\norm{\xb_t - \mathbf{1}\bbx_t} 
			+
			\frac{2\rho \eta }{1+\alpha} \norm{\bs_t-\mathbf{1}\bbs_t}
			\\
			\le&
			2\rho\norm{\yb_t - \mathbf{1}\bby_t}
			+
			\rho\cdot\frac{1}{\rho}\norm{\xb_t - \mathbf{1}\bbx_t} 
			+
			2\rho M\eta\cdot\frac{1}{M} \norm{\bs_t-\mathbf{1}\bbs_t}.
		\end{align*}		
		Furthermore, by Eq.~\eqref{eq:sp}, we have
		\begin{align*}
			&\eta \norm{\bs_{t+1}-\mathbf{1}\bbs_{t+1}} 
			\\
			\le\,& 
			\eta \norm{\TB(\bs_t) -\mathbf{1}\bbs_t} 
			+ 
			\eta\norm{\nabla F(\yb_{t+1}) - \nabla F(\yb_t) -\mathbf{1}(\bbg_{t+1} - \bbg_t)} 
			+
			\norm{\TB(\yb_t) - \yb_t}
			\\
			\stackrel{\eqref{eq:xx_m}}{\le}&
			\eta \norm{\TB(\bs_t) -\mathbf{1}\bbs_t} 
			+ 
			\eta\norm{\nabla F(\yb_{t+1}) - \nabla F(\yb_t)} 
			+
			\norm{\TB(\yb_t) - \yb_t}
			\\
			\overset{\eqref{eq:M_sm}}{\le}&
			\rho\cdot \eta \norm{\bs_t - \mathbf{1}\bbs_t}
			+
			M\eta\norm{\yb_{t+1} - \yb_t}
			+\norm{\TB(\yb_t) - \yb_t}
			\\
			\le\,&
			\rho\cdot \eta \norm{\bs_t - \mathbf{1}\bbs_t}
			+
			M\eta\norm{\yb_{t+1} - \yb_t}
			+2\norm{\yb_t - \mathbf{1}\bby_t},
		\end{align*}
		where the last inequality is because of 
		\begin{align*}
			\norm{\TB(\yb_t) - \yb_t} = \norm{\TB(\yb_t) -\mathbf{1}\bby_t + \mathbf{1}\bby_t - \yb_t} 
			\leq (1+\rho) \norm{\yb_t - \mathbf{1}\bby_t}
			\leq 2\norm{\yb_t - \mathbf{1}\bby_t}.
		\end{align*}		
		By the update rule of $\yb_{t+1}$, we have
		\begin{align*}
		\small\begin{split}
		    \norm{\yb_{t+1} - \yb_t} 
			=& 
			\norm{\frac{2}{1+\alpha}\xb_{t+1} - \frac{1-\alpha}{1+\alpha}\xb_t  - \yb_t}
			\\
			\overset{\eqref{eq:x_y_s}}{=}&
			\norm{\frac{2}{1+\alpha}\TB(\yb_t -\eta \bs_t) - \frac{1-\alpha}{1+\alpha}\xb_t - \yb_t}
			\\
			\le&
			\frac{2}{1+\alpha}\norm{\TB(\yb_t) - \yb_t} + \frac{1-\alpha}{1+\alpha}\norm{\xb_t-\yb_t} + \frac{2\eta}{1+\alpha}\norm{\TB(\bs_t)}
			\\
			\le&
			\frac{4}{1+\alpha}\norm{\yb_t-\mathbf{1}\bby_t} + \frac{1-\alpha}{1+\alpha}(\norm{\xb_t-\mathbf{1}\bbx_t} + \norm{\yb_t-\mathbf{1}\bby_t}
			\\&+ \norm{\mathbf{1}(\bby_t-x^*)} + \norm{\mathbf{1}(\bbx_t - x^*)})
			+\frac{2\eta}{1+\alpha}(\norm{\TB(\bs_t) - \mathbf{1}\bbs_t} + \norm{\mathbf{1}\bbs_t}) 
			\\
			\overset{\eqref{eq:bs_1}}{\le}&
			\frac{5}{1+\alpha}\norm{\yb_t-\mathbf{1}\bby_t} + \frac{2\rho}{1+\alpha}\cdot \eta \norm{\bs_t-\mathbf{1}\bbs_t} + \frac{1-\alpha}{1+\alpha}(\norm{\xb_t -\mathbf{1}\bbx_t}) \\
			&+ \frac{1-\alpha}{1+\alpha}\left(\norm{\mathbf{1}(\bby_t - x^*)} + \norm{\mathbf{1}(\bbx_t - x^*)}\right) + \frac{2\eta\sqrt{m}}{1+\alpha}\norm{\bbg_t}.
		\end{split}	
		\end{align*}
		Furthermore, by Eq.~\eqref{eq:g_y}, we have
		\begin{align*}
			\norm{\bbg_t} \le \norm{\bbg_t - \nabla f(\bby_t)} + \norm{\nabla f(\bby_t)} \le \frac{M}{\sqrt{m}}\norm{\yb_t - \mathbf{1}\bby_t} + \norm{\nabla f(\bby_t)}.
		\end{align*}
		Therefore, we can obtain that
		\begin{align*}
        \small\begin{split}            
			&\frac{1}{M} \norm{\bs_{t+1}-\mathbf{1}\bbs_{t+1}}  
			\\
			\le&
			\rho(1+2\rho M\eta) \cdot \frac{1}{M} \norm{\bs_t-\mathbf{1}\bbs_t}
			+
			\left(\frac{5+2M\eta}{1+\alpha} + \frac{2}{M\eta}\right)\norm{\yb_t-\mathbf{1}\bby_t}
			\\
			& + \frac{1-\alpha}{1+\alpha}\norm{\xb_t-\mathbf{1}\bbx_t} 
			+ \frac{1-\alpha}{1+\alpha} \left(\norm{\mathbf{1}(\bby_t - x^*)} + \norm{\mathbf{1}(\bbx_t - x^*)}\right)
			+
			\frac{2\eta\sqrt{m}}{1+\alpha}\norm{\nabla f(\bby_t)}
			\\
			\le&
			\rho(1+2\rho  M\eta)\cdot\frac{1}{M}\norm{\bs_t-\mathbf{1}\bbs_t}
			+
			\left(7+2M\eta \right)\norm{\yb_t-\mathbf{1}\bby_t}
			+ \norm{\xb_t-\mathbf{1}\bbx_t} 
			\\
			&
			+ \norm{\mathbf{1}(\bby_t - x^*)} + \norm{\mathbf{1}(\bbx_t - x^*)} 
			+ 2\eta\sqrt{m}\norm{\nabla f(\bby_t)}
			\\
			\le&
			\rho \cdot 3 M\eta \cdot \frac{1}{M} \norm{\bs_t-\mathbf{1}\bbs_t}
			+
			9M\eta \cdot \norm{\yb_t-\mathbf{1}\bby_t}
			+
			\norm{\xb_t-\mathbf{1}\bbx_t} 
			\\
			&
			+ \norm{\mathbf{1}(\bby_t - x^*)} + \norm{\mathbf{1}(\bbx_t - x^*)} 
			+ 2\eta\sqrt{m}\norm{\nabla f(\bby_t)},
        \end{split}
		\end{align*}
		where the last two inequalities use $1<1+\alpha$, $\eta = {1}/{L}$ and $L\le M$.
		Furthermore, we have
		\begin{align*}
			&\norm{\mathbf{1}(\bby_t - x^*)} + \norm{\mathbf{1}(\bbx_t - x^*)} 
			+ 2\eta\sqrt{m}\norm{\nabla f(\bby_t)}
			\\
			\le&
			\norm{\mathbf{1}(\bby_t - x^*)} + \norm{\mathbf{1}(\bbx_t - x^*)} 
			+ 2L\eta\sqrt{m}\norm{\bby_t - x^*}
			\\
			\le&3\sqrt{m} \norm{\bby_t - x^*} + \sqrt{m}\norm{\bbx_t - x^*}
			\\
			\overset{\eqref{eq:y_x_V}}{\le}&3\sqrt{m}\sqrt{\frac{2V_t}{\mu}} + \sqrt{m}\norm{\bbx_t - x^*}
			\\
			\le&4\sqrt{m}\sqrt{\frac{2V_t}{\mu}}.
		\end{align*}
		The first inequality is because of the $L$-smoothness of $f(x)$.
		The second inequality follows from the step size $\eta = {1}/{L}$.
		The last inequality is due to the $\mu$-strong convexity.
		Thus, we can obtain that
		\begin{align*}
			& \frac{1}{M} \norm{\bs_{t+1}-\mathbf{1}\bbs_{t+1}}   \\
			\le &
			\rho \cdot 3M\eta \cdot\frac{1}{M}\norm{\bs_t-\mathbf{1}\bbs_t}
			+
			9M\eta \norm{\yb_t-\mathbf{1}\bby_t}
			+ \rho\cdot \frac{1}{\rho}\norm{\xb_t-\mathbf{1}\bbx_t} 
			+ 4\sqrt{m}\sqrt{\frac{2}{\mu}V_t}.
		\end{align*}
		By the definition of $\zb_t$, we can obtain that
		\begin{align*}
			\zb_{t+1} 
			=
			{\mathbf{A}}
			\zb_t 
			+
			\big[0,0,4\sqrt{2mV_t/\mu}\,\big]^\top.
		\end{align*}
	\end{proof}
	
	Next, we will prove the above two conditions which guarantee the convergence of $\norm{\zb_t}$. 
	In the following lemma, we show the properties of $\mathbf{A}$ and prove that the spectrum radius of $\mathbf{A}$ is less than~$\frac{1}{2}$ if $\rho$ is small enough.
	
	\begin{lemma}
		\label{lem:lam_max_1}
		Matrix  ${\mathbf{A}}$ defined in
		Lemma~\ref{lem:yvs_1} satisfies that
		\begin{equation*}
			0<\lambda_1(\mathbf{A}) \qquad \text{and} \qquad |\lambda_3(\mathbf{A})| \le |\lambda_2(\mathbf{A})|< \lambda_1(\mathbf{A}),
		\end{equation*}
		with $\lambda_i(\mathbf{A})$ being the $i$-th largest eigenvalue of $\mathbf{A}$. 
		Let $\eta = {1}/{L}$ and $\rho$ satisfy the condition that
		\begin{equation*}
			\rho \le \frac{1}{108 (M\eta)^3 + 288 (M\eta)^2 + 24 M\eta + 16},
		\end{equation*}
		then it holds that  
		\begin{equation*}
			\lambda_1({\mathbf{A}})\le \frac12,
		\end{equation*}
		and the eigenvector $\bv$ associated with $\lambda_1(\mathbf{A})$ is positive and its entries satisfy
		\begin{equation}
			\label{eq:v_v_1}
			\bv(1) \le \frac{\bv(3)}{18M\eta}, \quad \bv(2) \le \left(\frac{1}{18\sqrt{\rho}M\eta} + \frac{M\eta}{\sqrt{\rho}}\right)\bv(3),\quad 0<\bv(3),
		\end{equation}  
		where $\bv(i)$ is $i$-th entry of $\bv$.
	\end{lemma}
	\begin{proof}
		It is easy to check that $\mathbf{A}$ is non-negative and irreducible. 
		Furthermore, every diagonal entry of $\mathbf{A}$ is not zero.
		Thus, by Perron-Frobenius theorem and Corollary 8.4.7 of \cite{horn2012matrix}, $\mathbf{A}$ has a real-valued positive number $\lambda_1(\mathbf{A})$ which is algebraically simple and associated with a strictly positive eigenvector $\bv$. 
		It also holds that $\lambda_1(\mathbf{A})$ is strictly larger than $|\lambda_i(\mathbf{A})|$ with~$i=2,3$.

		We write down the characteristic polynomial $p(\zeta)$ of ${\mathbf{A}}$,
		\begin{align*}
			p(\zeta) = \zeta p_0(\zeta) - 9(M\eta)^2\rho + 3\rho^2M\eta,
		\end{align*}
		where
		\begin{equation*}
			p_0(\zeta) = \zeta^2 - \rho\left(2+3 M\eta\right) \zeta - \rho\left(18(M\eta)^2+ M\eta + 1 - 6\rho M\eta \right).
		\end{equation*}
		Let us denote 
		\begin{equation}
			\label{eq:Delta_1}
			\Delta = 4 \rho\left(18(M\eta)^2+ M\eta + 1 - 6\rho M\eta \right).
		\end{equation}
		It holds that
		\begin{align*}
			\frac{\Delta}{M\eta} 
			= &	4\rho (18 M\eta + 1 + \frac{1}{M\eta} - 6\rho)
			\ge 4\rho\left( 18  - 6 \right)
			> 0.
		\end{align*}
		Thus, two roots of $p_0(\zeta)$, $\zeta_1$ and $\zeta_2$ are
		\begin{equation*}
			\zeta_1 = \frac{ \rho(2+3 M\eta)+\sqrt{ (2+3 M\eta)^2\rho^2+\Delta}}{2}
		\end{equation*} 
		and
		\begin{equation*}
			\zeta_2 = \frac{ \rho(2+3 M\eta)-\sqrt{ (2+3 M\eta)^2\rho^2+\Delta}}{2}.
		\end{equation*} 
		Letting
		\begin{equation*}
			\zeta^* = \frac{2\rho\cdot\left(9(M\eta)^2+2\right)(2+3 M\eta)+\sqrt{\max\{\Delta,\frac{1}{4}\}}}{2},
		\end{equation*}
		we have
		\begin{align*}
			&p\left(\zeta^*\right)
			=
			\frac{2\rho\cdot\left(9(M\eta)^2+2\right)(2+3 M\eta)+\sqrt{\max\{\Delta,\frac{1}{4}\}}}{2}
			\\
			&\cdot
			\frac{2\rho\cdot\left(9(M\eta)^2+2\right)(2+3 M\eta)+\sqrt{\max\{\Delta,\frac{1}{4}\}} - \rho(2 + 3M\eta) - \sqrt{(2+3M\eta)^2 \rho^2 + \Delta}}{2}
			\\
			&\cdot
			\frac{2\rho\cdot\left(9(M\eta)^2+2\right)(2+3 M\eta)+\sqrt{\max\{\Delta,\frac{1}{4}\}} - \rho(2 + 3M\eta) + \sqrt{(2+3M\eta)^2 \rho^2 + \Delta}}{2}
			\\&
			- 9(M\eta)^2\rho + 3\rho^2M\eta
			\\
			&\ge
			\frac{2\rho\cdot\left(9(M\eta)^2+2\right)(2+3 M\eta)+\sqrt{\max\{\Delta,\frac{1}{4}\}}}{2}
			\\
			&\cdot
			\frac{\left(2\rho\cdot\left(9(M\eta)^2+1\right)(2+3 M\eta)+\sqrt{\max\{\Delta,\frac{1}{4}\}}\right)^2 - (\sqrt{(2+3M\eta)^2 \rho^2 + \Delta})^2}{2}
			\\&
			- 9(M\eta)^2\rho + 3\rho^2M\eta
			\\
			=&
			\frac{2\rho\cdot\left(9(M\eta)^2+2\right)(2+3 M\eta)+\sqrt{\max\{\Delta,\frac{1}{4}\}}}{2}\\
			&\cdot
			\bigg( \frac{\left(2\rho\cdot\left(9(M\eta)^2+1\right)(2+3 M\eta)\right)^2+\max\{\Delta,\frac{1}{4}\} - ((2+3M\eta)^2 \rho^2 + \Delta)}{2} 
			\\&+
			\left(2\rho\cdot\left(9(M\eta)^2+1\right)(2+3 M\eta)\right) \sqrt{\max\{\Delta,\frac{1}{4}\}} \bigg)
			\\
			&
			- 9(M\eta)^2\rho + 3\rho^2M\eta
			\\
			\ge&
			\frac{2\rho\cdot\left(9(M\eta)^2+2\right)(2+3 M\eta)+\sqrt{\max\{\Delta,\frac{1}{4}\}}}{2} \\
			& \cdot 
			\left(2\rho\cdot\left(9(M\eta)^2+1\right)(2+3 M\eta)\right) \sqrt{\max\{\Delta,\frac{1}{4}\}} \\ & - 9(M\eta)^2\rho
			\\
			>&
			\frac{\left(2\rho\cdot\left(9(M\eta)^2+1\right)(2+3 M\eta)\right)\cdot \max\{\Delta,\frac{1}{4}\} }{2}
			- 9(M\eta)^2\rho 
			\\
			\ge&
			\frac{2\rho(9(M\eta)^2 +1) \cdot 5}{8} - 9(M\eta)^2 \rho
			> 0.
		\end{align*}
		
		Note that $p(\zeta)$ is monotonely increasing in the range $\left[\zeta^*, \infty\right]$.
		Thus, $p(\zeta)$ does not have real roots in this range.
		This implies 
		$\lambda_1(\mathbf{A})\le \zeta^*$.
		By Eq.~\eqref{eq:Delta_1}, we can obtain that if $\rho$ satisfies \begin{align*}
			\rho \le \left(16 \cdot (18(M\eta)^2 + M\eta + 1)\right)^{-1},
		\end{align*}
		then it holds that
		$
		\Delta \le \frac{1}{4}
		$.
		If $\rho$ also satisfies the condition that
		\begin{align*}
			\rho \le \left(4 \cdot \left(9(M\eta)^2 + 2\right) (2+3M\eta)\right)^{-1},
		\end{align*}
		then we can obtain that
		\begin{align*}
			\lambda_1(\mathbf{A})
			\le \zeta^* \le
			\frac{\frac12+\sqrt{\max\{\Delta,\frac{1}{4}\}}}{2}
			=\frac{1}{2}.
		\end{align*}
		Combining the above conditions of $\rho$, we only need that
		\begin{equation*}
			\rho \le \frac{1}{108 (M\eta)^3 + 288 (M\eta)^2 + 24 M\eta + 16}.
		\end{equation*}		
		Now, we show that $\sqrt{\rho}<\lambda_1(\mathbf{A})$.
		We can conclude this result once it holds $p(\sqrt{\rho})<0$.
		This is because $p(\zeta)$ will have a root between $\sqrt{\rho}$ and $1/2$ and $\lambda_1(\mathbf{A})$ must be no less than this root.
		We have
		\begin{align*}\small
        \begin{split}            
			p(\sqrt{\rho}) 
			=&
			\sqrt{\rho}\,p_0(\sqrt{\rho}) - 9(M\eta)^2\rho + 3\rho M\eta
			\\
			=&
			\rho\left(\sqrt{\rho} - \rho(2+3 M\eta) - \frac{\Delta}{4\sqrt{\rho}} - 9 (M\eta)^2 + 3\rho M\eta\right)
			\\
			=&
			\rho\left(\sqrt{\rho} -2\rho  - \frac{\Delta}{4\sqrt{\rho}} - 9M^2\eta^2\right)
			\\
			=&
			\rho\left(-2\left(\sqrt{\rho} - \frac{1}{4}\right)^2 + \frac{1}{8} - \frac{\Delta}{4\sqrt{\rho}}  - 9M^2\eta^2\right)
			< 0,
        \end{split}
		\end{align*}
		where the  last inequality is because of $M\eta\ge 1$ (by Eq.~\eqref{eq:lg}).
		
		Since $\bv$ is the eigenvector associated with $\lambda_1(\mathbf{A})$, we can obtain that $\mathbf{A}\bv = \lambda_1(\mathbf{A})\bv$ and have the following equations
		\begin{align*}
			2\rho \bv(1) + \rho\bv(2) + 2\rho M\eta\bv(3) &= \lambda_1(\mathbf{A}) \bv(1),
			\\
			\bv(1) + M\eta\bv(3) &= \lambda_1(\mathbf{A})\bv(2), 
			\\
			9M\eta \bv(1) + \rho\bv(2) + 3\rho M\eta\bv(3) &= \lambda_1(\mathbf{A})\bv(3).
		\end{align*}
		Thus, combining with $\sqrt{\rho} \le \lambda_1(\mathbf{A}) \le \frac{1}{2}$, we can obtain that 
		\begin{align*}
			\bv(1) \le \frac{1}{9M\eta}\left(\lambda_1(\mathbf{A})\bv(3) - (\rho\bv(2) + 3\rho M\eta \bv(3))\right) < \frac{\bv(3)}{18M\eta},
		\end{align*}
		and 
		\begin{align*}
			\bv(2) = \frac{\bv(1)+M\eta\bv(3)}{\lambda_1(\mathbf{A})} \le \left(\frac{1}{18\sqrt{\rho}M\eta} + \frac{M\eta}{\sqrt{\rho}}\right)\bv(3).
		\end{align*}
	\end{proof}
	
	\begin{lemma}
		Letting $V_t$ be the Lyapunov function defined in Eq.~\eqref{eq:V_t} associated to Algorithm~\ref{alg:DAGD}, then it satisfies the following property 
		\begin{equation}
			\label{eq:V_up}
			V_{t+1} \le \left( 1 - \frac{3}{4}\alpha \right) V_t + \left( 1 + \frac{8}{\alpha^3} \right) \cdot \frac{M^2}{L} \cdot \frac{1}{m} \norm{\yb_t - \mathbf{1}\bby_t}^2.
		\end{equation}
	\end{lemma}
	\begin{proof}
		When $r(x) = 0$, $h(x)$ equals to $f(x)$. Thus, we use $f(x)$ directly instead of $h(x)$.
		By the update procedure of Algorithm~\ref{alg:DAGD}, we have
		\begin{equation}
			\label{eq:x_up}
			\begin{aligned}
				f(\bbx_{t+1}) 
				\le& 
				f(\bby_t) - \eta \dotprod{\nabla f(\bby_t), \bbg_t} + \frac{L\eta^2}{2}\norm{\bbg_t}^2
				\\
				=&
				f(\bby_t) - \eta \dotprod{\bbg_t, \bbg_t} + \eta \dotprod{\bbg_t,\bbg_t - \nabla f(\bby_t)}
				+ \frac{L\eta^2}{2}\norm{\bbg_t}^2
				\\
				=& f(\bby_t) - \frac{1}{2L} \norm{\bbg_t}^2 + \frac{1}{L}\dotprod{\bbg_t,\bbg_t - \nabla f(\bby_t)},
			\end{aligned}
		\end{equation}
		where the last equation is because $\eta = {1}/{L}$.
		Furthermore, by the definition of $V_t$, we have
		\begin{align*}
			V_{t+1} =\,& \frac{\mu}{2} \norm{\bbv_{t+1} - x^*}^2 + f(\bbx_{t+1}) - f(x^*)
			\\
			\overset{\eqref{eq:vv_p}}{=}&
			\frac{\mu}{2}\norm{(1-\alpha)\bbv_t + \alpha\bby_t - x^*}^2 - \frac{\mu}{L\alpha}\dotprod{\bbg_t, (1-\alpha)\bbv_t + \alpha \bby_t - x^*} 
			\\&+
			\frac{\mu}{2L^2\alpha^2}\norm{\bbg_t}^2 + f(\bbx_{t+1}) - f(x^*)
			\\
			\overset{\eqref{eq:x_up}}{\le}&
			\frac{\mu}{2}\norm{(1-\alpha)\bbv_t + \alpha\bby_t - x^*}^2 - \alpha \dotprod{\bbg_t, (1-\alpha)\bbv_t + \alpha \bby_t - x^*} 
			\\&+ f(\bby_t) - f(x^*) + \frac{1}{L}\dotprod{\bbg_t,\bbg_t - \nabla f(\bby_t)}.
		\end{align*}
		Furthermore, by Eq.~\eqref{eq:y_v}, we can obtain that $\bbv_t = \bby_t + \frac{1}{\alpha}(\bby_t - \bbx_t)$. 
		Then we can obtain
		\begin{align*}
			(1-\alpha)\bbv_t + \alpha\bby_t = \bby_t + \frac{1-\alpha}{\alpha} (\bby_t - \bbx_t).
		\end{align*}
		Hence, we have
		\begin{align*}
			&f(\bby_t) - \alpha \dotprod{\bbg_t, (1-\alpha)\bbv_t + \alpha \bby_t - x^*} -f(x^*)
			\\
			=&f(\bby_t) + \dotprod{\bbg_t, \alpha x^* + (1-\alpha)\bbx_t - \bby_t} - f(x^*)
			\\
			=&
			(\alpha + 1-\alpha) f(\bby_t) + \dotprod{\nabla f(\bby_t), \alpha (x^* - \bby_t) + (1-\alpha)(\bbx_t - \bby_t)} - f(x^*)
			\\&+ \dotprod{\bbg_t - \nabla f(\bby_t), \alpha x^* + (1-\alpha)\bbx_t - \bby_t}
			\\
			\le&
			(1-\alpha)(f(\bbx_t) - f(x^*)) - \frac{\alpha\mu}{2}\norm{x^* - \bby_t}^2+ \dotprod{\bbg_t - \nabla f(\bby_t), \alpha x^* + (1-\alpha)\bbx_t - \bby_t},
		\end{align*}
		where the last inequality is because $f(x)$ is $\mu$-strongly convex.
		Therefore, we can obtain that
		\begin{align*}
        \small\begin{split}            
			V_{t+1} 			
			\le&
			\frac{\mu}{2} \norm{(1-\alpha)\bbv_t + \alpha\bby_t - x^*}^2 + \frac{1}{L}\dotprod{\bbg_t,\bbg_t - \nabla f(\bby_t)}
			\\& + (1-\alpha)(f(\bbx_t) - f(x^*)) - \frac{\alpha\mu}{2}\norm{x^* - \bby_t}^2+ \dotprod{\bbg_t - \nabla f(\bby_t), \alpha x^* + (1-\alpha)\bbx_t - \bby_t}
			\\
			\le&
			\frac{\mu(1-\alpha)}{2} \norm{\bbv_t - x^*}^2 + \frac{\mu\alpha}{2}\norm{\bby_t - x^*}^2  + (1-\alpha)(f(\bbx_t) - f(x^*))
			\\& - \frac{\alpha\mu}{2}\norm{x^* - \bby_t}^2+ \dotprod{\bbg_t - \nabla f(\bby_t), \alpha x^* + (1-\alpha)\bbx_t - \bby_t + \frac{1}{L}\bbg_t}
			\\
			=&(1-\alpha) V_t  + \dotprod{\bbg_t - \nabla f(\bby_t), \alpha x^* + (1-\alpha)\bbx_t - \bby_t }+ \frac{1}{L}\norm{\bbg_t - \nabla f(\bby_t)}\norm{\bbg_t},
		\end{split}
        \end{align*}
		where the  second inequality is because of
		\begin{align*}
        \small\begin{split}            
			\norm{(1-\alpha)\bbv_t + \alpha\bby_t - x^*}^2 \leq \left((1-\alpha)\norm{\bbv_t - x^*} + \alpha \norm{\bby_t - x^*}\right)^2 
			\leq & (1-\alpha) \norm{\bbv_t - x^*}^2 + \alpha \norm{\bby_t - x^*}^2.
        \end{split}
		\end{align*}
		Furthermore, we have
		\begin{align*}
			\norm{\alpha x^* + (1-\alpha)\bbx_t - \bby_t} 
			\leq
			(1-\alpha)\norm{\bbx_t - x^*} + \alpha\norm{\bby_t - x^*}
			{\le}
			\max\left\{\sqrt{\frac{2}{\mu}V_t}, \sqrt{\frac{2}{\mu}V_t}\right\} 
			\le
			\sqrt{\frac{2V_t}{\mu}}.
		\end{align*}
		Therefore, we have
		\begin{align*}
        \small\begin{split}            
        	V_{t+1} 
			\leq& 
			(1-\alpha) V_t + \frac{1}{L}\norm{\bbg_t - \nabla f(\bby_t)}\norm{\bbg_t} + \sqrt{\frac{2V_t}{\mu}} \norm{\bbg_t - \nabla f(\bby_t)}
			\\
			\le&(1-\alpha) V_t +  \frac{1}{L}\norm{\bbg_t - \nabla f(\bby_t)}^2 + \frac{1}{L}\norm{\bbg_t - \nabla f(\bby_t)}\norm{\nabla f(\bby_t)} 
			+ \sqrt{\frac{2V_t}{\mu}} \norm{\bbg_t - \nabla f(\bby_t)}
			\\
			{\le}&
			(1-\alpha) V_t +  \frac{1}{L}\norm{\bbg_t - \nabla f(\bby_t)}^2 + 2\sqrt{\frac{2V_t}{\mu}} \norm{\bbg_t - \nabla f(\bby_t)}
			\\
			\le& (1-\alpha) V_t +  \frac{1}{L}\norm{\bbg_t - \nabla f(\bby_t)}^2 + \frac{\alpha}{4} V_t + \frac{8}{\alpha^3} \cdot \frac{1}{L}\norm{\bbg_t - \nabla f(\bby_t)}^2
			\\
			\stackrel{\eqref{eq:g_y}}{\le}&
			\left( 1 - \frac{3}{4}\alpha \right) V_t + \left( 1 + \frac{8}{\alpha^3} \right) \cdot \frac{M^2}{L} \cdot \frac{1}{m} \norm{\yb_t - \mathbf{1}\bby_t}^2.
		\end{split}
		\end{align*}
	\end{proof}

    Now, we provide the proof of Theorem~\ref{thm:main_1}.
	\begin{proof} Let the eigenvector $\bv$ be defined in Lemma~\ref{lem:lam_max_1} and set $\bv(3)=1$. 
		Combining with the fact that first two entries of $\zb_0$ are zero, we can obtain that,
		\begin{equation*}
			\zb_0 \le \norm{\zb_0}\bv \qquad \mbox{and}\qquad [0,0, 1]^\top \le \bv.
		\end{equation*} 
		By Eq.~\eqref{eq:zaz_1}, we can obtain that
		\begin{equation}
			\label{eq:nm_z_1}
			\begin{aligned}
				\zb_{t+1} 
				\le& 
				\norm{\zb_0} \cdot\mathbf{A}^{t+1}\bv 
				+ 
				4\sqrt{\frac{2m}{\mu}}\cdot\sum_{i=0}^t\sqrt{V_i}\cdot\mathbf{A}^{t-i}\bv
				\\
				=&
				\norm{\zb_0}\lambda_1(\mathbf{A})^{t+1}\bv 
				+
				4\sqrt{\frac{2m}{\mu}}\cdot\sum_{i=0}^t\sqrt{V_i}\cdot\lambda_1(\mathbf{A})^{t-i}\bv 
				\\
				\le&
				\norm{\zb_0}\left(\frac{1}{2}\right)^{t+1}\cdot\bv 
				+
				4\sqrt{\frac{2m}{\mu}}\cdot\sum_{i=0}^t\left(\frac{1}{2}\right)^{t-i}\sqrt{V_i}\cdot\bv,
			\end{aligned}
		\end{equation}
		where the first equality is because $\bv$ is the eigenvector associated with $\lambda_1(\mathbf{A})$ and the last inequality is because of Lemma~\ref{lem:lam_max_1}.
		
		Next, we will prove our result by induction. 
		We have $\norm{\bbs_0 - \eta\nabla f(\bby_0)} = 0$, 
		because the initial values $\xb_0(i,:)$ are equal to each other.
		Then by Eq.~\eqref{eq:V_up}, we have
		\begin{equation*}
			V_1 \leq \left(1-\frac{3\alpha}{4}\right) V_0 \leq \left(1 - \frac{\alpha}{2}\right) \left(V_0+\frac{\mu}{288m}\norm{\zb_0}^2\right).
		\end{equation*}
		Next, we assume that for $i= 1,\dots, t$,  it holds that 
		\begin{equation*}
			V_i \leq \left(1 - \frac{\alpha}{2}\right)^i \left(V_0+\frac{\mu}{288m}\norm{\zb_0}^2\right).
		\end{equation*}
		Combining with Eq.~\eqref{eq:nm_z_1}, we can obtain that
		\begin{equation}
			\label{eq:z_1}
			\begin{aligned}
				\zb_{t-1}
				\le&
				\bv\cdot\left(4 \sqrt{\frac{2m}{\mu}} \sum_{j=0}^{t-2} 2^{-(t-2-j)} \sqrt{V_j}
				+
				2^{-(t-1)}\norm{\zb_0}
				\right)
				\\
				\le&
				\bv \cdot
				\left(
				4 \sqrt{\frac{2m}{\mu}}\sum_{j=0}^{t-2} 2^{-(t-2-j)} \left(\sqrt{1 - \frac{\alpha}{2}}\,\right)^j \sqrt{V_0+\frac{\mu}{288m}\norm{\zb_0}^2}
				+
				2^{-(t-1)}\norm{\zb_0}
				\right)
				\\
				=&
				\bv\cdot
				\left(4\sqrt{\frac{2m}{\mu}}\frac{2\left(\sqrt{1 - {\alpha}/{2}}\,\right)^{t-1}- 2^{-(t-2)}}{2\sqrt{1-{\alpha}/{2}}-1} \sqrt{V_0+\frac{\mu}{288m}\norm{\zb_0}^2}
				+
				2^{-(t-1)}\norm{\zb_0}
				\right)
				\\
				\le&
				\bv\cdot
				\left(
				12\sqrt{\frac{2m}{\mu}}\left(\sqrt{1 - \frac{\alpha}{2}}\,\right)^{t-1} \sqrt{V_0+\frac{\mu}{288m}\norm{\zb_0}^2}
				+
				2^{-(t-1)}\norm{\zb_0}
				\right).
			\end{aligned}
		\end{equation}
		
		Now we upper bound the value of $\norm{\bbs_t - \nabla f(\bby_t)}$.
		First, by Lemma~\ref{lem:yvs_1}, we can obtain that
		\begin{align*}
			\begin{split}
				&\norm{\yb_t - \mathbf{1}\bby_t} \\
				\le\,&
				\dotprod{[2\rho, \rho, 2\rho M\eta], \zb_{t-1}}
				\\
				\overset{\eqref{eq:z_1}}{\le}& \rho\left(2\bv(1) + \bv(2) + 2M\eta\right) \cdot
				\left(
				12\sqrt{\frac{2m}{\mu}}\left(\sqrt{1 - \frac{\alpha}{2}} \right)^{t-1} \sqrt{V_0+\frac{\mu}{288m}\norm{\zb_0}^2}
				+
				2^{-(t-1)}\norm{\zb_0}
				\right)
				\\
				\stackrel{\eqref{eq:v_v_1}}{\le}&
				\rho\cdot \left( \frac{2}{18 M\eta } +  \frac{1 }{18\sqrt{\rho} M\eta} + \frac{M\eta}{\sqrt{\rho}} + 2M\eta\right)  \\
				& \cdot
				\left(
				12\sqrt{\frac{2m}{\mu}}\left(\sqrt{1 - \frac{\alpha}{2}} \right)^{t-1} \sqrt{V_0+\frac{\mu}{288m}\norm{\zb_0}^2}
				+
				2^{-(t-1)}\norm{\zb_0}
				\right)
				\\
				\le\,&\sqrt{\rho} \cdot 2 M\eta \cdot 
				\left(
				12\sqrt{\frac{2m}{\mu}}\left(\sqrt{1 - \frac{\alpha}{2}} \right)^{t-1} \sqrt{V_0+\frac{\mu}{288m}\norm{\zb_0}^2}
				+
				2^{-(t-1)}\norm{\zb_0}
				\right) .
			\end{split}
		\end{align*}
		Combining the inductive hypothesis with Eq.~\eqref{eq:V_up}, we have
		\begin{equation}
			\begin{aligned}
				& V_{t+1} \\
				\overset{\eqref{eq:V_up}}{\le}&
				\left( 1 - \frac{3}{4}\alpha \right) V_t + \left( 1 + \frac{8}{\alpha^3} \right) \cdot \frac{M^2}{L} \cdot \frac{1}{m} \norm{\yb_t - \mathbf{1}\bby_t}^2
				\\
				\le&
				\left(1 - \frac{3\alpha}{4}\right)\left(1-\frac{\alpha}{2}\right)^{t}\left( V_0 +  \frac{\mu}{288m}\norm{\zb_0}^2\right)
				\\
				&+ 2\rho \cdot \left( 1 + \frac{8}{\alpha^3} \right) \cdot \frac{M^2}{L}\cdot   (2 M\eta)^2 \cdot 
				\left( \frac{288m}{\mu} \left( 1 - \frac{\alpha}{2} \right)^{t-1} \left( V_0 +  \frac{\mu}{288m}\norm{\zb_0}^2\right) + 4^{-(t-1)} \norm{\zb_0}^2 \right)
				\\
				\le&
				\left(1 - \frac{3\alpha}{4}\right)\left(1-\frac{\alpha}{2}\right)^{t}\left( V_0 +  \frac{\mu}{288m}\norm{\zb_0}^2\right)
				\\&
				+ 8 \cdot 288\cdot \rho \cdot \left( 1 + \frac{8}{\alpha^3} \right) \cdot \frac{M^2}{L^2} \cdot \frac{L}{\mu}\cdot   (2 M\eta)^2 \cdot 
				\left( {1-\frac{\alpha}{2}}\right) ^t\left(  {V_0}  
				+
				\frac{\mu}{288m}\norm{\zb_0}^2\right)
				\\
				\le&
				\left(1-\frac{\alpha}{2}\right)^{t+1}\left(V_0+\frac{\mu}{288m}\norm{\zb_0}^2\right),
			\end{aligned} 
		\end{equation}
		where the last inequality is because of  
		\begin{align*}
			\rho \le \frac{1}{4^3\cdot9\cdot 288 } \cdot \left(\frac{L}{M}\right)^4 \kappa_g^{-3}.
		\end{align*}		
		Therefore, we can obtain that at the  $(t+1)$-th iteration, it also holds that
		\begin{equation*}
			V_{t+1} \le \left(1-\frac{\alpha}{2}\right)^{t+1} \left(V_0+\frac{\mu}{288m}\norm{\zb_0}^2\right).
		\end{equation*}		
		Furthermore, 
		\begin{align*}
			&\frac{1}{\rho}\norm{\xb_t - \mathbf{1}\bbx_t} 
			\stackrel{\eqref{eq:zaz_1}}{\le} \dotprod{[1, 0, M\eta], z_{t-1}} \\
			\stackrel{\eqref{eq:z_1}}{\le}& \dotprod{[1, 0, M\eta], \bv} \cdot \left(
			12\sqrt{\frac{2m}{\mu}}\left(\sqrt{1 - \frac{\alpha}{2}} \right)^{t-1} \sqrt{V_0+\frac{\mu}{288m}\norm{\zb_0}^2}
			+
			2^{-(t-1)}\norm{\zb_0}
			\right)
			\\
			\stackrel{\eqref{eq:v_v_1}}{\le}& 2 M\eta \cdot \left(
			12\sqrt{\frac{2m}{\mu}}\left(\sqrt{1 - \frac{\alpha}{2}} \right)^{t-1} \sqrt{V_0+\frac{\mu}{288m}\norm{\zb_0}^2}
			+
			2^{-(t-1)}\norm{\zb_0}
			\right).
		\end{align*}
		Thus, we can obtain that 
		\begin{align*}
			\norm{\xb_t - \mathbf{1}\bbx_t} 
			\le & \rho \cdot  2 M\eta \cdot \left(
			12\sqrt{\frac{2m}{\mu}}\left(\sqrt{1 - \frac{\alpha}{2}\,} \right)^{t-1} \sqrt{V_0+\frac{\mu}{288m}\norm{\zb_0}^2}
			+
			2^{-(t-1)}\norm{\zb_0}
			\right) \\
			=& \cO\left(\sqrt{\frac{m\epsilon}{\mu}}\,\right).
		\end{align*}
		This finishes our proof.

	\end{proof}
 
	\bibliography{ref}		

\end{document}